\documentclass{article}

\usepackage{preprint}
\usepackage[numbers, square]{natbib}

\usepackage[T1]{fontenc}
\usepackage[utf8]{inputenc}
\usepackage[english]{babel}
\usepackage[bottom]{footmisc}
\usepackage{etoolbox}

\usepackage{mathtools, amssymb, amsthm, bm, eucal, dsfont, amsmath}

\usepackage{etoc}

\usepackage{wrapfig}
\usepackage{graphicx}
\usepackage{subcaption}
\usepackage{tikz}
\usepackage{pgfplots}
\pgfplotsset{compat=1.18}
\usetikzlibrary{arrows.meta}
\usepgfplotslibrary{groupplots}

\usepackage[table]{xcolor}
\usepackage{tabularx}

\usepackage{float}
\usepackage{verbatim}
\usepackage{gensymb}
\usepackage{pifont}
\usepackage{changepage}
\usepackage{empheq}

\definecolor{myblack}{HTML}{000000}
\definecolor{myorange}{HTML}{E69F00}
\definecolor{myblue}{HTML}{56B4E9}
\definecolor{mygreen}{HTML}{009E73}
\definecolor{myyellow}{HTML}{F0E442}
\definecolor{mydarkblue}{HTML}{0072B2}
\definecolor{myred}{HTML}{D55E00} 
\definecolor{mypurple}{HTML}{CC79A7} 
\definecolor{mygrey}{HTML}{999999}

\newcommand{\Tr}[1]{\operatorname{Tr}\left\{#1\right\}}
\newcommand{\R}{\mathbb{R}}
\makeatletter
\newcommand{\norm}{\@ifnextchar[{\norm@i}{\norm@i[]}}
\newcommand{\norm@i}[2][]{%
  \ensuremath{%
    \left\lVert #2 \right\rVert
    \if\relax\detokenize{#1}\relax
    \else_{#1}\fi
  }%
}
\makeatother
\newcommand{\normal}[1]{\mathcal{N}\left( #1 \right)}

\newcommand{\abs}[1]{\left\lvert #1 \right\rvert}
\newcommand{\prob}[1]{\mathbb{P}\left( #1 \right)}

\DeclareMathOperator*{\argmin}{arg\,min}

\newcommand{\empadvpcarisk}[1]{\widehat{\mathcal{R}}_\delta(#1)}
\newcommand{\evat}[2]{{#1}\smash{\raisebox{-0.5ex}{$\Big|_{#2}$}}}
\newcommand{\snorm}[1]{\left\| \smash{#1}\vphantom{\text{X}} \right\|}
\newcommand{\eqstack}[1]{\stackrel{\text{#1}}{=}}
\newcommand{\leqstack}[1]{\stackrel{\text{#1}}{\leq}}

\newcommand{\scaleeq}[2]{\scalebox{#2}{$#1$}}

\makeatletter
\providecommand{\macc@style}{}
\providecommand{\macc@kerna}{}
\DeclareRobustCommand*\widebar[1]{%
  \begingroup
  \def\mathaccent##1##2{%
    \setbox\z@\hbox{$\macc@style{\macc@nucleus}_{}$}%
    \setbox\tw@\hbox{$\macc@style{\macc@nucleus}{}_{}$}%
    \dimen@\wd\tw@ \advance\dimen@-\wd\z@
    \divide\dimen@ 3
    \@tempdima\wd\tw@ \advance\@tempdima-\scriptspace \divide\@tempdima 10
    \advance\dimen@-\@tempdima
    \ifdim\dimen@>\z@ \dimen@0pt\fi
    \rel@kern{0.6}\kern-\dimen@
    \overline{\rel@kern{-0.6}\kern\dimen@\macc@nucleus\rel@kern{0.4}\kern\dimen@}%
    \advance\dimen@0.4\dimexpr\macc@kerna
    \kern-\dimen@
    \let\macc@nucleus\relax 
  }%
  \macc@depth\@ne
  \let\math@bgroup\@empty \let\math@egroup\macc@set@skewchar
  \mathsurround\z@ \frozen@everymath{\mathgroup\macc@group\relax}%
  \macc@set@skewchar\relax
  \let\mathaccentV\macc@nested@a
  \macc@nested@a\relax111{#1}%
  \endgroup
}
\newcommand*\rel@kern[1]{\kern#1\dimexpr\macc@kerna}
\makeatother

\usepackage[shortlabels]{enumitem}
\setlist{itemsep=0pt, leftmargin=15pt, topsep=0pt,}
\setlist[1]{labelindent=\parindent}

\usepackage[framemethod=TikZ]{mdframed}
\usepackage{thmtools}
\mdfsetup{
    skipabove=8pt, 
    skipbelow=0pt,
    hidealllines=false,
    innerleftmargin=10pt,
    innerrightmargin=10pt,
    innertopmargin=5pt,
    innerbottommargin=5pt,
    linewidth=1pt,
    roundcorner=3pt 
}
\declaretheoremstyle[
    mdframed={
        backgroundcolor=white, 
        linecolor=black
    }]{thmstyle}

\declaretheorem[style=thmstyle]{proposition}

\declaretheorem[style=thmstyle]{lemma}

\usepackage[ruled,vlined]{algorithm2e} 
\SetKwInput{KwInput}{input}
\SetKwInput{KwOutput}{output:}
\SetKw{KwInitialize}{init:}
\SetKw{KwRet}{return}
\DontPrintSemicolon
\SetKw{KwTo}{to}
\SetCommentSty{textit}
\SetKwComment{Comment}{\(\triangleright\)\ }{}

\usepackage[colorlinks=true,linkcolor=myblue,citecolor=myblue,urlcolor=myblue]{hyperref}

\title{A Robust Optimization Approach to Sparse Principal Component Analysis}
\author{\vspace{-0.8cm}\\David Vävinggren$^{1}$, Francis Bach$^{2}$,
  André M. H. Teixeira$^{1}$,\\ Dave Zachariah$^{1}$, and Antônio H. Ribeiro$^{1, 3}$\\[0.8em]
  \normalsize $^{1}$Uppsala University, Sweden\\
  \normalsize $^{2}$PSL Research University / INRIA, France\\
  \normalsize $^{3}$Science for Life Laboratory, Sweden
}

\begin{document}
\etocdepthtag.toc{main}
\maketitle

\vspace{-0.7cm}
\begin{abstract}
\noindent While principal component analysis (PCA) is a fundamental tool for dimensionality reduction, its dense representations make it ill-suited for high-dimensional data. Existing methods address this by promoting sparsity through explicit $\ell_1$-penalties, but these are not obvious to tune due to the unsupervised nature of the task. In contrast, we propose Adversarial PCA (AdvPCA), which leverages robust optimization to achieve sparsity by optimizing the reconstruction objective against bounded, worst-case latent space perturbations. We show that this formulation admits a closed-form reduction, leading to a practical iterative algorithm that alternates between adversarial linear regression-style updates for the sparse encoder and orthogonal updates for the decoder. By theoretically characterizing the solution, we derive a data-adaptive parameterization that allows the algorithm to perform effectively out of the box. We validate these claims through numerical experiments on synthetic and real-world genomics data.
\end{abstract}
\vspace{-0.2cm}
\section{Introduction}
Principal component analysis (PCA) is a standard technique for reducing the dimensionality of data \cite{pearson_liii_1901, shlens_tutorial_2014}, and is one of the most common techniques for data compression and data visualization \cite[ch.~10]{deisenroth_mathematics_2020}. For a target dimension $k \leq d$, PCA can be formulated as the reconstruction problem
\begin{equation}\label{eq:pca}
\min_{A, B} \sum_{i=1}^n \norm[2]{x_i - AB^\top x_i}^2 \quad \text{subject to} \quad A^\top A = I_k,
\end{equation}
where zero-mean datapoints $\mathcal{D} = \{x_i\}_{i=1}^n$ in $\R^d$ are compressed and then reconstructed by linear transformations $B \in \R^{d \times k}$ and $A \in \R^{d \times k}$ respectively \cite[ch.~23]{shalev-shwartz_understanding_2014}. Intuitively, the columns of $A$ span a linear subspace through the origin and decide the space in which the reconstructions must ultimately reside. This subspace is typically of much lower dimension than the original input space, thereby enabling dimensionality reduction. On the other hand, $B$ plays a different role in that it determines how the original input variables are linearly combined to form the latent representation. This selection is well-known to be dense in the case of PCA, meaning that \emph{all} input dimensions contribute to the reconstruction, albeit at different magnitudes \cite[ch.~11]{jolliffe_principal_2002}.

Because PCA produces dense combinations of the input variables, it is generally poorly adapted to the high-dimensional regime ($n<d$) \cite[ch.~8]{hastie_statistical_2015} \cite{johnstone_pca_2018}. To combat this issue, it is common to promote \emph{sparse combinations} of the input variables \cite[ch.~8]{wainwright_high-dimensional_2019}, i.e., sparsity in the matrix~$B$, motivated by the fact that the intrinsic structure of high-dimensional data often can be captured by a small subset of features \cite[ch.~10]{deisenroth_mathematics_2020}. Prior works, e.g., the Sparse PCA algorithm \cite{zou_sparse_2006}, have formulated this by augmenting~\eqref{eq:pca} with $\ell_1$-norm penalties over the columns of $B$. A problem with this approach, however, is that it is not obvious how the hyperparameters introduced in this Lasso-type \cite{tibshirani_regression_1996} setting should be tuned. In particular, the unsupervised nature of the task makes standard techniques like cross-validation difficult to apply directly. Recent works have shown adversarially trained linear regression to be a viable alternative to Lasso, exhibiting similar regularization behavior while enjoying favorable properties \cite{ribeiro_regularization_2023, ribeiro_kernel_2025}. Most notably, contrary to the Lasso, the hyperparameter tuning does not require knowledge of the noise level present in the data. By building on the linear adversarial training framework, we aim to achieve a PCA method that can leverage this adaptive property for compression in the high-dimensional data regime.

To this end, we propose a robust optimization approach to sparse PCA which (instead of the Lasso) leverages linear adversarial training to naturally promote sparse input variable combinations. By optimizing against worst-case perturbations, we derive a min-max formulation of PCA~\eqref{eq:pca} given by
\begin{equation}\label{eq:advpca1}
\min_{A,B} \sum_{i=1}^n \underset{r_i \in \Omega_{\delta}}{\text{max}}  \norm[2]{x_i - A(B^\top x_i + r_i)}^2 \quad \text{subject to} \quad A^\top A = I_k,
\end{equation}
where we introduce an adversary $r$ that is allowed to perturb each $B^\top x$ within a fixed budget $\Omega_\delta$. As will be shown, our formulation allows for the inner maximum to be solved in closed-form, decomposing the problem into $k$ independent adversarial linear regression problems that induce sparsity in the matrix $B$. This makes the method well-equipped to handle high-dimensional problems out of the box. We name our method \mbox{\textbf{Adversarial PCA} (\textbf{AdvPCA})} motivated by the strong connection to linear adversarial training. An illustration of the AdvPCA procedure is presented in Figure~\ref{fig:3d_advpca}.

Our primary contribution is the \emph{proposal of Adversarial PCA}~\eqref{eq:advpca1}, casting sparse dimensionality reduction as a robust optimization problem over the latent space. Specifically:
\begin{enumerate}
\item In Section~\ref{sec:proposed_formulation}, we \textit{derive a closed-form solution to the inner maximization} of~\eqref{eq:advpca1} and \emph{provide an efficient solver} tailored to the resulting objective.
\item In Section~\ref{sec:perturb_set}, we \emph{establish theoretical connections between our latent space formulation and the input-perturbed counterpart}, proving that the two are perfectly equivalent in the $k=1$ case.
\item In Section~\ref{sec:properties_of_sol}, we \emph{study the theoretical properties of the solution}. We \emph{characterize the regularization path of the solution} by quantifying regimes of high and low regularization, and use the intuition gained to \emph{propose a practical choice of adversarial radius $\delta$}.
\item In Section~\ref{sec:experiments}, we \emph{validate our method with numerical experiments} on synthetic and real-world data, demonstrating the efficacy of Adversarial PCA when compared to sparse PCA baselines.
\end{enumerate}

\begin{figure}[t]
  \centering
  \resizebox{0.8\textwidth}{!}{%
  \begin{tikzpicture}[
      image node/.style={inner sep=0pt},
      operator label/.style={font=\bfseries\small}
  ]
    \matrix [column sep=0.5cm, ampersand replacement=\&] (my_images) {

      \node[image node, yshift=0.05cm] (input) {
          \includegraphics[width=3.5cm]{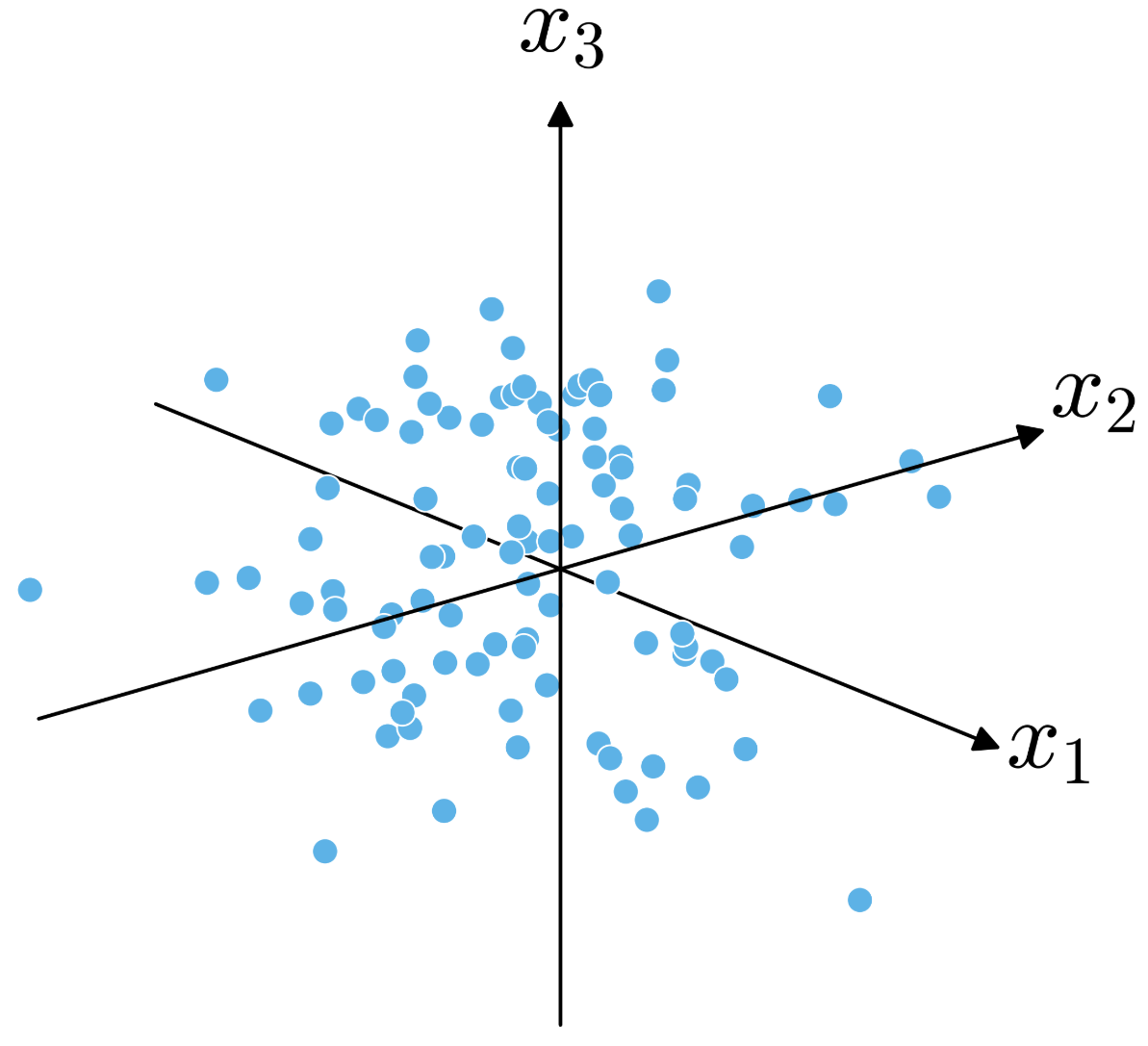} 
      }; 
      \& 
      
      \node[image node] (latent) {
          \includegraphics[width=3.5cm]{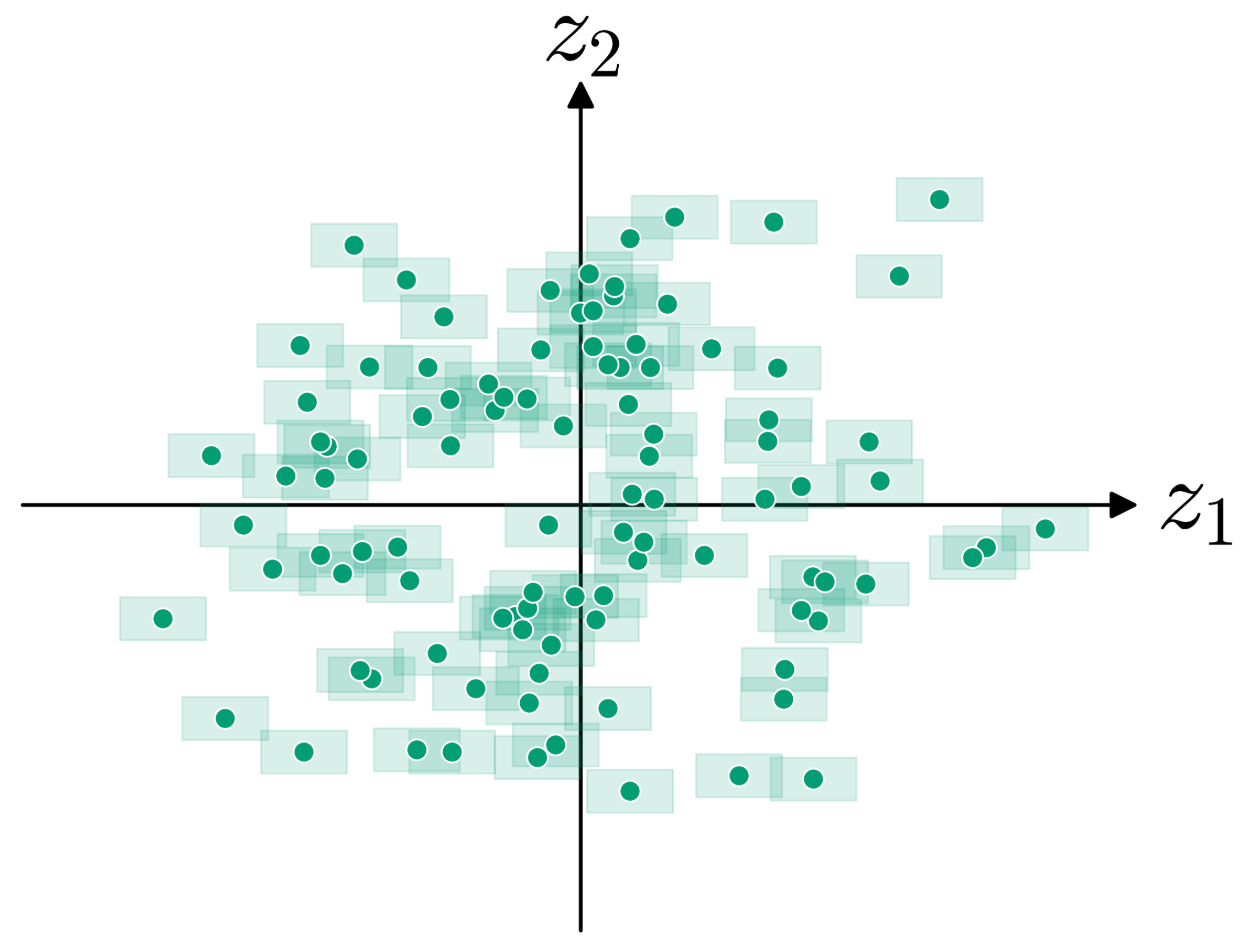} 
      }; 
      \& 
    
      \node[image node, yshift=-0.04] (recon) {
          \includegraphics[width=3.5cm]{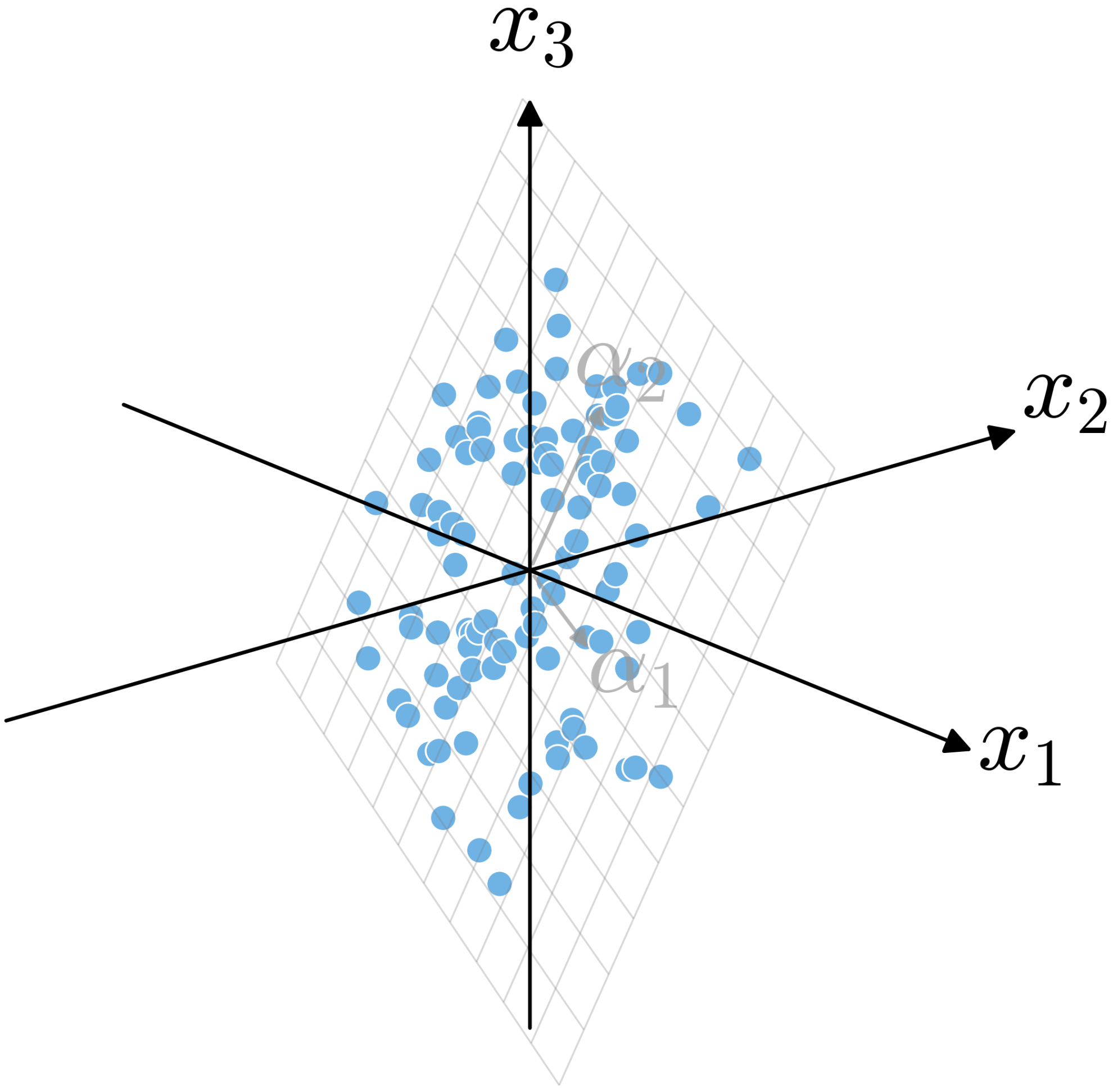} 
      }; \\
    };

    \node[anchor=north, yshift=-0.18cm] at (input.south) {(a) Input space};
    \node[anchor=north, yshift=-0.43cm] at (latent.south) {(b) Latent space};
    \node[anchor=north, yshift=0.01cm] at (recon.south) {(c) Reconstruction};

    \draw[->, thin, >=stealth] 
        ([xshift=1.35cm, yshift=-0.7cm]input.north)  
        to [out=45, in=135, looseness=1]          
        node[midway, above=0.1cm] {$z = B^\top x + r$}   
        ([xshift=3cm, yshift=-0.73cm]input.north);

    \draw[->, thin, >=stealth] 
        ([xshift=1.6cm, yshift=-0.36cm]latent.north)  
        to [out=45, in=135, looseness=1]          
        node[midway, above=0.1cm] {$\tilde{x} = Az$}   
        ([xshift=3.25cm, yshift=-0.39cm]latent.north);
  \end{tikzpicture}%
  }
  \caption{\textbf{Visualization of Adversarial PCA.} (a) Input data $x_i \in \R^{3}$. (b) The encoder $B^\top \in \R^{2 \times 3}$ maps data $x$ to the latent space $\mathcal{Z} = \R^{2}$, where an adversary $r$ is allowed to perturb within a limited budget (shaded rectangle). (c) The perturbed latent representation $\tilde{x}$ is mapped back to the original space via the orthonormal decoder $A \in \R^{3 \times 2}$. The resulting reconstructions lie on a $2$-dimensional plane in $\R^3$ defined by the columns of $A$, denoted $\alpha_1$ and $\alpha_2$.}
  \label{fig:3d_advpca}
\end{figure}
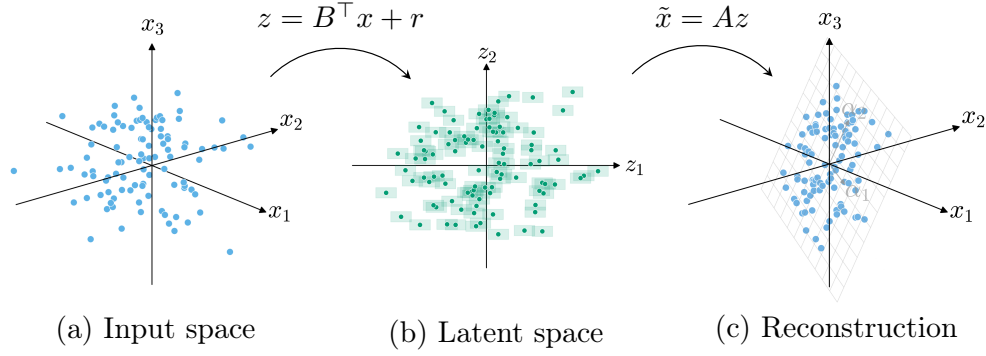

\section{Related Work}
\textbf{Sparse principal component analysis.} Thresholding (see Section~\ref{sec:background}) is an ad hoc way of achieving sparse principal components. \cite{cadimaLoadingCorrelationsInterpretation1995} find this unreliable, while \cite{johnstone_consistency_2009} find empirically that incorporating a threshold can be helpful in high dimensions. Equivalent to posing PCA as a reconstruction problem is the ``maximize-variance'' perspective. Imposing an $\ell_0$-restriction on the criterion is then a natural way of introducing sparsity, however it makes the problem doubly nonconvex \cite[ch.~8.2]{hastie_statistical_2015}. \cite{jolliffe_modified_2003}~relaxes this formulation when introducing the SCoTLASS procedure by replacing the $\ell_0$-constraint with $\ell_1$ and making use of the Lasso \cite{tibshirani_regression_1996}. \cite{df537581-8346-3a06-a12b-234347a76196} in turn relaxes SCoTLASS and propose the convex semidefinite program DSPCA. Shifting perspective, \cite{zou_sparse_2006} recast PCA as a regression-type reconstruction problem and extend it into an elastic net formulation \cite{zou_regularization_2005}, thereby contributing the Sparse PCA algorithm. More recent literature shifts from algorithm development towards establishing limits and rates for different sparse PCA methods, as seen in works from \cite{deshpande_information-theoretically_2014} and \cite{cai_sparse_2013}.

\textbf{Robust principal component analysis.} Robustness in the context of PCA typically refers to robustness against grossly corrupted observations, or outliers. Specifically, since PCA maximizes variance using squared errors, outliers heavily affect the empirical covariance by inflating the variance in certain directions \cite[ch.~10]{jolliffe_principal_2002}. Early methods replace the empirical covariance with robust alternatives that downweight extreme observations \cite{lerman_overview_2018}. \cite{maronna_robust_1976} proposes foundational M-estimators for multivariate scatter, with \cite{tyler_distribution-free_1987} alleviating previously necessary assumptions. Another common approach to robust PCA is based on projection pursuit (PP) \cite{friedman_projection_1974, huber_projection_1985}. \cite{hubert_robpca_2005} combine robust covariance estimation and PP when proposing the ROBPCA method, yielding more accurate estimates with improved computational properties. Other robust PCA methods aim to decompose the data matrix into a low-rank and a sparse component. \cite{candes_robust_2011} show that it is possible to recover the two components by solving a convex program called principal component pursuit (PCP). Although convex, the PCP program suffers from high computational cost \cite{bouwmansRobustPCAPrincipal2014}, which more recent methods focus on improving \cite{lin_linearized_2011, yuan_sparse_2009}. To conclude the robust PCA section, we want to emphasize that the goal of Adversarial PCA is not to be robust against outliers like above mentioned methods. \textit{Adversarial PCA simply uses tools from robust optimization to achieve a sparse PCA formulation}. This is an important distinction.

\textbf{Robust regression and adversarial training.} The robust regression framework has a long standing connection with sparse methods. Early papers connect robust linear regression with square-root Lasso~\cite{NIPS2008_24681928}, a version where the unsquared residual norm naturally scales out the unknown data noise level \cite{4561409b-2457-3af7-afb4-8a4e5ac01eec}. Similar tools are used to study adversarial training in linear models. \cite{ilyas_adversarial_2019, tsipras_robustness_2019} use linear models to explain robustness phenomena in neural networks; \cite{ribeiro_overparameterized_2023} study the impact of overparameterization on adversarial robustness; \cite{pmlr-v161-min21a} investigate how dataset size affects adversarial performance. Other works provide insight into the asymptotics of linear adversarial training \cite{javanmard_precise_2020, taheri_asymptotic_2022}. More recent contributions focus on the implicit regularization of adversarial linear regression \cite{ribeiro_regularization_2023, ribeiro_kernel_2025, xie_highdimensional_2024}.

\section{Background}\label{sec:background}
\textbf{Principal component analysis.} PCA plays a central role in machine learning as a standard dimensionality reduction technique. PCA \eqref{eq:pca} aims to find linear transformations $A,B \in \R^{d \times k}$ such that the reconstruction error over the data is minimized in the least squared sense. Specifically, given $k \leq d$, $B^\top \in \R^{k \times d}$ compresses the input $x \in \R^d$ into a lower-dimensional representation $z = B^\top x \in \R^k$. Then, $A \in \R^{d \times k}$ (approximately) recovers the original input from its compressed counterpart by reconstructing $x$ as $\tilde{x} = Az = AB^\top x \in \R^d$. Because $\tilde{x} = Az$, $\tilde{x}$ is constrained to the $k$-dimensional subspace defined by the columns of $A$ (see Figure~\ref{fig:3d_advpca}). An important characteristic of PCA is that any optimal solution to \eqref{eq:pca} must have matrices $A,B$ that coincide \cite[ch.~23.1]{shalev-shwartz_understanding_2014}, which implies that the solution is an orthogonal projection. This makes the subsequent solution simple and very elegant. To be precise, let $\mathcal{D} =\{ x_i \}_{i=1}^n$ be a set of $n$ zero-mean datapoints in $\R^d$, and let $X \in \R^{n \times d}$ be the corresponding data matrix. Then, PCA is solved simply by letting the columns of $A$ \emph{and} $B$ be the $k$ leading eigenvectors of $X^\top X$, or equivalently, the $k$ leading right-singular vectors of the singular value decomposition (SVD) of $X$. 

\textbf{Thresholding.} The issue with PCA is that the solution is notoriously dense \cite{jolliffe_principal_2002}, which only makes it suitable in data-rich settings ($n>d$). To circumvent this, a common approach is to force the lower-dimensional representations to be sparse combinations of the input variables. Perhaps the simplest way of achieving this (which typically serves as a baseline) is through thresholding: solve PCA and decouple $A$ and $B$ by thresholding $B$ according to $B_{ij} := B_{ij} \mathds{1} \{ \abs{B_{ij}} \geq \varepsilon \} \; \forall (i,j) \in [d] \times [k]$.

\textbf{Sparse PCA.} A more sophisticated approach to enforce sparsity in $B$ is to extend~\eqref{eq:pca} with $\ell_1$-penalties and make use of the Lasso. This is the core idea behind the Sparse PCA algorithm \cite{zou_sparse_2006}:
\begin{equation}\label{eq:spca} 
    \min_{A, B} \sum_{i=1}^n \norm[2]{x_i - AB^\top x_i}^2 + \sum_{j=1}^k \lambda_{j} \norm[1]{\beta_j} \quad \text{subject to} \quad A^\top A = I_k,
\end{equation}
 which in essence is a \textit{regularized} reconstruction problem where $\lambda_j \in \R^+$ $\forall j \in [k]$ are hyperparameters controlling the sparsity, and $A = [\alpha_1, \dots, \alpha_k]$ and  $B = [\beta_1, \dots, \beta_k]$ are both matrices in $\R^{d \times k}$. By solving~\eqref{eq:spca}, an approximate solution to PCA can be achieved with the added benefit of sparse input combinations. The full objective of Sparse PCA also includes an additional sum of ridge penalties that ensure uniqueness, but they are excluded in \eqref{eq:spca} for simplicity.

\textbf{Adversarial linear regression.} Adversarially trained linear regression has been shown to be a viable alternative to traditional Lasso and ridge regression, exhibiting similar regularization behavior while enjoying favorable properties \cite{ribeiro_regularization_2023, ribeiro_kernel_2025}. A defining feature of the formulation is that the inner maximization can be solved in closed-form, leaving a convex problem with an explicit regularization term. To be precise, adversarial linear regression satisfies
\begin{equation}\label{eq:adv_linear_reg}
    \min_\beta \frac{1}{n} \sum_{i=1}^n \max_{\norm{\Delta x_i} \leq \delta} \left(y_i - \beta^\top (x_i + \Delta x_i)\right)^2 = \min_\beta \frac{1}{n} \sum_{i=1}^n \left(\abs{y_i - \beta^\top x_i} + \delta \norm{\beta}_*\right)^2,
\end{equation}
where $\smash{\norm{\cdot}_*}$ denotes the dual norm of $\smash{\norm{\cdot}}$ and $\smash{\delta \geq 0}$. Indeed, comparing~\eqref{eq:adv_linear_reg} to a traditional method such as the Lasso, with objective 
\begin{equation*}
\frac{1}{n} \sum_{i=1}^n \left(y_i - \beta^\top x_i\right)^2 + \lambda \norm{\beta}_1
\end{equation*} 
for $\smash{\lambda \in \R^+}$, the main connection is clear: the norm-penalty is now \emph{inside} the square. Another feature is that the type of penalty directly depends on the norm defining the perturbation set. Lastly, and perhaps most importantly, a key advantage of the adversarial formulation is that it has been shown to attain near-oracle performance for a choice of $\delta$ independent of the data noise level. This is analogous to the square-root Lasso~\cite{4561409b-2457-3af7-afb4-8a4e5ac01eec}.

\section{Proposed Formulation}\label{sec:proposed_formulation}
Building on the adversarial linear regression framework, we propose a min-max formulation that naturally induces sparsity by optimizing the reconstruction objective against worst-case perturbations. This is formulated as an extension of the original PCA formulation~\eqref{eq:pca}, where the objective is augmented with adversarial perturbations acting directly in the latent space:
\begin{equation}\label{eq:advpca2}
\min_{A,B} \sum_{i=1}^n \underset{r_i \in \Omega_{\delta}}{\text{max}}  \norm[2]{x_i - A(B^\top x_i + r_i)}^2 \quad \text{subject to} \quad A^\top A = I_k,
\end{equation}
where $A = [\alpha_1, \dots, \alpha_k]$ and  $B = [\beta_1, \dots, \beta_k]$ are both matrices in $\R^{d \times k}$. We define $\Omega_\delta$ to be an axis-aligned hyperrectangle (or weighted $\ell_\infty$-norm) in the latent space, given by
\begin{equation}\label{eq:hyperrect}
\Omega_{\delta} = \prod_{j=1}^k \left[-\delta_j \norm{\beta_j}_1, \delta_j \norm{\beta_j}_1\right] \subset \R^k,
\end{equation}
where $\{\delta_1, \dots, \delta_k\}$ is a set of nonnegative adversarial radii. While this choice of set may initially appear unmotivated, subsequent sections will demonstrate its necessity for the algorithm's efficacy. For now, we summarize the motivation with (i) perturbations in the latent space in combination with the axis-aligned geometry of the set (product of intervals) decomposes \eqref{eq:advpca2} into \emph{independent adversarial linear regression problems}, making the inner maximum solvable in closed-form; and (ii) the $\ell_1$-norm is what promotes sparsity in $\beta$, and comes from the relationship between input and latent space perturbations. We refer to Section~\ref{sec:perturb_set} for a detailed discussion on this manner.
\subsection{Closed-Form Reformulation}\label{sec:closed-formula}
By design, the geometry of $\Omega_\delta$ ensures a closed-form solution for the inner maximum in~\eqref{eq:advpca2}:
\begin{proposition}\label{prop:1}
    Let $x \in \mathbb{R}^d$ and $\Omega_\delta \subset \R^k$ be given by~\eqref{eq:hyperrect}. Furthermore, let $A, B \in \mathbb{R}^{d \times k}$ such that $A^\top A = I_k$. Then, the inner maximization of~\eqref{eq:advpca2} can be solved in closed-form:
    \begin{equation*}
        \underset{r \in \Omega_{\delta}}{\text{max}} \; \norm[2]{x - A(B^\top x + r)}^2 = \norm{x - A A^\top x}_2^2 + \sum_{j=1}^k \left( |\beta_j^\top x - \alpha_j^\top x| + \delta_j \norm[1]{\beta_j} \right)^2,
    \end{equation*}
    with maximizer $\hat{r}$ given entry-wise by $\hat{r}_j = \text{sign}(\beta_j^\top x - \alpha_j^\top x) \, \delta_j \norm{\beta_j}_1$ for all $j =   1,\dots,k$.
\end{proposition}
Proposition \ref{prop:1} shows that bounded latent perturbations have a regularizing effect on $B$, controlled by the adversarial radii $\{ \delta_1,\dots, \delta_k \}$. Notably, the closed form decomposes into the standard PCA reconstruction error $\norm{x - AA^\top x}_2^2$ plus an adversarial penalty regularizing $B$, and recovers ordinary PCA when \mbox{$\delta_1 = \dots = \delta_k = 0$}. Using Proposition~\ref{prop:1}, we equivalently reformulate~\eqref{eq:advpca2}:
\begin{equation}\label{eq:final_advpca}
    \min_{A,B} \sum_{i=1}^n \biggl( \norm{x_i - A A ^\top x_i}_2^2 + \sum_{j=1}^k \left( \lvert \beta_j^\top x_i - \alpha_j^\top x_i \rvert + \delta_j \norm[1]{\beta_j} \right)^2 \biggr) \quad \text{subject to} \quad A^\top A = I_k.
\end{equation}
A key property of \eqref{eq:final_advpca} is that for a fixed $A$, the minimization over $B$ decomposes into $k$ \textit{independent} optimization tasks, one for each column. That is, for the $j$th component, \eqref{eq:final_advpca} can be written as 
\begin{equation*}
\argmin_{\beta_j} \sum_{i=1}^n \left( \lvert \alpha_j^\top x_i - \beta_j^\top x_i \rvert + \delta_j \norm[1]{\beta_j} \right)^2
\end{equation*} 
which exactly recovers the robust objective of adversarial linear regression~\eqref{eq:adv_linear_reg} with the exception that $\beta$ is fitted to $\alpha^\top x$ instead of a supervised target~$y$. Moreover, the $\ell_1$-penalty on $\beta$ naturally promotes the desired sparsity. In the following section, we show how to leverage this decoupled structure to solve~\eqref{eq:final_advpca} through a block coordinate descent optimization scheme.

\subsection{Proposed Solver}\label{sec:solver}
We propose an iterative solver, alternating between solving for $A$ and $B$. If $A$ is considered fixed,~\eqref{eq:final_advpca} reduces to $k$ one-dimensional adversarial linear regression problems that are solved independently. On the other hand, if $B$ and the adversarial perturbations are fixed, $A$ admits a closed-form solution.

\textbf{Fix $\boldsymbol{A}$, solve for $\boldsymbol{B}$.} First, $A$ is initialized to the solution of ordinary PCA, i.e., $A := V_{1:k}$ where $V$ is given by the SVD of $X$ and the subscript denotes the first $k$ columns. With $A$ fixed,~\eqref{eq:final_advpca} simplifies~to
\begin{equation}\label{eq:fixA}
    \min_{B = [\beta_1, \dots \beta_k]} \sum_{j=1}^k \underbrace{\sum_{i=1}^n \left( \lvert \beta_j^\top x_i - \alpha_j^\top x_i \rvert + \delta_j \norm[1]{\beta_j} \right)^2}_{f_j(\beta_j)},
\end{equation}
where the sums have been swapped (without issue since all elements are non-negative). By denoting the inner sum $f_j(\beta_j)$, it is easy to see that the minimization separates over $j$. Therefore, \eqref{eq:fixA} equals $\smash{\sum_{j=1}^k \min_{\beta_j} f_j(\beta_j)}$ where each $\min_{\beta_j} f_j(\beta_j)$ is an adversarial linear regression problem trying to fit $\alpha^\top x$ with an $\ell_1$-norm penalty on $\beta$. This subproblem is convex, and is therefore readily solved using, e.g., \texttt{cvxpy} \cite{agrawal2018rewriting, diamond2016cvxpy}. However, to speed up computation we use the solver proposed by~\cite{ribeiro_efficient_2025} (referred to as \texttt{eta\_trick()} in Algorithm~\ref{alg:advpca}), which is tailored to the linear adversarial training problem. This allows for more efficient optimization; see Appendix~\ref{app:eta-trick} for details.

\noindent
\begin{minipage}[b]{0.5\textwidth}
\textbf{Freeze the adversary $\boldsymbol{R}$.} When solving for $A$, we first return to the primal problem formulation~\eqref{eq:advpca2} and write it in matrix form as 
\begin{equation}\label{eq:matrixform1}
\begin{aligned}
&\min_{A,B} \underset{R \in \Omega_{\delta}^n}{\text{max}}  \norm{X - (XB + R) A^\top}_F^2 \quad \\
&\; \text{subject to} \quad A^\top A = I_k,
\end{aligned}
\end{equation}
where $\Omega_{\delta}^n = \prod_{i=1}^n \Omega_\delta$ and $\smash{\norm{\cdot}_F}$ denotes the Frobenius norm. Further, we also extend the maximizing argument  $\hat{r}$ in Proposition~\ref{prop:1} into matrix form according to \mbox{$\widehat{R} = \text{sign}(X(B-A)) D \in \R^{n \times k}$}, where the sign is operating element-wise over the matrix $X(B-A) \in \R^{n \times k}$ and right-multiplying with $D = \mathrm{diag}(\delta_1 \norm{\beta_1}_1, \dots, \delta_k \norm{\beta_k}_1)$ simply scales each column with its diagonal. Substituting the maximizing argument $\smash{\widehat{R}}$ back into~\eqref{eq:matrixform1} and only optimizing over $A$ then yields 
\begin{equation}\label{eq:matrixform2}
\begin{aligned}
&\min_{A} \snorm{X - (XB + \widehat{R}) A^\top}_F^2 \quad \\
&\; \text{subject to} \quad A^\top A = I_k.
\end{aligned}
\end{equation}    
    \vspace{0pt}
\end{minipage}
\hfill
\raisebox{0.1cm}{
\begin{minipage}[b]{0.45\textwidth}
    \begin{algorithm}[H] 
        \caption{AdvPCA}
        \label{alg:advpca}
        \KwInput{$X$; $\delta_1, \dots, \delta_k$; $\epsilon=0.5$}
        \KwInitialize{$U, \Sigma, V^\top \gets \mathtt{SVD}(X);$ $A \gets V_{1:k}$}
        
        \Repeat{\texttt{StopCondition}}{
          \Comment{Fix A, solve for B}
          \For{$j \gets 1$ \KwTo $k$}{
          $B_{j} \gets \mathtt{eta\_trick}(X, A_j, \delta_j )$\;
          }
          \Comment{Solve for  R}  
          $E \gets X(B - A)$\;
          \For{$j \gets 1$ \KwTo $k$}{
          $R_{j} \gets \mathtt{sign}( E_{j} ) \delta_j \lVert B_j \rVert_1$\;
          }
          \Comment{Fix B, solve for A}
          $U,\Sigma,V^\top \gets \mathtt{SVD}(X^\top (XB + R))$\;
          $A_{\mathrm{upd}} \gets U_{1:k}V^\top$\;
          $A_{\mathrm{smooth}} \gets \epsilon A + (1-\epsilon) A_{\mathrm{upd}}$\;
          $ U,\Sigma,V^\top \gets \mathtt{SVD}(A_{\mathrm{smooth}})$\;
          $A \gets U_{1:k} V^\top$\;
         } 
        \KwRet{$A,B$}
    \end{algorithm}
    \vspace{0pt}
\end{minipage}%
}

Now, Equation~\eqref{eq:matrixform2} has a closed-form solution given $XB+\widehat{R}$ is \textit{fixed} with respect to $A$, which evidently is not the case. To circumvent this dependency, we freeze the adversaries when updating~$A$. More precisely, letting $\widebar{A}$ denote $A$ from the previous optimization step, we decouple the adversary from the current optimization variable $A$ by computing $\smash{\widehat{R}}$ using $\widebar{A}$ and the newly updated $B$. This alternating min-max strategy is a well-established practice in robust optimization, mirroring standard adversarial training routines such as projected gradient descent (PGD)  \cite{madry_deep_2018, pmlr-v97-wang19i}.

\textbf{Fix $\boldsymbol{B}$, solve for $\boldsymbol{A}$.} With $B$ fixed and $\widehat{R}$ constant w.r.t. $A$,~\eqref{eq:matrixform2} becomes an orthogonal Procrustes problem \cite[ch.~6.4.1]{golub_matrix_2013} which we solve in the next proposition:
\begin{proposition}\label{prop:2}
    Let $\widehat{R} = \text{sign}(X(B-\widebar{A})) D \in \R^{n \times k}$ and let $X^\top (XB + \widehat{R}) = U \Sigma V^\top$ be the singular value decomposition. Then, $\widehat{A} = U_{1:k}V^\top$ is the global minimizer of~\eqref{eq:matrixform2}, where $U_{1:k}$ denotes the first $k$ columns of $U$.
\end{proposition}

We now have all the pieces to present the full Adversarial PCA algorithm; see Algorithm~\ref{alg:advpca}. We assume $X$ to be centered (zero-mean columns) and let matrices with subscripts indicate the corresponding column. To ensure stable convergence for larger $\delta$, we dampen the global $A$-update by taking a weighted average of the current and updated $A$ matrix followed by an orthogonal projection back onto the manifold of orthogonal matrices. See Appendix \ref{app:smoothing} for details.

\section{Relationship with Input Space Perturbations}\label{sec:perturb_set}
We establish the relationship between input and latent space perturbations to motivate our specific choice of $\Omega_\delta$. We provide intuition behind perturbing the latent representations, demonstrating that unlike input space perturbations, this approach yields a formulation consistent with ordinary PCA.

\textbf{The equivalence in one dimension.} Typically, adversarial perturbations are introduced directly in the input space, constrained by an upper bound on the norm. For the case of linear regression, as shown in~\eqref{eq:adv_linear_reg}, this amounts to augmenting the standard objective with a local maximization:
\begin{equation}\label{eq:linregtoadv}
    \min_{\beta \in \R^d} \frac{1}{n} \sum_{i=1}^n \left(y_i - \beta^\top x_i\right)^2 \xrightarrow{\text{robust}} \min_{\beta \in \R^d} \frac{1}{n} \sum_{i=1}^n { \color{myblue} \max_{\norm{\Delta x_i} \leq \delta} } \left(y_i - \beta^\top (x_i + { \color{myblue} \Delta x_i })\right)^2.
\end{equation}
To simplify the analogy to PCA, we consider a special case of the PCA objective~\eqref{eq:pca} where $k=1$, meaning the matrices $A, B \in \R^{d \times k}$ reduce to vectors $\alpha, \beta \in \R^d$. The corresponding, perhaps most standard way of introducing adversarial perturbations in PCA would then be given by
\begin{equation}\label{eq:pcatoadv}
    \min_{\alpha, \beta \in \R^d} \sum_{i=1}^n \norm[2]{x_i - \alpha \beta^\top x_i }^2 \xrightarrow{\text{robust}} \min_{\alpha, \beta \in \R^d} \sum_{i=1}^n {\color{myblue} \max_{\norm{\Delta x_i} \leq \delta}} \norm[2]{x_i - \alpha \beta^\top (x_i + { \color{myblue}\Delta x_i }) }^2.
\end{equation}

Mimicking the linear regression case, the target input is left undisturbed and the adversary is only allowed to act on the input being reconstructed. To connect this standard input space formulation in~\eqref{eq:pcatoadv} to the general latent AdvPCA objective in~\eqref{eq:advpca2}, let $\mathcal{L}(\alpha, \beta)$ denote the inner maximum of~\eqref{eq:pcatoadv} for a single sample~$x$. Through Proposition~\ref{prop:3}, we establish that this natural way of perturbing in the input space is exactly equivalent to perturbing in the latent space.
\begin{proposition}\label{prop:3}
    Let $x \in \R^d$ and  $\alpha, \beta \in \mathbb{R}^d$ such that $\norm[2]{\alpha} = 1$. Further, let $\mathcal{L}(\alpha, \beta)$ denote the worst-case adversarial reconstruction loss. Then, for an adversarial radius $\delta \geq 0$ and conjugate exponents $p, q \geq 1$ satisfying $1/p + 1/q = 1$, the following expressions for $\mathcal{L}(\alpha, \beta)$ are equal:

    (a) $\mathcal{L}(\alpha, \beta) = \max_{\norm[p]{\Delta x} \leq \delta} \norm[2]{x - \alpha \beta^\top(x + \Delta x)}^2$. \quad (Input space perturbation)

    (b) $\mathcal{L}(\alpha, \beta) = \max_{|r| \leq \delta \norm[q]{\beta}} \norm[2]{x - \alpha (\beta^\top x + r)}^2$. \quad (Latent space perturbation)

    (c) $\mathcal{L}(\alpha, \beta) = \norm[2]{x - \alpha \alpha^\top x}^2 + \left(|\beta^\top x - \alpha^\top x| + \delta \norm[q]{\beta}\right)^2$. \quad (Closed-form maximum) \\
 Additionally, if $\alpha = \beta$ is imposed, then $\mathcal{L}(\alpha, \beta) = \norm[2]{x - \beta \beta^\top x}^2 + \delta^2 \norm[q]{\beta}^2$.
\end{proposition}
Parts (a) and (b) establish the core equivalence; the substitution $r = \beta^\top \Delta x$ projects the adversary into a single dimension and Hölder's inequality~\cite[ch.~9]{steele_cauchy-schwarz_2010} then translates the norm constraint. This bounds the latent perturbation as $|r| = |\beta^\top \Delta x| \leq \|\Delta x\|_p \|\beta\|_q \leq \delta \|\beta\|_q$ and shows the intuition behind letting $\Omega_\delta$ be a product of independent intervals $[-\delta \norm{\beta}_q, +\delta \norm{\beta}_q]$ (a hyperrectangle). Part~(c) subsequently evaluates this maximum in closed form, yielding an objective convex in $\beta$ where the adversarial radius $\delta$ explicitly dictates the amount of regularization. This behavior is most apparent in the special case where $\alpha = \beta$. Finally, to promote sparsity in $\beta$, we set $q=1$ in the proposed formulation \eqref{eq:advpca2}, inducing an $\ell_1$-norm penalty. By conjugate exponents, this mathematically corresponds to defending against an $\ell_\infty$-norm adversary in the original input space.

\textbf{Inconsistency when scaling to multiple dimensions.} A defining feature of PCA is that the solution stays consistent across multiple values of $k$. To be precise, solving PCA for $k$ components and then another time for $k+1$ components (using the same data), the first $k$ components of the two solutions should be consistent with each other. Now, consider generalizing part~(a) of Proposition~\ref{prop:3} to multiple dimensions. For $k>1$, the vector $\beta \in \R^d$ is extended to the matrix $B \in \R^{d \times k}$ meaning the adversary becomes a vector $r = B^\top \Delta x \in \R^k$. Thus, because $\Delta x$ maps to all $k$ dimensions of the latent representation, the optimal adversarial direction for $\Delta x$ will generally depend on \textit{all} $k$ columns of $B$ \textit{simultaneously}. Hence, the solution of Adversarial PCA would change with $k$, which is not in line with ordinary PCA. Instead, by perturbing in the latent space, independence between components is kept intact meaning solving for different $k$ has no impact on the solution. 

\section{Properties of the Solution}\label{sec:properties_of_sol}
Many structural properties of adversarial linear regression naturally extend to Adversarial PCA. To isolate this behavior, we examine the step in the solver when $A$ is fixed, which reduces the problem to $k$ independent adversarial linear regression problems as was shown in Equation~\eqref{eq:fixA}. For the analysis, we define the (empirical) Adversarial PCA risk as
\begin{equation}\label{eq:advpca_risk}
    \widehat{\mathcal{R}}_\delta(\beta) = \sum_{i=1}^n \left( \lvert \alpha^\top x_i - \beta^\top x_i \rvert + \delta \norm[1]{\beta} \right)^2,
\end{equation}
where we drop the subscript $j$ for notational convenience since all $k$ subproblems have identical properties. Let $\widehat{\beta} \in \argmin_\beta \widehat{\mathcal{R}}_\delta(\beta)$ denote the resulting estimator. The behavior of the estimator is governed by the adversarial radius $\delta \geq 0$. Specifically, Proposition~\ref{prop:4} and Proposition~\ref{prop:5} characterize its behavior by quantifying two extremes of the regularization path. Proposition~\ref{prop:4} defines the regime where the adversary is sufficiently weak. This allows the minimum $\ell_1$-norm interpolator to remain optimal. Conversely, Proposition~\ref{prop:5} establishes the threshold at which the adversary becomes so strong that the trivial zero vector is the only optimal solution. These results are adapted from \cite[thm.~1]{ribeiro_regularization_2023} and \cite[prop. 3]{ribeiro_regularization_2023}, respectively.
\begin{proposition}\label{prop:4}
    Assume $n < d$ and that the data matrix $X \in \R^{n \times d}$ has full row rank. Further, let $\hat{\nu}$ denote the solution of $\max_{\norm{X^\top \nu}_\infty \leq 1} \nu^\top X\alpha$ and let $\bar{\delta}$ be given by
    \begin{equation*}
        \bar{\delta} = \frac{1}{n}\norm{\hat{\nu}}_\infty^{-1}.
    \end{equation*}
    Then, the minimum $\ell_1$-norm interpolator $\widehat{\beta} = \argmin_{X\alpha = X\beta} \norm{\beta}_1$ minimizes the Adversarial PCA risk~\eqref{eq:advpca_risk} if and only if $\delta \in [0, \bar{\delta}\,]$.
\end{proposition}

\begin{proposition}\label{prop:5}
    The zero solution $\widehat{\beta} = 0$ minimizes the Adversarial PCA risk~\eqref{eq:advpca_risk} if and only~if 
    \begin{equation*}
        \delta \geq \delta_{\mathrm{max}} = \frac{\norm{X^\top X \alpha}_\infty}{\norm{X \alpha}_1}.
    \end{equation*}
    In addition, let $X = U \Sigma V^\top$ be the singular value decomposition and consider the case where $\alpha$ is set to the $j$th standard principal component, namely $\alpha := v_j$ with corresponding left singular vector $u_j$ and singular value $\sigma_j$. Then, $\delta_{\mathrm{max}}$ simplifies to $\delta_{\mathrm{max}} = \sigma_j \frac{\norm[\infty]{v_j}}{\norm[1]{u_j}}$.
\end{proposition}

\noindent
\begin{minipage}[t]{0.52\textwidth}
  \vspace{0pt} 
   Figure~\ref{fig:wrapped_phase_trans} visualizes the two regimes induced by $\bar{\delta}$ and $\delta_{\textrm{max}}$. We fit AdvPCA using $k=1$ over a range of $\delta$s on $n=500$ datapoints in $\R^{250}$. Worth noting is that when $\delta \leq \bar{\delta}$ in accordance with Proposition~\ref{prop:4}, since \mbox{AdvPCA} is initialized to PCA, it recovers the same $\beta$ as applying basis pursuit to the principal component. \\
   
   To better understand the generalization of Adversarial PCA, we bound the estimation error of our empirical estimator $\widehat{\beta}$ under the spiked covariance data model (see Appendix~\ref{app:spiked_cov}). The following proposition provides a deterministic upper bound on this error, measured in the squared $\Sigma_*$-norm.
\end{minipage}
\hfill
\begin{minipage}[t]{0.43\textwidth}
  \vspace{0pt}
  \centering
  \includegraphics[width=\textwidth]{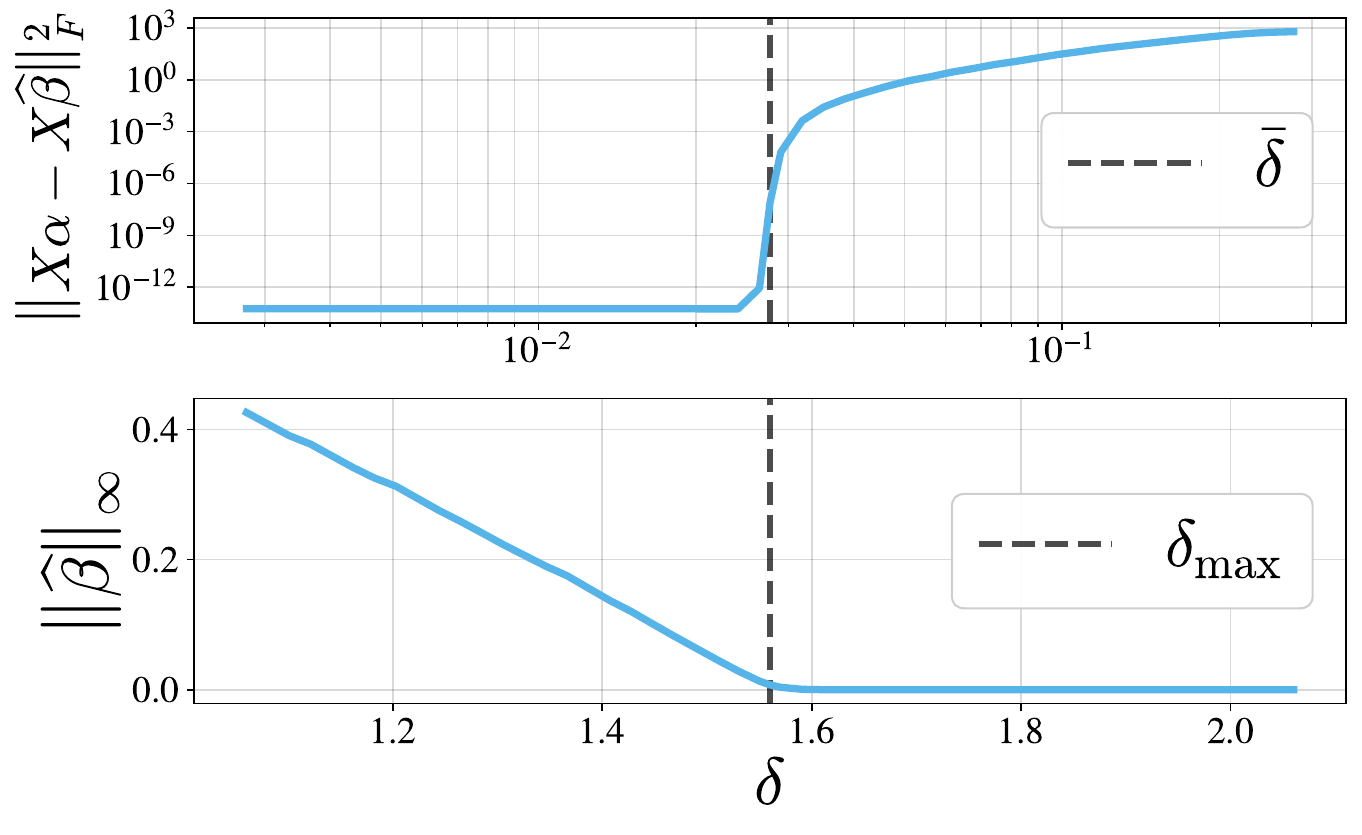}
  \captionof{figure}{\textbf{\protect\boldmath Phase transitions $\bar{\delta}$, $\delta_{\mathrm{max}}$.} \mbox{(\textit{Bottom})} $\widehat{\beta}=0$ for $\delta \geq \delta_{\textrm{max}}$. (\textit{Top}) $\widehat{\beta}$~interpolates $X\alpha$ when $\delta \in [0, \bar{\delta}\,]$.}
  \label{fig:wrapped_phase_trans}
\end{minipage}
\begin{proposition}\label{prop:6}
    Let $\widehat{\beta} \in \argmin_{\beta} \widehat{\mathcal{R}}_\delta(\beta)$ minimize the Adversarial PCA risk \eqref{eq:advpca_risk}, and let $\beta_*$ denote the true parameter with parameter error $\widehat{\Delta} = \widehat{\beta} - \beta_*$. Further, assume $X \in \R^{n \times d}$ follows the spiked covariance model with true covariance $\Sigma_*$ and empirical covariance $\widehat{\Sigma} = \frac{1}{n} X^\top X$. Then, for any fixed vector $\alpha$ and $\delta \geq 0$, the following bound on the error $\widehat{\Delta}$ holds:
    \begin{align*} 
    \scaleeq{\snorm{\widehat{\Delta}}_{\Sigma_*}^2 \leq \snorm{\widehat{\Delta}}_{\Sigma_* - \widehat{\Sigma}}^2 + \frac{4}{n} (\alpha - \beta_*)^\top X^\top X \widehat{\Delta} + \frac{4 \delta}{n} \norm[1]{X(\alpha - \beta_*)} \left( \norm[1]{\beta_*} - \snorm{\widehat{\beta}}_1 \right) + 2 \delta^2\left( \norm[1]{\beta_*}^2 + \snorm{\widehat{\beta}}_1^2 \right)}{0.87}.
    \end{align*}
\end{proposition}
\textit{Remark.} The utility of the bound in Proposition~\ref{prop:6} is that it isolates different sources of estimation error and demonstrates how the estimator behaves as $n$ grows. Specifically, we show in the proof given in Appendix \ref{app:proof:prop6} that \mbox{$\snorm{\widehat{\Delta}}_{\Sigma_* - \widehat{\Sigma}}^2 \leq \snorm{\Sigma_* - \widehat{\Sigma}}_\infty \snorm{\widehat{\Delta}}_1^2$} where the covariance gap shrinks at a rate of \mbox{$\snorm{\Sigma_* - \widehat{\Sigma}}_\infty \lesssim \sqrt{\ln(d)/n}$}, meaning this term vanishes asymptotically. Similarly, the term $\frac{1}{n}\norm[1]{X(\alpha - \beta_*)}$ concentrates around a constant. Therefore, by setting the radius $\delta \propto 1/\sqrt{n}$, all terms on the right-hand side vanish with increasing $n$ except for $\frac{4}{n} (\alpha - \beta_*)^\top X^\top X \widehat{\Delta}$. This term persists as a constant error floor, representing the geometric projection of the estimation error onto the adversarial direction $(\alpha - \beta_*)$ evaluated over the empirical data.

\section{Numerical Experiments}\label{sec:experiments}
While AdvPCA is valid for any $\delta \geq 0$, practical use requires \mbox{$\delta \in [0, \delta_{\max})$}. Guided by the theoretical error decay in the previous remark, we scale this maximum bound by setting $\delta := \delta_{\max} \sqrt{\ln (d) / n}$ throughout the experiments, unless stated otherwise. In Appendix~\ref{app:sec:multicomp}, we present additional results for other choices of $\delta$.
We compare AdvPCA to ordinary PCA, thresholded PCA and Sparse PCA. To establish straightforward baselines without relying on hyperparameter tuning (as this is not an obvious procedure), we run Sparse PCA using the default \texttt{scikit-learn} configuration. Similarly, for thresholded PCA, we simply apply a threshold of $\varepsilon=0.1$, which we find as a natural heuristic for truncating unit-norm principal components.

\textbf{Synthetic data from the spiked covariance model.} We evaluate our method using the spiked covariance model from \cite[ch. 8.2.2]{wainwright_high-dimensional_2019}, described in Appendix~\ref{app:spiked_cov}. In short, let $\beta_{*j} := e_j \in \mathbb{R}^d$ for $j = 1,\dots,k$ be the true principal directions with associated eigengaps $\tau_j \geq 0$. We form the covariance matrix $\Sigma_* = I_d + \sum_{j=1}^k \tau_j \beta_{*j} \beta_{*j}^\top$, ensuring all eigenvalues equal one except in the $k$ spiked directions. Finally, we draw $n$ datapoints $x_i \sim \mathcal{N}(0,\Sigma_*)$ to form $X \in \mathbb{R}^{n \times d}$, from which we aim to recover the true directions $\beta_{*1}, \dots,\beta_{*k}$; see Figure~\ref{fig:sparsevsdense} and Figure~\ref{fig:multi-comp-recovery}. We supplement these experiments with material in Appendix~\ref{app:smoothing} and Appendix~\ref{app:sec:spiked_cov}.

\begin{figure}[t]
     \centering
     \begin{subfigure}[b]{0.425\textwidth}
         \centering
         \includegraphics[width=1\textwidth]{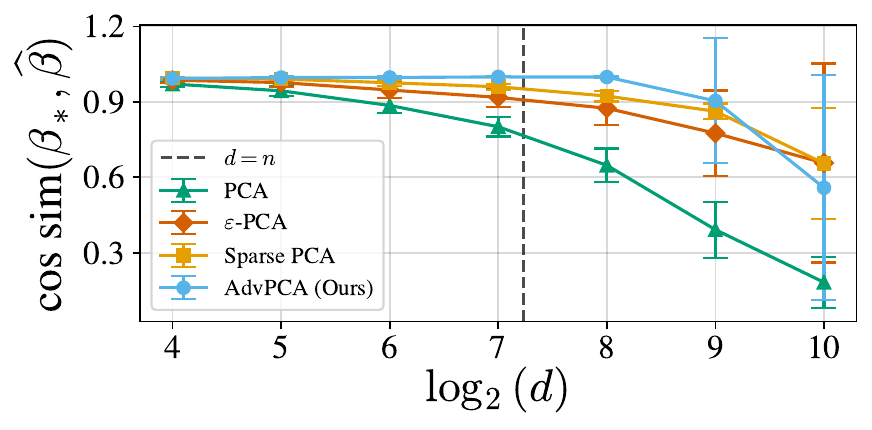}
     \end{subfigure}
     \hspace{0.5cm}
     \begin{subfigure}[b]{0.425\textwidth}
         \centering
         \includegraphics[width=1\textwidth]{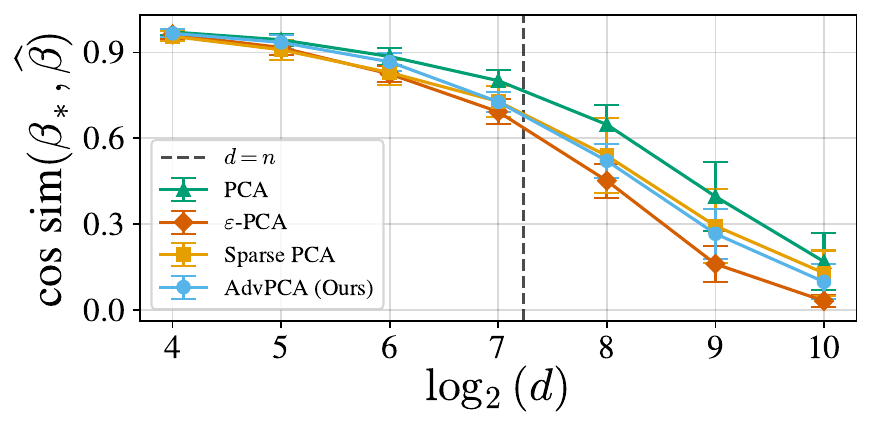}
     \end{subfigure}
     
     \caption{\textbf{Sparse versus dense system.} (\textit{Left}) We draw $n=150$ datapoints from the spiked covariance model with $\beta_* = e_1$ as the true sparse spike direction, and evaluate the recovery of this direction from the generated data by fitting the models using $k=1$. We measure the absolute cosine similarity over various $d$ for a fixed $\tau=2.5$. We report mean $\pm 1$ standard deviation over 20 independent draws. (\textit{Right}) The data matrix $X$ and the true direction $\beta_*$ are subjected to a random orthogonal rotation, which redistributes the signal across all dimensions and makes $\beta_*$~dense.}
     \label{fig:sparsevsdense}
\end{figure}
\begin{figure}[H]
     \centering
     \begin{subfigure}[b]{0.425\textwidth}
         \centering
         \includegraphics[width=1\textwidth]{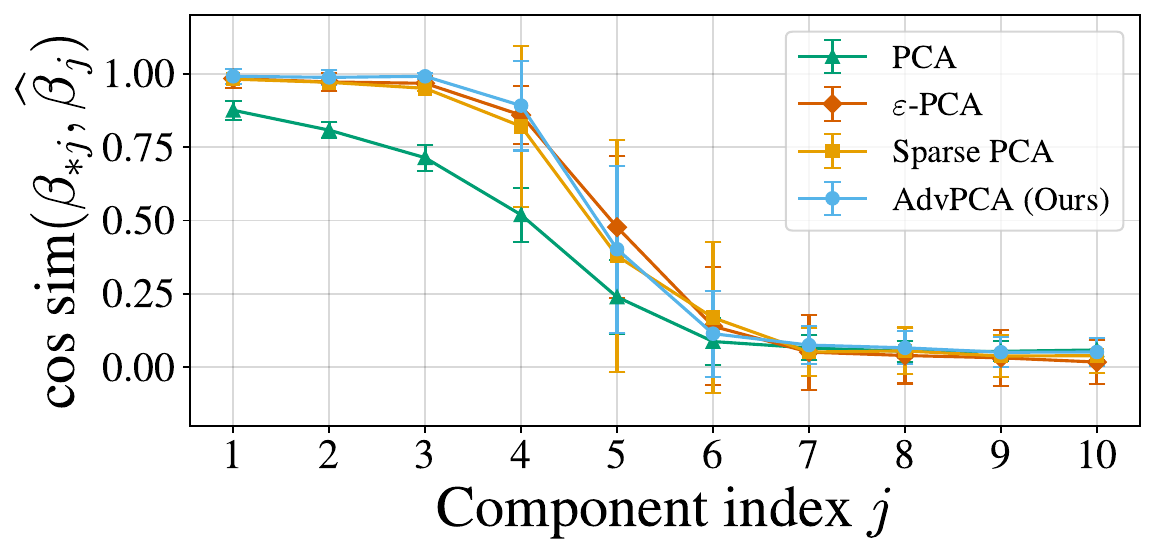}
     \end{subfigure}
     \hspace{0.5cm}
     \begin{subfigure}[b]{0.425\textwidth}
         \centering
         \includegraphics[width=1\textwidth]{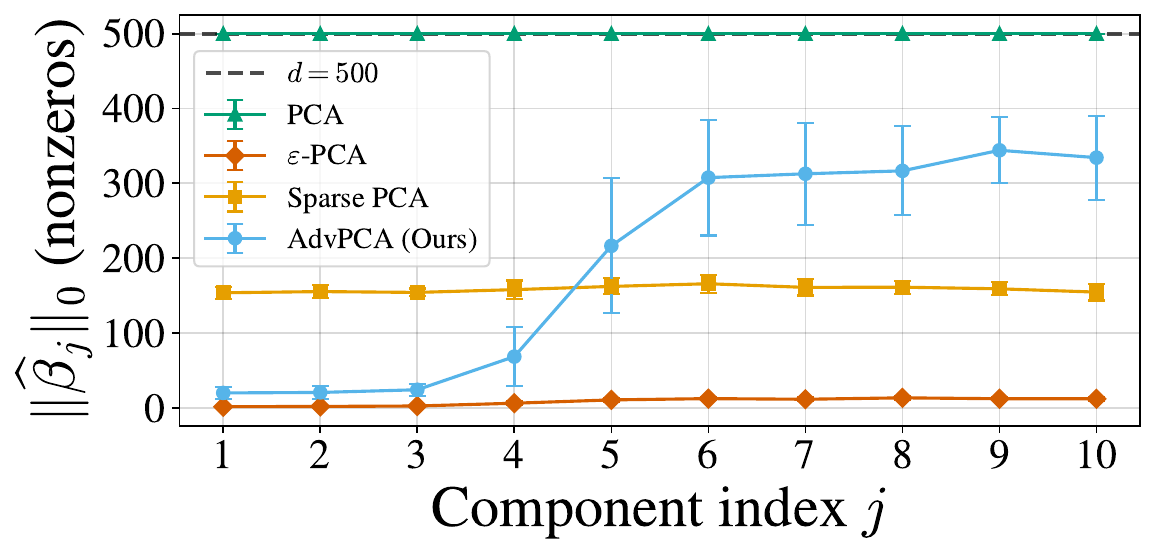}
     \end{subfigure}
     \caption{\textbf{Multi-component recovery.} We evaluate the methods on $k=10$ principal components where $n=500$ datapoints in $\mathbb{R}^{500}$ are drawn from the spiked model, and report mean $\pm 1$ standard deviation over 20 independent draws. We let the eigengaps $\tau_j$ decrease with $j$, making successive directions harder to recover (see Appendix~\ref{app:sec:multicomp} for details). (\textit{Left}) Absolute cosine similarity with true $\beta_{*j} = e_j$. (\textit{Right}) Corresponding $\ell_0$-norm showing sparsity of the solution.}
     \label{fig:multi-comp-recovery}
\end{figure}
\textbf{An application in genomics.} To complement the synthetic data, we analyze the high-dimensional MAGIC wheat dataset \cite{scott_limited_2021}, where the input features are gene sequence variations (SNPs). Sparsity is highly desirable in this context to identify distinct, co-varying genetic markers rather than dense combinations of the entire genome. As shown in Figure~\ref{fig:magic}, AdvPCA achieves an out-of-sample reconstruction error on par with the baseline methods while yielding sparser and more interpretable components. Further details on this experiment can be found in Appendix~\ref{app:sec:magic}.
\begin{figure}[H]
     \centering
     \begin{subfigure}[b]{0.365\textwidth}
         \centering
         \includegraphics[width=\textwidth]{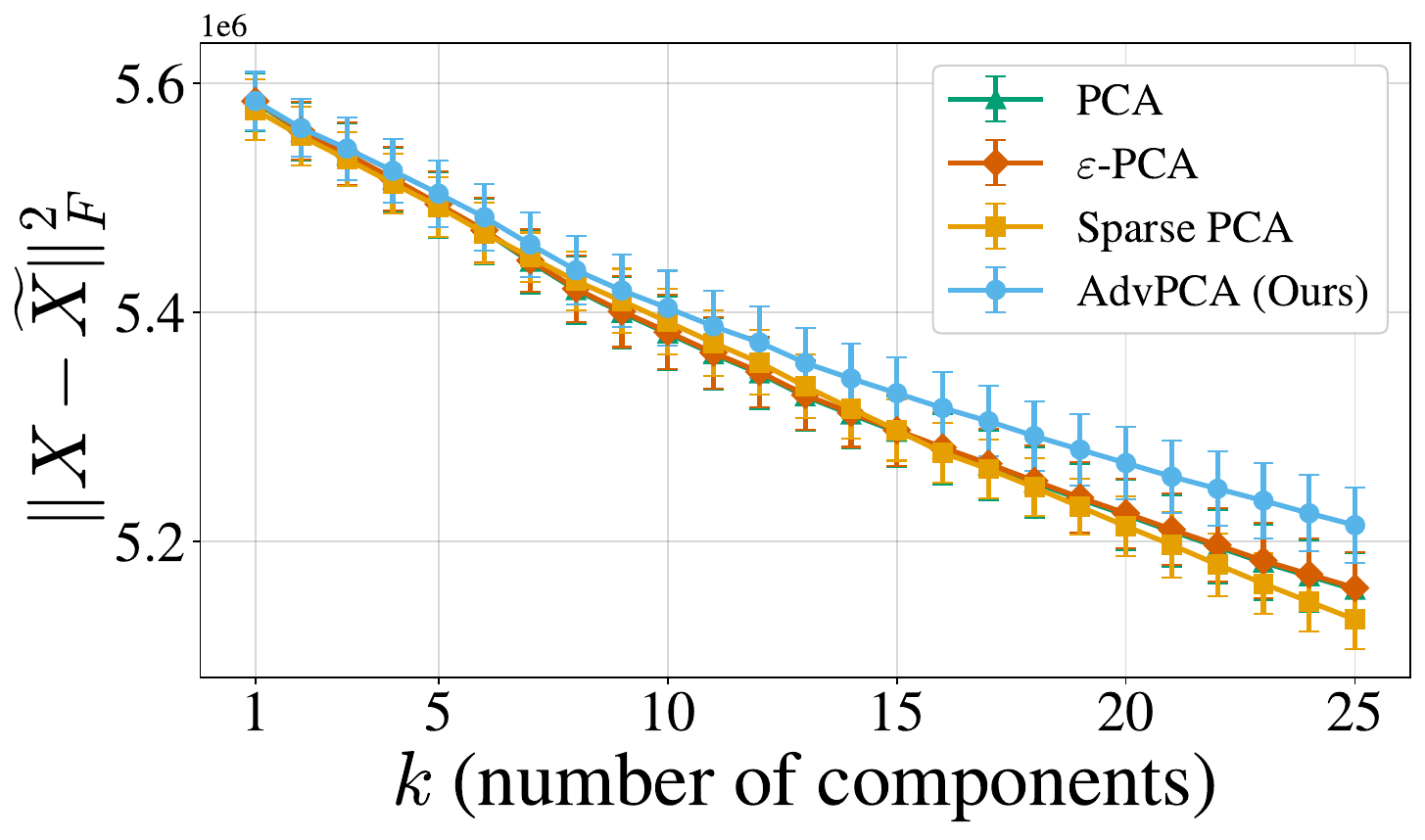}
     \end{subfigure}
     \hspace{0.5cm}
     \begin{subfigure}[b]{0.355\textwidth}
         \centering
         \includegraphics[width=\textwidth]{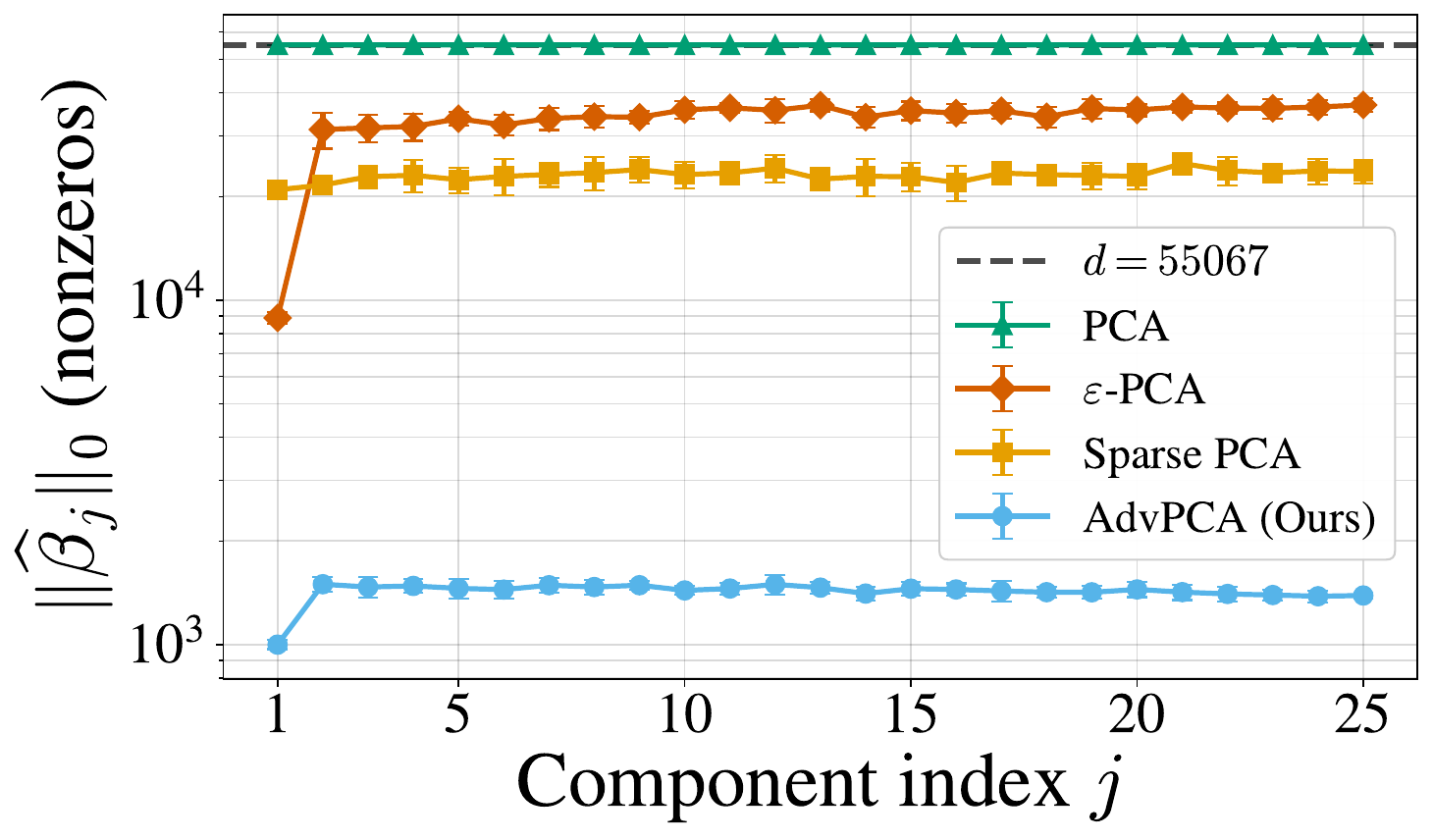}
     \end{subfigure}
    \caption{\textbf{Reconstructing the MAGIC wheat dataset.} We evaluate the methods on a multi-parent inter-cross population of wheat comprising $n = 504$ inbred lines genotyped across $d = 55\,067$ SNP markers. (\textit{Left}) We use an 80/20 train-test split, fitting $k=25$ components on the training set and evaluating the reconstruction error on the held-out test set. We report the mean $\pm$ 1 standard deviation across 10 independent splits. (\textit{Right})~Corresponding sparsity of the solution.}
     \label{fig:magic}
\end{figure}

\section{Conclusion}\label{sec:conclusion}
We propose a robust optimization approach to sparse PCA that naturally induces sparsity and is well-suited for high-dimensional problems. Hyperparameter tuning in unsupervised learning remains an important challenge in modern data science. The adversarial formulation provides a valuable perspective with favorable theoretical properties to guide the parameterization, showing promising results out of the box. Our work is best understood in the context of its limitations. AdvPCA solves a min-max problem that is not concave-convex, restricting global theoretical guarantees. We underline that the nonconvexity of the outer problem also applies to Sparse PCA \cite{zou_sparse_2006}. Furthermore, establishing the exact decay properties in Proposition~\ref{prop:6} warrants further study to better understand how the radius $\delta$ dictates the solution. A promising avenue for future work is extending the framework to kernel PCA using techniques similar to~\cite{ribeiro_kernel_2025}, as well as exploring other norms, e.g., structured PCA via the nonconvex $\ell_{1/2}$-norm~\cite{pmlr-v9-jenatton10a}. Finally, analyzing the adjusted total variance is also another important avenue for future work.

\section{Acknowledgments and Disclosure of Funding}
\label{sec:aknowledgment}
We would like to thank Paul Häusner and Ayca Özcelikkale for their valuable feedback on early versions of this manuscript.
AHR is partially supported  by the Wallenberg AI, Autonomous Systems and Software Program (WASP) funded by the Knut and Alice Wallenberg Foundation and by eSSENCE and SciLifeLab. 
AMHT and DZ are both supported by the Swedish Research Council under grants 2023-05234 and 2024-03903 respectively, as well as the Knut and Alice Wallenberg Foundation. FB is supported by the France 2030 program “PR[AI]RIE-PSAI” (ANR-23-IACL-0008).

\bibliographystyle{plainnat}
\bibliography{references}

\clearpage
\appendix
\pagenumbering{roman}
\thispagestyle{empty}
\onecolumn 
\etocdepthtag.toc{app} 

{\hypersetup{allcolors=black} 

\begin{center}
    {\LARGE A Robust Optimization Approach to Sparse Principal Component Analysis} \\[0.4cm]
    {\LARGE Appendix} \\[0.5cm]
    {\large \makebox[0pt]{David Vävinggren, Francis Bach, André M. H. Teixeira, Dave Zachariah, Antônio H. Ribeiro}} \\[0.5cm]
    
    \begin{minipage}{0.9\textwidth}
        \itshape
    In the Appendix, we provide additional material supporting the main text. In Appendix~\ref{app:sec:supp}, we supply information left out in the paper and also expand on some areas; in Appendix~\ref{app:sec:proofs}, we prove all results in the paper; and in Appendix~\ref{app:sec:exp}, we provide additional experimental results.
    \end{minipage}
\end{center}

\vspace{1cm}
\hrule
\vspace{0.3cm}

\etocsettagdepth{main}{-1}
\etocsettagdepth{app}{2}

\etocsettocstyle{\section*{Contents}}{} 

\begin{center}
\begin{minipage}{0.9\textwidth}
    \etocarticlestyle 
    \tableofcontents 
\end{minipage}
\end{center}

\vspace{0.5cm}
\hrule
\vspace{1cm}

\newpage
}

\setcounter{equation}{0}
\renewcommand{\theequation}{A.\arabic{equation}}
\setcounter{figure}{0}
\renewcommand{\thefigure}{A.\arabic{figure}}

\section{Supplementary Material}\label{app:sec:supp}

In this section, we provide supplementary material complementing the main text. We expand on some areas and provide additional information that could not fit in the paper.

\subsection{PCA and Eigenvector Inconsistency}\label{app:eig_incons}
PCA is a linear dimensionality reduction technique that exploits an important property of high-dimensional data: dimensions are often correlated meaning it is typically concentrated near a lower-dimensional manifold. This manifold is not necessarily linear, but PCA finds the best rank-$k$ linear approximation to it. Once the subspace is found the data can be projected onto it, which reduces the dimension while preserving as much information as possible. 
\begin{figure}[b!]
    \centering
    \includegraphics[width=0.4\linewidth]{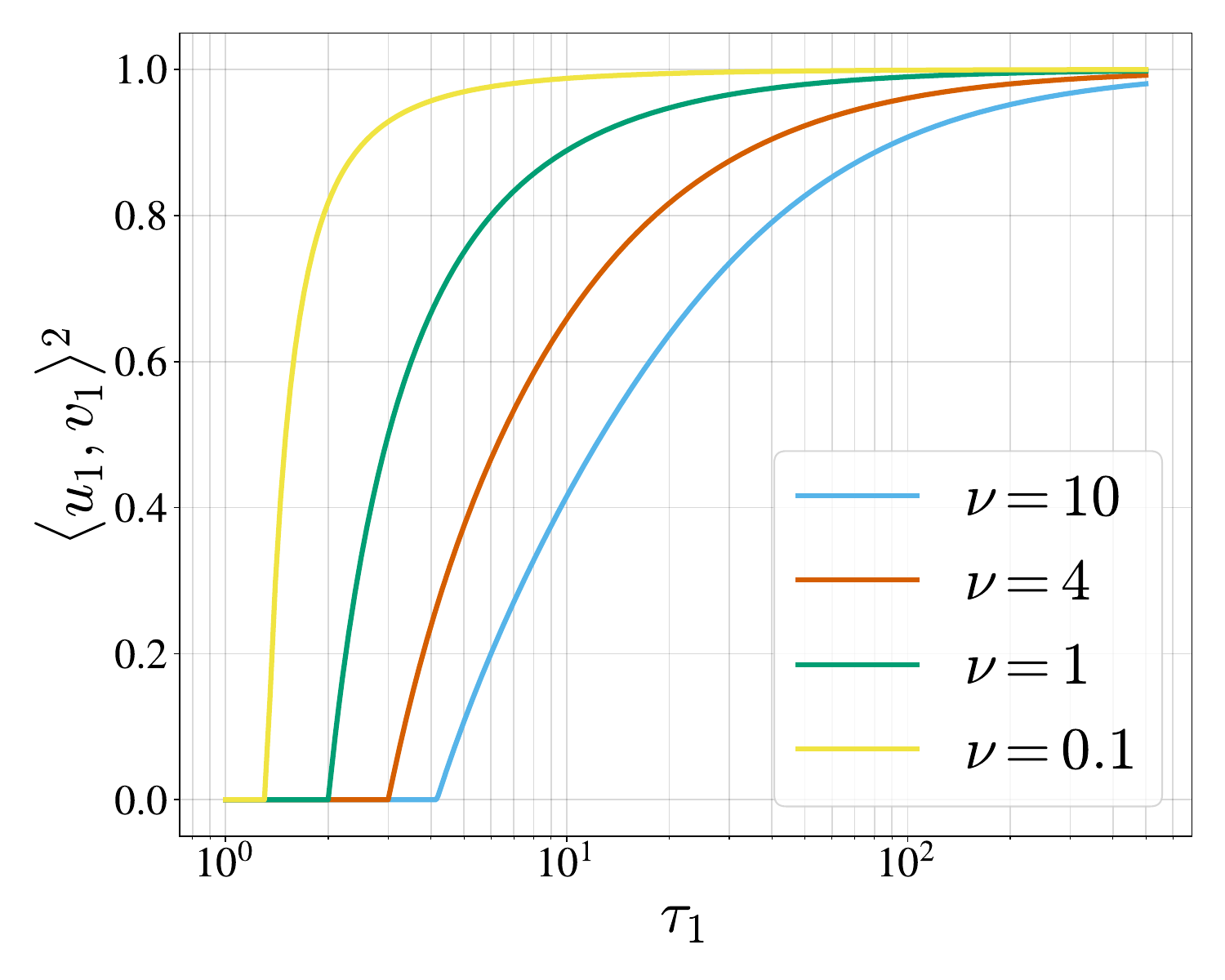}
    \caption{Plot illustrating the eigenvector inconsistency showcased in~\eqref{eq:inconcistency}. Here, $\langle u_1,v_1 \rangle = 0$ means that the estimate $v_1$ is orthogonal to $u_1$. Hence, the estimate is as good as a random guess and provides no information about $u_1$ for low signal strength $\tau_1$.}
    \label{fig:inconsistency}
\end{figure}

PCA hinges on the assumption that the dimension of the data $d$ is sufficiently small in relation to the number of datapoints available, $n$. The high-dimensional setting poses issues that have been studied extensively over the years. \citet{johnstone_pca_2018} give an orientation into these phenomena, and labels them \textit{eigenvalue spreading}, \textit{eigenvalue bias} and \textit{eigenvector inconsistency}. Perhaps least recognized of the three is the eigenvector inconsistency, which refers to the fact that leading eigenvectors in high-dimensional PCA can be inconsistent estimates when $d$ grows proportionally with $n$.

We assume the data  $\mathcal{D} = \{x_i\}_{i=1}^n$ are drawn i.i.d. from an underlying distribution $p_{x}$ on $\mathbb{R}^d$, with zero mean and covariance $\Sigma$. The sample covariance matrix is then given by $\widehat{\Sigma} = \frac{1}{n} \sum_{i=1}^n x_i x_i^\top$. The traditional formulation of large sample theory assumes $d$ to be fixed and $n$ to be large; in this setting it is well-known that $\widehat{\Sigma}$ estimates $\Sigma$ consistently. For example, if the data is assumed to be i.i.d., then $\widehat{\Sigma} \to \Sigma$ as $n \to \infty$ \citep{wainwright_high-dimensional_2019}. We write the population covariance as 
\begin{equation*}
    \Sigma = \sum_{\ell=1}^d \tau_{\ell} u_{\ell} u_{\ell}^\top = UTU^\top,
\end{equation*}
where $U$ is a $d \times d$ orthogonal matrix whose columns $\{u_{\ell}\}_{\ell=1}^d$ are the eigenvectors of $\Sigma$ and $T$ is a diagonal matrix containing the corresponding eigenvalues. By convention and without loss of generality, the eigenvalues are arranged in decreasing order such that \mbox{$\tau_1 \geq \tau_2 \geq \ldots \geq \tau_d$}. In the same fashion, we write the sample covariance matrix as
\begin{equation*}
    \widehat{\Sigma} = \sum_{\ell=1}^d \xi_{\ell} v_{\ell} v_{\ell}^\top = V \Xi V^\top,
\end{equation*}
where we use the same convention as for the population. The asymptotic result for the covariance matrix presented above now extends to its decomposition; if $d$ is considered fixed, one can show that $\xi_{\ell} \to \tau_{\ell}$ and $v_{\ell} \to u_{\ell}$ $\forall\ell=1,\dots,d$ as $n \to \infty$. However, if we no longer consider $d$ as fixed and instead let $d$ grow proportionally with $n$, the eigenvector inconsistency presents itself. 

For the leading eigenvector, \citet{johnstone_pca_2018} show that
\begin{equation}\label{eq:inconcistency}
\langle u_1, v_1\rangle ^2 \to
\begin{dcases}
\dfrac{1-\nu/(\tau_1 - 1)^2}{1+\nu/(\tau_1 - 1)}, &\tau_1 > 1 + \sqrt{\nu},\\
0,                                                   &\tau_1 \in [1,\, 1 + \sqrt{\nu}],
\end{dcases}
\end{equation}
when $d/n \to \nu > 0$. This result is plotted in Figure \ref{fig:inconsistency}. Broadly speaking, this result shows that the leading eigenvalue (corresponding to the signal strength along the leading eigenvector) needs to be large enough to match $d$. If $\tau_1$ remains fixed and $d$ and is allowed to grow, the inconsistency properties can only get worse.

\subsection{The Spiked Covariance Model}\label{app:spiked_cov}
In the Sparse PCA literature, it is common to study models under simpler classes of covariance matrices. Two common ones are the spiked Wigner model and the spiked Wishart model, e.g., considered in \cite{deshpande_information-theoretically_2014}. These allow for analysis and tractable derivations. 

In this work, we use the spiked covariance model from \cite[ch. 8.2.2]{wainwright_high-dimensional_2019}. Let $z_i \sim \mathcal{N}(0,1)$ and \mbox{$w_i \sim \mathcal{N}(0,I_d)$} be i.i.d. noise and $\beta_*\in \mathbb{R}^d$ be the true (unknown) spike direction. Moreover, let $\tau$ define the spike strength, which is a quantity similar to the Signal-to-Noise Ratio (SNR). Then, let \mbox{$x_i = \sqrt{\tau} z_i \beta_* + w_i, \; i=1,\dots,n$}, which means $x_i$ has zero mean and true covariance matrix $\Sigma_* = \tau \beta_* {\beta_*}^\top + I_d$. By construction, this setup means that $\beta_*$ is the unique maximal eigenvector of $\Sigma_*$ with eigenvalue $1+\tau$. All other eigenvalues are located at 1, meaning $\tau$ is the eigengap. 

To accommodate arbitrary eigenvalue spectra and multiple spike directions in our empirical evaluations, we generalize the rank-1 model to a full spectral formulation. Let $D = \mathrm{diag}(\lambda_1, \dots, \lambda_d)$ be a diagonal matrix of eigenvalues and $V \in \mathbb{R}^{d \times d}$ be an orthogonal matrix whose columns encode the corresponding eigenvectors. For a generalized $k$-spiked model, we configure the spectrum by setting the top $k$ eigenvalues to $\lambda_j = 1 + \tau_j$ for $j=1,\dots,k$, where $\tau_j$ are the respective eigengaps, and the remaining $d-k$ ambient eigenvalues to $1$. The orthogonal matrix $V$ can then be instantiated to inject arbitrary, potentially sparse, spike directions into the ambient space.

To draw a finite-sample data matrix $X \in \mathbb{R}^{n \times d}$ with the exact population covariance $\Sigma_* = V D V^\top$, we first sample a latent noise matrix $Z \in \mathbb{R}^{n \times d}$ whose entries are drawn i.i.d.\ from a standard normal distribution, $\mathcal{N}(0,1)$. The observed data matrix is then constructed via the transformation
\begin{equation}
    X = Z D^{1/2} V^\top.
\end{equation}
By construction, the $i$-th sample (corresponding to the $i$-th row of $X$) is formed by the linear combination $x_i = V D^{1/2} z_i$, where $z_i \in \mathbb{R}^d$ is the $i$-th row of $Z$. Because $z_i$ has an identity covariance matrix, this affine transformation guarantees that each sample has zero mean and exactly recovers the prescribed true covariance:
\begin{align*}
    \mathbb{E}[x_i x_i^\top] &= V D^{1/2} \mathbb{E}[z_i z_i^\top] D^{1/2} V^\top \\
    &= V D I_d V^\top \\
    &= V D V^\top = \Sigma_*.
\end{align*}
This generative procedure explicitly connects our theoretical covariance assumptions to the practical simulation of our datasets, allowing independent control over the geometric structure of the eigenvectors in $V$ and the signal-to-noise ratios in $D$.

\subsection{Efficient Solver of Linear Adversarial Regression Problems}\label{app:eta-trick}
In Section \ref{sec:solver}, we use Algorithm~\ref{alg:eta-trick} for the $k$ independent adversarial linear regression problems. Algorithm~\ref{alg:eta-trick} leverages what is sometimes referred to as the $\eta$-trick, see e.g~\cite{ItrickReloadedMultiple2019}.

\noindent
\begin{minipage}[t]{0.33\textwidth}
    \vspace{0pt} 
    It is an alternating minimization scheme designed to improve computational performance over standard general-purpose convex solvers like \texttt{cvxpy}~\cite{agrawal2018rewriting, diamond2016cvxpy}. Solving adversarial linear regression problems this way was proposed by~\cite{ribeiro_efficient_2025}. By reformulating the objective, the optimization is reduced to iteratively solving a weighted ridge regression problem (which admits a fast, closed-form solution) and updating the variational parameters $\eta$. We provide a proof for the algorithm in Appendix~\ref{app:proof:eta}. We want to highlight that a critical implementation detail (left out of the algorithm for clarity) is numerical stability. When $\delta \to 0$ or residuals vanish, the $\eta$ updates risk division by zero. In the actual implementation, this is routinely handled by adding a small smoothing constant $\varepsilon > 0$ to the denominators. We refer the reader to~\cite{ribeiro_efficient_2025} for a detailed discussion on this stabilization step.
\end{minipage}
\hfill
\begin{minipage}[t]{0.62\textwidth}
    \vspace{0pt}
    \begin{algorithm}[H]
        \caption{$\eta$-trick}
        \label{alg:eta-trick}
    \KwInput{$\{x_i\}_{i=1}^n;$ $\alpha;$ $\delta\geq0$}
    \KwInitialize{$w_i\gets 1$ for $i=1,\dots,n$ \\ $\qquad \; \gamma_j\gets 1$ for $j=1,\dots,d$}
    
    \Repeat{\texttt{StopCondition}}{
      \Comment{Solve weighted ridge problem}
      $\beta \gets \argmin_{\beta} \sum_{i=1}^n w_i(\alpha^\top x_i - \beta^\top x_i)^2 + \sum_{j=1}^d \gamma_j \beta_j^2$\;
      \Comment{Set $\eta$}
      \For{$i \gets 1$ \KwTo $n$}{
      $\eta_0^{(i)} \gets \dfrac{\lvert \alpha^\top x_i - \beta^\top x_i\rvert}{\lvert \alpha^\top x_i - \beta^\top x_i\rvert + \delta \lVert \beta \rVert_1}$\;
      \For{$j \gets 1$ \KwTo $d$}{
      $\eta_j^{(i)} \gets
       \dfrac{\delta\,\lvert \beta_j\rvert}{\lvert \alpha^\top x_i - \beta^\top x_i\rvert + \delta \lVert \beta \rVert_1}$\;
      }
      }    
      \Comment{Update weights}
      \For{$i \gets 1$ \KwTo $n$}{
      $w_i \gets \dfrac{1}{\eta_0^{(i)}}$\;
      }
      \For{$j \gets 1$ \KwTo $d$}{
      $\gamma_j \gets \delta^2 \sum_{i=1}^n \dfrac{1}{\eta_j^{(i)}}$\;
      }
      
    }
    \KwRet{$\beta$}
    \label{alg:eta-trick}
    \end{algorithm}
\end{minipage}

\subsection{Smoothed Update of the $A$ Matrix}\label{app:smoothing}
In Algorithm~\ref{alg:advpca}, we employ a smoothed update of the $A$ matrix to stabilize the iterative optimization procedure. For small values of $\delta$, we find empirically that the optimization is stable even without the smoothing, however, smoothing greatly improves the ability to converge for larger $\delta$s. In Figure~\ref{app:fig:convergence}, we show an example of the difference when running Algorithm~\ref{alg:advpca} with $\epsilon=0.5$ (smoothing) versus $\epsilon=0$ (no smoothing) using $\delta_j = 0.4 \delta_{\mathrm{max},j}$ for $j=1,\dots,10$ in the multi-component recovery experiment presented in Figure~\ref{fig:multi-comp-recovery}. The training loss refers to the objective in \eqref{eq:final_advpca}. We see that we achieve a much more stable decrease in loss as well as a lower final value.
\begin{figure}[H]
     \centering
     \begin{subfigure}[b]{0.49\textwidth}
         \centering
         \includegraphics[width=0.9\textwidth]{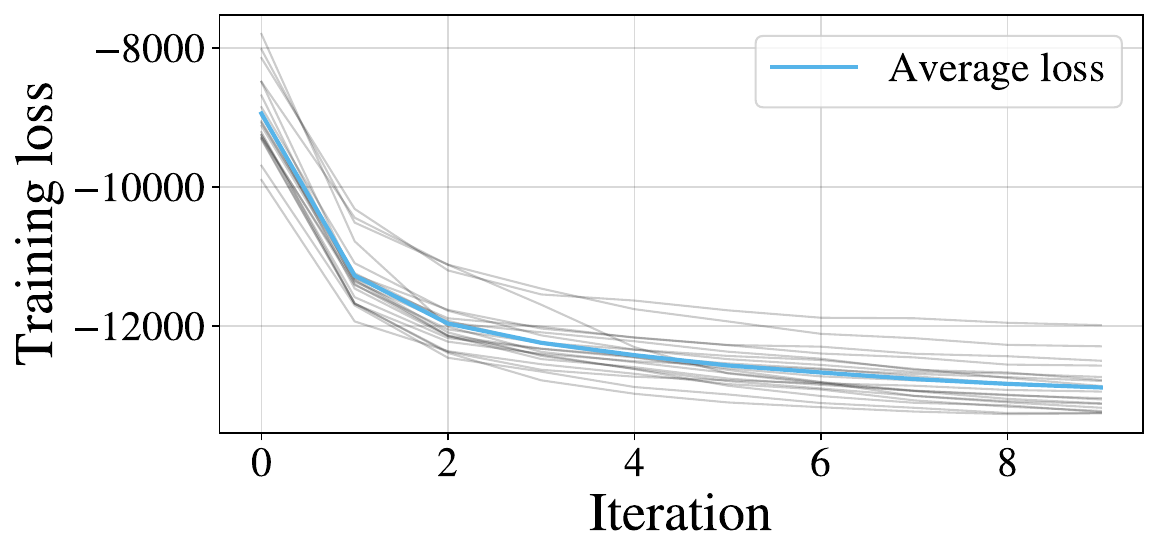}
     \end{subfigure}
     \hfill 
     \begin{subfigure}[b]{0.49\textwidth}
         \centering
         \includegraphics[width=0.9\textwidth]{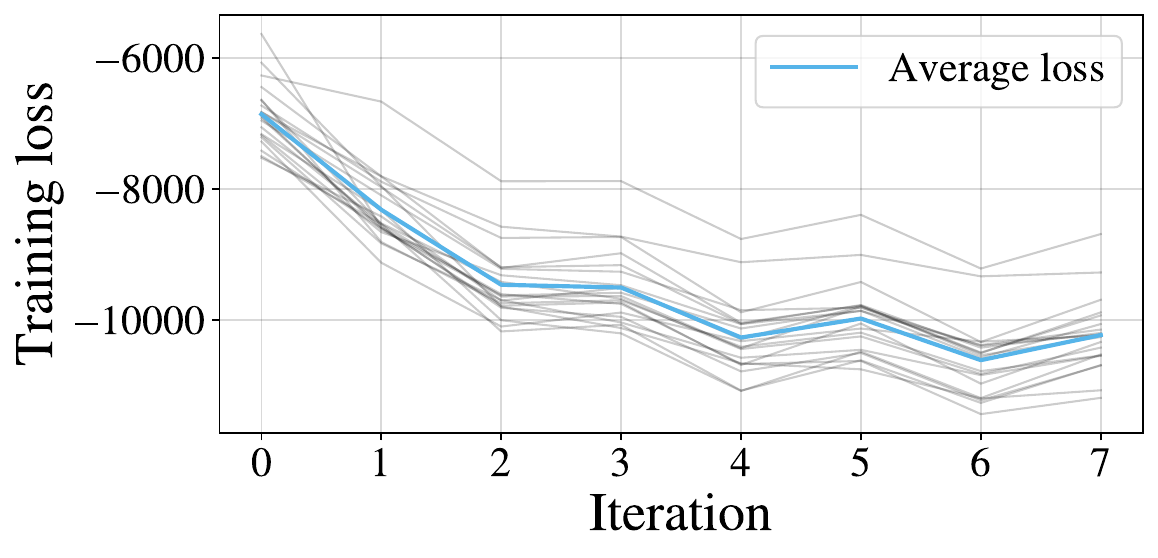}
     \end{subfigure}
     \caption{We use $\delta_j = 0.4 \delta_{\mathrm{max},j}$ for $j=1,\dots,10$ in the multi-component recovery experiment presented in Figure~\ref{fig:multi-comp-recovery} to illustrate the difference between optimizing with and without smoothing. The training loss refers to the objective in \eqref{eq:final_advpca}, and the grey curves are 20 independent runs from different random draws of data from the spiked model. (\textit{Left}) With smoothing, using $\epsilon=0.5$. (\textit{Right}) Without smoothing, using $\epsilon=0$.}
     \label{app:fig:convergence}
\end{figure}
To implement the smoothing, we cannot simply update \mbox{$A \gets \epsilon A + (1-\epsilon) A_{\mathrm{upd}}$}. Because the set of orthogonal matrices forms a curved manifold, taking a linear step pushes the matrix off the constraint surface into the ambient space, violating the strict requirement that $A^\top A = I_k$. To address this issue, we allow the step into the ambient space but also follow it up by projecting the smoothed matrix back onto the nearest valid point on the manifold of orthogonal matrices~\cite{edelman_geometry_1998}. Finding the closest orthogonal matrix to our smoothed target $A_{\mathrm{smooth}}$ is formulated as the orthogonal Procrustes problem
\begin{equation*}
    \min_{A^\top A = I} \lVert A - A_{\mathrm{smooth}} \rVert_F^2.
\end{equation*}
We solve this Procrustes problem in the exact same way as when updating $A$, i.e., using the singular value decomposition; see Appendix~\ref{app:proof:prop2}.

\subsection{The Subgradient}\label{app:subgradient}
We compute the subgradient of the Adversarial PCA risk~\eqref{eq:advpca_risk}, as this will be used in the proofs later on. Consider the minimization of the risk, 
\[
\argmin_\beta \widehat{\mathcal{R}}_\delta(\beta) = \argmin_\beta \sum_{i=1}^n \left( \lvert \alpha^\top x_i - \beta^\top x_i \rvert + \delta \norm[1]{\beta} \right) ^2,
\]
where we minimize over $\beta \in \mathbb{R}^d$. Let $g_i: \mathbb{R}^d \to \mathbb{R}$ be defined as $g_i(\beta) = \lvert \alpha^\top x_i - \beta^\top x_i \rvert + \delta \norm[1]{\beta}$, so that the risk $\empadvpcarisk{\beta}$ is given by
\[
\empadvpcarisk{\beta} = \sum_{i=1}^n g_i(\beta) ^2.
\]
$\widehat{\mathcal{R}}_\delta$ is convex in $\beta$ since $g_i$ is convex and nonnegative, but it is not differentiable everywhere because of the absolute value. For convex functions however, the notion of subgradients naturally extends derivatives into the space of non-differentiable functions. 

Given a convex function $f: \mathbb{R}^d \to \mathbb{R}$, a vector $v \in \mathbb{R}^d$ is said to be a subgradient of $f$ at $\beta$ if
\[
f(\beta') \geq f(\beta) + \langle v, \beta' - \beta \rangle \quad \text{for all $\beta' \in \mathbb{R}^d$}.
\]
The set of all subgradients of $f$ at $\beta$ is called the subdifferential, and is a convex set denoted $\partial f(\beta)$. If $f$ is differentiable at $\beta$, then this consists of a single element.

In case of the absolute value function $f : \R \to \R$ given by $f(\beta) = \lvert \beta \rvert$, the subdifferential is given by
\[
\partial \lvert \beta \rvert = 
\begin{dcases}
    \{+1\} & \beta > 0, \\
    \{-1\} & \beta < 0, \\
    [-1, +1] & \beta = 0.
\end{dcases}
\]
We use this to derive an expression for the subdifferential of $\empadvpcarisk{\beta}$, which is the set of all subgradients. To start,
\[
\partial \empadvpcarisk{\beta} = \partial \sum_{i=1}^n g_i(\beta)^2 = 2\sum_{i=1}^n g_i(\beta) \partial g_i(\beta),
\]
where $\partial g_i(\beta) = \partial \lvert \alpha^\top x_i - \beta^\top x_i \rvert + \delta \partial \norm[1]{\beta}$. Here, we apply the result for $\lvert \cdot \rvert$, which gives
\[
\partial \lvert \alpha^\top x_i - \beta^\top x_i \rvert = 
\underbrace{
\begin{dcases}
    \{+1\} & \alpha^\top x_i - \beta^\top x_i > 0 \\
    \{-1\} & \alpha^\top x_i - \beta^\top x_i < 0 \\
    [-1, +1] & \alpha^\top x_i - \beta^\top x_i = 0
\end{dcases}
}_{s_i}
\; \partial \left( \alpha^\top x_i - \beta^\top x_i \right) = s_i (-x_i),
\]
and 
\[
\partial \norm[1]{\beta} = z, \quad \text{where} \quad z_j = 
\begin{dcases}
    \{+1\} & \beta_j > 0 \\
    \{-1\} & \beta_j < 0 \\
    [-1, +1] & \beta_j = 0
\end{dcases}
\quad \text{for} \quad  j=1,\dots,d.
\]
We therefore have
\[
\partial \empadvpcarisk{\beta} = 2 \sum_{i=1}^n \left( \lvert \alpha^\top x_i - \beta^\top x_i \rvert + \delta \norm[1]{\beta} \right) (-s_i x_i + \delta z) \subseteq \mathbb{R}^d.
\]
By noting that $\lvert \alpha^\top x_i - \beta^\top x_i \rvert s_i = \alpha^\top x_i - \beta^\top x_i$, this can be simplified to
\[
\frac{1}{2} \partial \empadvpcarisk{\beta} = \{ X^\top X \left( \beta - \alpha \right) \} + \norm[1]{X(\beta - \alpha)}\delta z - \delta \norm[1]{\beta} X^\top s + n \delta^2 \norm[1]{\beta} z.
\]

\section{Proofs}\label{app:sec:proofs}
In the Proofs section, we give proofs to all results presented in the paper.

\subsection{Proof of Proposition~\ref{prop:1}}\label{app:proof:prop1}
 
\begin{proof}
We denote the adversarial PCA objective as
\begin{equation*}
\mathcal{L}(A, B) = \underset{r \in \Omega_{\delta}}{\text{max}} \; \norm[2]{x - A(B^\top x + r)}^2.
\end{equation*}

Expanding the norm gives
\begin{align*}
\norm[2]{x - A(B^\top x + r)}^2 &= (x - A(B^\top x + r))^\top (x - A(B^\top x + r)) \\
&= x^\top x - x^\top A(B^\top x + r) - (B^\top x + r)^\top A^\top x + (B^\top x + r)^\top \underbracket{A^\top A}_{= I_k} (B^\top x + r) \\
&= x^\top (I_d - 2A B^\top +B B^\top )x + 2r^\top(B-A)^\top x + r^\top r \\
&= x^\top \left[(B-A)(B-A)^\top + I_d - A A^\top \right]x + 2r^\top (B-A)^\top x + r^\top r.
\end{align*}

Letting $e = (B-A)^\top x$, where we note $e \in \mathbb{R}^k$, gives 
\begin{align*}
 \norm[2]{x - A(B^\top x + r)}^2 &= e^\top e + x^\top x - x^\top AA^\top x + 2r^\top e + r^\top r, 
\end{align*}

and hence
\begin{align*}
\mathcal{L}(A, B) &= e^\top e + x^\top x - x^\top A A^\top x + \underset{r \in \Omega_{\delta}}{\text{max}} \; \left\{ r^\top r + 2r^\top e \right\}.
\end{align*}

If we let $\Omega_\delta$ be the hyperrectangle given by $\Omega_{\delta} = \prod_{j=1}^k \left[-\delta_j \norm{\beta_j}_1, \delta_j \norm{\beta_j}_1\right]$, we get
\begin{align*}
\underset{r \in \Omega_{\delta}}{\text{max}} \; \left\{r^\top r + 2r^\top e \right\} &= \underset{|r_j| \leq \delta_j \norm[1]{\beta_j}}{\text{max}} \; \sum_{j=1}^k \left\{ r_j^2 + 2r_je_j \right\} = \sum_{j=1}^k \; \underset{|r_j| \leq \delta_j \norm[1]{\beta_j}}{\text{max}} \left\{ r_j^2 + 2r_je_j \right\}
\end{align*}

since we have $k$ separable constraints and $k$ separable objective functions. The maximization is now solvable directly:
\begin{align*}
\underset{|r_j| \leq \delta_j \norm[1]{\beta_j}}{\text{max}} \left\{ r_j^2 + 2r_je_j \right\} &= 
\begin{dcases}
\evat{r_j^2 + 2r_je_j}{r_j = +\delta_j \norm[1]{\beta_j}} \quad \text{if $e_j \geq 0$} \\[2.5ex]
\evat{r_j^2 + 2r_je_j}{r_j = -\delta_j \norm[1]{\beta_j}} \quad \text{if $e_j < 0$} \\ 
\end{dcases} \\[1em]
&= 
\begin{dcases}
    \delta_j^2\norm[1]{\beta_j}^2 + 2\delta_j\norm[1]{\beta_j}e_j \quad \text{if $e_j \geq 0$} \\
    \delta_j^2\norm[1]{\beta_j}^2 - 2\delta_j\norm[1]{\beta_j}e_j \quad \text{if $e_j < 0$}
\end{dcases} \\[1em]
&= \delta_j^2\norm[1]{\beta_j}^2 + 2\delta_j\norm[1]{\beta_j}|e_j|.
\end{align*}

So
\begin{empheq}[box=\fbox]{align*}
\mathcal{L}(A, B) &= x^\top x - x^\top AA^\top x + e^\top e + \sum_{j=1}^k \delta_j^2\norm[1]{\beta_j}^2 + 2\delta_j\norm[1]{\beta_j}|e_j| \\
&= \norm{x}_2^2 - \norm{A^\top x}_2^2 + \sum_{j=1}^k \left( |e_j| + \delta_j \norm[1]{\beta_j} \right)^2
\end{empheq}

leaving us with the final expression for $\mathcal{L}(A, B)$. Note that the perturbation maximizing the objective corresponds to $r_j = \text{sign}(e_j) \, \delta_j \norm{\beta_j}_1$. Also note that $\norm{x}_2^2 - \norm{A^\top x}_2^2 = \norm{x - A A^\top x }_2^2$ since $A^\top A = I_k$.

\end{proof}

\subsection{Proof of Proposition \ref{prop:2}}\label{app:proof:prop2}
\begin{proof}
We want to solve    
\begin{equation}\label{app:prop3:procrustes}
\min_{A} \norm{X - (XB + \widehat{R}) A^\top}_F^2 \quad \text{subject to} \quad A^\top A = I_k,
\end{equation}
with $\widehat{R} = \text{sign}(X(B-\widebar{A})) D \in \R^{n \times k}$. Expanding the norm gives
\begin{align*}
\norm{X - (XB + \widehat{R}) A^\top}_F^2
&= \Tr{(X - (XB + \widehat{R}) A^\top)^\top (X - (XB + \widehat{R}) A^\top)} \\
&\stackrel{\text{(a)}}{=} \Tr{X^\top X} - 2\Tr{X^\top (XB + \widehat{R})A^\top} \\
&\quad + \Tr{(XB + \widehat{R})^\top (XB + \widehat{R}) \underbracket{A^\top A}_{=I_k}},
\end{align*}
where (a) uses the cyclic property of the trace. Because $A^\top A = I_k$, the cross term is the only one dependent on $A$, meaning~\eqref{app:prop3:procrustes} is equivalent to
\begin{equation}
\max_{A} \; \Tr{X^\top (XB + \widehat{R})A^\top} \quad \text{subject to} \quad A^\top A = I_k.
\end{equation}
Now, let 
\begin{equation*}
    X^\top (XB + \widehat{R}) = U \Sigma V^\top
\end{equation*}
be the singular value decomposition. Since $X^\top (XB + \widehat{R}) \in \R^{d \times k}$ for $d\geq k$, we have $U \in \R^{d \times d}$, $\Sigma \in \R^{d \times k}$ and $V \in \R^{k \times k}$ where $U^\top U = UU^\top = I_d$ and $V^\top V = VV^\top = I_k$. Substituting the SVD gives
\begin{align*}
    \Tr{X^\top (XB + \widehat{R})A^\top} = \Tr{U \Sigma V^\top A^\top} = \Tr{V^\top A^\top U \Sigma}.
\end{align*}
Let $Q = V^\top A^\top U \in \R^{k \times d}$. Because $U$ and $V$ are orthogonal matrices and $A$ has orthonormal columns, $Q$ has orthonormal rows:
\[
QQ^\top = V^\top A^\top U (V^\top A^\top U)^\top = V^\top A^\top \underbracket{U U^\top}_{= I_d} A V = V^\top \underbracket{A^\top A}_{= I_k} V = V^\top V = I_k.
\]
Thus, we can rewrite the trace as
\begin{equation*}
    \Tr{Q \Sigma} = \sum_{i=1}^k q_{ii} \sigma_i.
\end{equation*}
Since $QQ^\top = I_k$, the row vectors of $Q$ have a unit $\ell_2$-norm, which bounds its elements such that $q_{ii} \leq 1$. Furthermore, by definition, the singular values are non-negative ($\sigma_i \geq 0$). Therefore, the sum $\sum q_{ii} \sigma_i$ is maximized when every diagonal element $q_{ii}$ is as large as possible. This maximum is achieved when $q_{ii} = 1$ for all $i=1, \dots, k$, which implies that $Q$ must be a rectangular (fat) identity matrix $I_{k \times d} = [I_k \mid 0]$. Setting $Q = I_{k \times d}$, we solve for the optimal $A$:
\[
\boxed{
    V^\top A^\top U = Q = I_{k \times d} \Leftrightarrow A = U I_{d \times k} V^\top = U_{1:k} V^\top
    } \;,
\]
where $U_{1:k}$ denotes the first $k$ columns of $U$. This yields the final closed-form update for $A$.

\end{proof}

\subsection{Proof of Proposition~\ref{prop:3}}\label{app:proof:prop3}
We start by showing the equivalence between input and latent space perturbations. 
\begin{proof}
We define 
\begin{equation*}
    \mathcal{L}(\alpha, \beta) = \underset{\norm[p]{\Delta x} \leq \delta}{\mathrm{max}} \norm[2]{x - \alpha \beta^\top(x + \Delta x)}^2.
\end{equation*}
Making the variable substitution $\beta^\top \Delta x = r$ and using Hölder's inequality \cite[ch.~9]{steele_cauchy-schwarz_2010}, namely
\[
\abs{r} = |\beta^\top \Delta x| \leq \norm{\Delta x}_p \norm{\beta}_q \leq \delta \norm{\beta}_q,
\]
gives
\[
\boxed{
\mathcal{L}(\alpha, \beta) = \underset{\norm[p]{\Delta x} \leq \delta}{\mathrm{max}} \norm[2]{x - \alpha (\beta^\top x + \beta^\top \Delta x)}^2 = \underset{\abs{r} \leq \delta \norm{\beta}_q}{\mathrm{max}} \norm[2]{x - \alpha (\beta^\top x + r)}^2}\; .
\]
To complete the proof, we must verify that the upper bound of Hölder's inequality is tight for any conjugate pair $p, q \geq 1$. By the definition of the dual norm (see e.g. \cite[app. A.1.6]{Boyd_Vandenberghe_2004})
\[
\norm{\beta}_q = \sup_{\norm{\nu}_p \leq 1} \beta^\top \nu,
\]
there always exists a vector $\nu$ satisfying $\norm{\nu}_p \leq 1$ that maximizes the inner product such that $\beta^\top \nu = \norm{\beta}_q$. By scaling this vector to utilize the full adversarial budget, we construct the worst-case perturbation $\Delta x_\star = \delta \nu$. This guarantees strict feasibility, as $\norm[p]{\Delta x_\star} \leq \delta$, and achieves exactly $\beta^\top \Delta x_\star = \delta \norm[q]{\beta}$. By symmetry, the lower bound $-\delta \norm[q]{\beta}$ is also attainable. Consequently, the inner product $r = \beta^\top \Delta x$ can span the entire interval $[-\delta \norm[q]{\beta}, \delta \norm[q]{\beta}]$, confirming the equivalence of the optimization domains.

\end{proof}

We move on and solve the inner maximum.

\begin{proof}
We start from the input space formulation
\begin{equation*}
    \mathcal{L}(\alpha, \beta) = \underset{\norm[p]{\Delta x} \leq \delta}{\mathrm{max}} \norm[2]{x - \alpha \beta^\top(x + \Delta x)}^2
\end{equation*}
and expand the norm:
\begin{align*}
    \norm[2]{x - \alpha \beta^\top(x + \Delta x)}^2 &= (x - \alpha \beta^\top(x + \Delta x))^\top (x - \alpha \beta^\top(x + \Delta x)) \\
    &= x^\top x - 2x^\top \alpha \beta^\top (x + \Delta x) + (x + \Delta x)^\top \beta \underbracket{\alpha^\top \alpha}_{=1} \beta^\top (x + \Delta x) \\
    &= \norm[2]{x}^2 - (\alpha^\top x)^2 + (\beta^\top x - \alpha^\top x)^2 +  (\beta^\top \Delta x)^2 + 2\beta^\top \Delta x(\beta^\top x - \alpha^\top x).
\end{align*} 
Isolating terms dependent on $\Delta x$ then gives
\begin{equation*}
    \mathcal{L}(\alpha, \beta) = \norm[2]{x}^2 - (\alpha^\top x)^2 + (\beta^\top x - \alpha^\top x)^2 + \underset{\norm[p]{\Delta x} \leq \delta}{\mathrm{max}}\Bigl\{ (\beta^\top \Delta x)^2 + 2\beta^\top \Delta x(\beta^\top x - \alpha^\top x)\Bigl\}.
\end{equation*}
We now let $r = \beta^\top \Delta x$ and $e = (\beta - \alpha)^\top x$:
\begin{equation*}
    \mathcal{L}(\alpha, \beta) = \norm[2]{x}^2 - (\alpha^\top x)^2 + e^2 + \underset{\norm[p]{\Delta x} \leq \delta}{\mathrm{max}}\{r^2+2re\}.
\end{equation*}
Again, by using Hölder's inequality, we finally have
\begin{equation*}
    \mathcal{L}(\alpha, \beta) = \norm[2]{x}^2 - (\alpha^\top x)^2 + e^2 + \underset{|r| \leq \delta\norm[q]{\beta}}{\mathrm{max}}\{r^2+2re\},
\end{equation*}
which is maximized by $r=\delta\norm[q]{\beta}$ if $e \geq 0$ and $r=-\delta\norm[q]{\beta}$ if $e < 0$. Hence
{\setlength{\jot}{10pt}
\begin{empheq}[box=\fbox]{align*}
    \mathcal{L}(\alpha, \beta) &= \norm[2]{x}^2 - (\alpha^\top x)^2 + e^2 + 
    \begin{dcases}
        \delta^2\norm[q]{\beta}^2 + 2\delta\norm[q]{\beta}e & \text{if } e \geq 0 \\
        \delta^2\norm[q]{\beta}^2 - 2\delta\norm[q]{\beta}e & \text{if } e < 0
    \end{dcases} \\
    &= \norm[2]{x}^2 - (\alpha^\top x)^2 + |e|^2 + \delta^2\norm[q]{\beta}^2 + 2\delta\norm[q]{\beta}|e| \\
    &= \norm[2]{x}^2 - (\alpha^\top x)^2 + (|e| + \delta \norm[q]{\beta})^2
\end{empheq}}
leaving us with our final expression for $\mathcal{L}(\alpha, \beta)$. Note that the perturbation maximizing the objective corresponds to $r = \text{sign}(e) \, \delta \norm{\beta}_q$. Also note that $\norm{x}_2^2 - (\alpha^\top x)^2 = \norm{x - \alpha \alpha^\top x }_2^2$ since $\alpha^\top \alpha = 1$.

\end{proof}

We end with showing the case when $\alpha = \beta$.

\begin{proof}
\small
\begin{align*}
    \underset{\norm[p]{\Delta x} \leq \delta}{\mathrm{max}} \norm[2]{x - \beta \beta^\top(x + \Delta x)}^2 &= \underset{\norm[p]{\Delta x} \leq \delta}{\mathrm{max}} {\footnotesize \left\{ x^\top x - 2x^\top\beta \beta^\top(x + \Delta x) + (x + \Delta x)^\top\beta \underbracket{\beta^\top \beta}_{=1} \beta^\top  (x + \Delta x)\right\}} \\
    &= \underset{\norm[p]{\Delta x} \leq \delta}{\mathrm{max}}\left\{ x^\top(I - \beta\beta^\top)x + \Delta x^\top\beta\beta^\top\Delta x \right\} \\
    &= 
x^\top(I - \beta \beta^\top)x + \underset{\norm[p]{\Delta x} \leq \delta}{\mathrm{max}} (\beta^\top \Delta x)^2 \stackrel{\text{(a)}}{=} \snorm{x - \beta \beta^\top x}_2^2 + \delta^2 \norm[q]{\beta}^2
\end{align*}

\normalsize where we use Hölder's inequality in (a), i.e., that \mbox{$|\beta^\top\Delta x| \leq \norm[p]{\Delta x} \norm[q]{\beta} \leq \delta \norm[q]{\beta}$}.

\end{proof}

\subsection{Proof of Proposition \ref{prop:4}}
This proof is adapted from \cite[thm.~1]{ribeiro_regularization_2023}. 

The following lemma is used to prove Proposition \ref{prop:4}. We start by proving this.
\begin{lemma}\label{app:lemma:dual_prob}
    Let $X \in \R^{n \times d}$ be full row rank. The minimum $\ell_1$-norm interpolator then admits the following dual problem:
    \begin{equation*}
    \min_{X\alpha = X\beta} \norm{\beta}_1 = \max_{\norm{X^\top \nu}_\infty \leq 1} \nu^\top X \alpha.
    \end{equation*}
    Furthermore, $\widehat{\beta}$ and $\hat{\nu}$ are primal and dual optimal, respectively, if and only if \mbox{$X^\top \hat{\nu} \in \partial\snorm{\widehat{\beta}}_1$}.
\end{lemma}

\begin{proof}
We study the primal problem
\[
\min_{X\alpha = X\beta} \norm{\beta}_1.
\]
We first note that we have a feasible problem ($\beta=\alpha$ trivially feasible) with a convex objective function subject to an affine constraint. Therefore, Slater's condition is satisfied and strong duality holds (see e.g. \cite[ch.~5.2.3]{Boyd_Vandenberghe_2004}). 

The following is standard duality theory, with notation from \cite[ch.~5]{Boyd_Vandenberghe_2004}. The Lagrangian of the primal problem is given by
\[
L(\beta, \nu) = \norm{\beta}_1 + \nu^\top (X\alpha - X \beta)
\]
where $\nu \in \R^n$ denotes the dual variable. The Lagrange dual function is then given by
\begin{align*}
g(\nu) 
= \inf_{\beta \in \R^d} L(\beta, \nu) 
= \inf_{\beta \in \R^d} \left\{ \norm{\beta}_1 + \nu^\top (X\alpha - X \beta) \right\} 
&= \nu^\top X\alpha + \inf_{\beta \in \R^d} \left\{ \norm{\beta}_1 - (X^\top \nu)^\top \beta \right\} \\
&\eqstack{(a)} \nu^\top X\alpha -\sup_{\beta \in \R^d} \left\{ (X^\top \nu)^\top \beta - \norm{\beta}_1 \right\},
\end{align*}
where we in (a) use that $\inf_x \{ f(x) \} = -\sup_x \{ -f(x) \}$ holds for an arbitrary function $f$. By the definition of the conjugate of a function, see e.g. \cite[ch.~3.3]{Boyd_Vandenberghe_2004}, we identify that the sup is a conjugate:
\[
\sup_{\beta \in \R^d} \left\{ (X^\top \nu)^\top \beta - \norm{\beta}_1 \right\} = \norm{X^\top \nu}_1^*.
\]
Further, the conjugate of this $\ell_1$-norm is given by
\[
\norm{X^\top \nu}_1^* = 
\begin{dcases}
0 & \norm{X^\top \nu}_\infty \leq 1, \\
\infty & \text{otherwise}.
\end{dcases}
\]

So for the Lagrange dual function $g(\nu)$ to be bounded, we have the implicit constraint \mbox{$\norm{X^\top \nu}_\infty \leq 1$}. Using this, we can formulate the dual problem:
\[
\boxed{
\sup g(\nu) = \sup_{\norm{X^\top \nu}_\infty \leq 1} \nu^\top X\alpha
} \;.
\]
Finally, to establish the optimality condition, we rely on the KKT conditions (see e.g. \cite[ch.~5.5.3]{Boyd_Vandenberghe_2004}). Since strong duality holds, the primal optimal $\widehat{\beta}$ must minimize the Lagrangian evaluated at the dual optimal $\hat{\nu}$. This means the zero vector must be in the subdifferential of $L(\beta, \hat{\nu})$ with respect to $\beta$, evaluated at $\widehat{\beta}$:
\[
\boxed{
0 \in \partial \snorm{\widehat{\beta}}_1 - X^\top \hat{\nu} \Rightarrow X^\top \hat{\nu} \in \partial \snorm{\widehat{\beta}}_1
} \;.
\]

\end{proof}

We now move on to prove Proposition \ref{prop:4}.

\begin{proof}
Let $\widehat{\beta}$ be the minimum $\ell_1$-norm interpolator $\argmin_{X\alpha = X\beta}\norm{\beta}_1$ and let $X \in \R^{n \times d}$ be full row rank such that $n < d$. We want to prove that $\widehat{\beta}$ minimizes the Adversarial PCA risk~\eqref{eq:advpca_risk}, given by
\[
\widehat{\mathcal{R}}_\delta(\beta) = \sum_{i=1}^n \left( \lvert \alpha^\top x_i - \beta^\top x_i \rvert + \delta \norm[1]{\beta} \right) ^2,
\]
if and only if 
\[
\delta \leq \frac{1}{n}\norm{\hat{\nu}}_\infty^{-1},
\]
where $\hat{\nu}$ is the solution to the convex problem $\max_{\norm{X^\top \nu}_\infty \leq 1} \nu^\top X\alpha$.

To prove the theorem, we make use of the subdifferential of the risk (see Appendix~\ref{app:subgradient}), satisfying
\[
\frac{1}{2} \partial \empadvpcarisk{\beta} = \{ X^\top X ( \beta - \alpha ) \} + \snorm{X(\beta - \alpha)}_1 \delta z - \delta \snorm{\beta}_1 X^\top s + n \delta^2 \snorm{\beta}_1 z.
\]
Evaluating the subdifferential at the interpolator $\widehat{\beta}$ means we have $X\widehat{\beta} = X\alpha$, which simplifies the subdifferential according to
\[
\frac{1}{2} \partial \widehat{\mathcal{R}}_\delta(\widehat{\beta}) = - \delta \snorm{\widehat{\beta}}_1 X^\top \partial \norm{0}_1 + n \delta^2 \snorm{\widehat{\beta}}_1 \partial \snorm{\widehat{\beta}}_1.
\]
Because $\widehat{\mathcal{R}}_\delta$ is convex, $\widehat{\beta}$ is a global minimizer if and only if $0 \in \partial \widehat{\mathcal{R}}_\delta(\widehat{\beta})$. For the zero vector to be in this set, there must exist some specific subgradient $\hat{s} \in \partial \norm[1]{0}$ and some $\hat{z} \in \partial \snorm{\widehat{\beta}}_1$ such that
\[
0 = -\delta \snorm{\widehat{\beta}}_1 X^\top \hat{s} + n \delta^2 \snorm{\widehat{\beta}}_1 \hat{z}.
\]
Assuming $\widehat{\beta} \neq 0$, we equivalently have
\[
X^\top \hat{s} = n \delta \hat{z}.
\]
From Lemma \ref{app:lemma:dual_prob}, we know that the specific subgradient $\hat{z} \in \partial \snorm{\widehat{\beta}}_1$ at optimality is determined by the dual solution as $\hat{z} = X^\top \hat{\nu}$. Substituting this yields
\[
X^\top \hat{s} = n \delta X^\top \hat{\nu}.
\]
By assumption, $X \in \R^{n \times d}$ with $n < d$ has full row rank, which means its transpose $X^\top$ has full column rank. Consequently, $X^\top$ has a trivial null space, allowing us to drop it from both sides of the equation and obtain
\[
\hat{s} = n\delta \hat{\nu}.
\]

By definition, any element $\hat{s} \in \partial \norm{0}_1$ must satisfy $\snorm{\hat{s}}_\infty \leq 1$. Therefore, for $\widehat{\beta}$ to be the global minimizer, we must have
\[
\norm{n \delta \hat{\nu}}_\infty \leq 1,
\]
which, because $n$ and $\delta$ are strictly non-negative, rearranges to the required bound
\[
\boxed{
\delta \leq \frac{1}{n}\norm{\hat{\nu}}_\infty^{-1}
} \;.
\]
Conversely, if $\delta > (n \norm{\hat{\nu}}_\infty)^{-1}$, we would have
\[
\norm{\hat{s}}_\infty = n \delta \norm{\hat{\nu}}_\infty > 1,
\]
which violates the bound required for $\hat{s}$ to be a valid element of the subdifferential $\partial \norm{0}_1$. This means no such subgradient exists, $0 \notin \partial \empadvpcarisk{\widehat{\beta}}$, and therefore $\widehat{\beta}$ is not the global minimizer.

\end{proof}

\subsection{Proof of Proposition \ref{prop:5}}
This proof is adapted from \cite[prop. 3]{ribeiro_regularization_2023}.
\begin{proof}
We want to prove that the zero solution $\widehat{\beta} = 0$ minimizes the Adversarial PCA risk~\eqref{eq:advpca_risk}, given by
\[
\widehat{\mathcal{R}}_\delta(\beta) = \sum_{i=1}^n \left( \lvert \alpha^\top x_i - \beta^\top x_i \rvert + \delta \norm[1]{\beta} \right) ^2,
\]
if and only if 
\[
\delta \geq \delta_{\mathrm{max}} = \frac{\norm{X^\top X \alpha}_\infty}{\norm[1]{X \alpha}}.
\]
To prove the proposition, we use the subdifferential of the risk; the complete derivation can be found in Appendix \ref{app:subgradient}. The subdifferential of the risk satisfies
\[
\frac{1}{2} \partial \empadvpcarisk{\beta} = \{ X^\top X \left( \beta - \alpha \right) \} + \delta \norm[1]{X(\beta - \alpha)} z - \delta \norm[1]{\beta} X^\top s + n \delta^2 \norm[1]{\beta} z.
\]
Because $\widehat{\mathcal{R}}_\delta$ is convex, $\widehat{\beta} = 0$ is a global minimizer if and only if $0 \in \partial \widehat{\mathcal{R}}_\delta(0)$. Evaluating the subdifferential at $\beta = 0$ simplifies the expression significantly, yielding the set
\[
\frac{1}{2} \partial \empadvpcarisk{0} = -\{ X^\top X \alpha \} + \delta \norm[1]{X \alpha} \partial \norm[1]{0}.
\]
For the zero vector to be in this set, there must exist some specific subgradient $\hat{z} \in \partial \norm[1]{0}$ such that
\[
0 = -X^\top X \alpha + \delta \norm[1]{X \alpha} \hat{z}.
\]
Solving for $\hat{z}$ then gives
\[
\hat{z} = \frac{1}{\delta \norm{X \alpha}_1} X^\top X \alpha,
\]
and taking the $\ell_\infty$-norm on both sides gives
\[
\norm{\hat{z}}_\infty = \frac{1}{\delta \norm{X \alpha}_1} \norm{X^\top X \alpha}_\infty.
\]
From the definition of the subdifferential of the absolute value function (see Appendix \ref{app:subgradient}), any element $\hat{z} \in \partial \norm[1]{0}$ must satisfy $\norm{\hat{z}}_\infty \leq 1$. Thus, if
\[
\delta = \frac{\norm{X^\top X \alpha}_\infty}{\norm{X \alpha}_1}, 
\]
then $\norm{\hat{z}}_\infty = 1$ and subsequently $0 \in \partial \empadvpcarisk{0}$, implying $\beta=0$ is the global minimizer. The same goes for 
\[
\delta > \frac{\norm{X^\top X \alpha}_\infty}{\norm{X \alpha}_1}
\]
since this implies $\norm{\hat{z}}_\infty < 1$. However, if
\[
\delta < \frac{\norm{X^\top X \alpha}_\infty}{\norm{X \alpha}_1},
\]
then $\norm{\hat{z}}_\infty > 1$, meaning no such subgradient $\hat{z} \in \partial \norm[1]{0}$ can exist, and consequently $\beta=0$ is not the global minimizer of $\empadvpcarisk{\beta}$. Thus, the statement is proven.

Finally, we show that $\delta_{\mathrm{max}}$ is directly linked to the singular values of $X$ when $\alpha$ is a standard principal component. Consider $k \leq d$ components and let the SVD of $X \in \R^{n \times d}$ be given by $U \Sigma V^\top$, where $U \in \R^{n \times n}$ such that $U^\top U = U U^\top = I_n$, $\Sigma \in \R^{n \times d}$, and $V \in \R^{d \times d}$ such that $V^\top V = V V^\top = I_d$. This implies that
\[
    X^\top X = V \Sigma^\top U^\top U \Sigma V^\top = V\Sigma^\top \Sigma V^\top.
\]
Consider the $j$th subproblem where $\alpha$ is set to the $j$th right singular vector, namely $\alpha = v_j$. We then have
\[
    X^\top X \alpha = V \Sigma^\top \Sigma V^\top v_j = V \Sigma^\top \Sigma e_j = \sigma_j^2 v_j,
\]
where $\sigma_j$ is the corresponding singular value. Taking the $\ell_\infty$-norm gives
\[
    \norm[\infty]{X^\top X \alpha} = \sigma_j^2 \norm[\infty]{v_j}.
\]
The denominator can similarly be written as
\[
    X\alpha = U\Sigma V^\top v_j = U \Sigma e_j = \sigma_j u_j,
\]
where $u_j$ denotes the $j$th left singular vector. Taking the $\ell_1$-norm yields
\[
    \norm[1]{X\alpha} = \sigma_j \norm[1]{u_j}.
\]
By substituting these back into the expression for $\delta_{\mathrm{max}}$, we conclude that
\[
\boxed{
    \delta_{\mathrm{max}} = \frac{\norm[\infty]{X^\top X \alpha}}{\norm[1]{X \alpha}} = \frac{\sigma_j^2 \norm[\infty]{v_j}}{\sigma_j \norm[1]{u_j}} = \sigma_j \frac{\norm[\infty]{v_j}}{\norm[1]{u_j}}
    } \; .
\]

\end{proof}

\subsection{Proof of Proposition \ref{prop:6}}\label{app:proof:prop6}
\textbf{Proposition bound: } 

\begin{proof}
As before, for a fixed vector $\alpha$ and $\delta \geq 0$, we define the Adversarial PCA risk \eqref{eq:advpca_risk} as 
\[
\widehat{\mathcal{R}}_\delta(\beta) = \frac{1}{n} \sum_{i=1}^n \left( \lvert x_i^\top \alpha - x_i^\top \beta \rvert + \delta \norm[1]{\beta} \right) ^2.
\]
Furthermore, let $\widehat{\beta} \in \argmin_{\beta} \empadvpcarisk{\beta}$ minimize the empirical risk and $\beta_*$ the population counterpart. Using this notation, we let $\widehat{\Delta} = \widehat{\beta} - \beta_*$ and want to find a bound for $\snorm{\widehat{\beta} - \beta_*}_{\Sigma_*}^2 = \snorm{\widehat{\Delta}}_{\Sigma_*}^2$, where we let $\Sigma_*$ denote the true covariance of the data $\mathcal{D} = \{ x_i \in \R^d: i \in [n] \}$.

Now, since $\widehat{\mathcal{R}}_\delta(\widehat{\beta}) \leq \widehat{\mathcal{R}}_\delta(\beta_*)$ by definition, we have that 
\[
0 \leq \snorm{\widehat{\Delta}}_{\Sigma_*}^2 \leq \snorm{\widehat{\Delta}}_{\Sigma_*}^2 + \underbrace{\widehat{\mathcal{R}}_\delta(\beta_*)  - \widehat{\mathcal{R}}_\delta(\widehat{\beta})}_{\geq 0}.
\]
The difference in risk can, after some algebra, be shown to equal
\begin{align*}
\widehat{\mathcal{R}}_\delta(\beta_*)  - \widehat{\mathcal{R}}_\delta(\widehat{\beta}) = 
&-\frac{1}{n}\snorm{X\widehat{\Delta}}_2^2 
+ \frac{2}{n} (\alpha - \beta_*)^\top X^\top X \widehat{\Delta} \\
&+ \frac{2\delta}{n} \left( \norm[1]{\beta_*} \norm[1]{X(\alpha - \beta_*)} - \snorm{\widehat{\beta}}_1 \snorm{X(\alpha - \widehat{\beta})}_1 \right) \\
&+\delta^2 ( \norm[1]{\beta_*}^2 - \snorm{\widehat{\beta}}_1^2 ) \geq 0.
\end{align*}
We simplify the $2\delta/n$-term: 
\begin{align*}
\norm[1]{\beta_*} \norm[1]{X(\alpha - \beta_*)} - \snorm{\widehat{\beta}}_1 \snorm{X(\alpha - \widehat{\beta})}_1 
&\leqstack{(a)} \snorm{\widehat{\beta}}_1 \snorm{X\widehat{\Delta}}_1 + \norm[1]{X(\alpha - \beta_*)} \left( \norm[1]{\beta_*} - \snorm{\widehat{\beta}}_1 \right) \\
&\leqstack{(b)} \scalebox{0.95}{$\sqrt{n} \snorm{\widehat{\beta}}_1 \snorm{X\widehat{\Delta}}_2 + \norm[1]{X(\alpha - \beta_*)} \left( \norm[1]{\beta_*} - \snorm{\widehat{\beta}}_1 \right)$},
\end{align*}
where (a) is the reverse triangle inequality and (b) the fact that $\norm[1]{\cdot} \leq \sqrt{n} \norm[2]{\cdot}$. Rearranging then gives
\begin{align*}
 \frac{1}{n}\snorm{X\widehat{\Delta}}_2^2 
 &-\frac{1}{\sqrt{n}}\snorm{X\widehat{\Delta}}_2 2 \delta \snorm{\widehat{\beta}}_1 \\
 &- \scaleeq{\left( \frac{2}{n} (\alpha - \beta_*)^\top X^\top X \widehat{\Delta}
+ \frac{2 \delta}{n}\norm[1]{X(\alpha - \beta_*)} \left( \norm[1]{\beta_*} - \snorm{\widehat{\beta}}_1 \right)
+ \delta^2 ( \norm[1]{\beta_*}^2 - \snorm{\widehat{\beta}}_1^2 )  \right) \leq 0}{0.95}.
\end{align*}
We now identify that we have something on the form $f(y) = y^2 -by-c \leq 0$ for $y = \frac{1}{\sqrt{n}}\snorm{X\widehat{\Delta}}_2$, and that we can use the following lemma:
\begin{lemma}\label{app:lemma:f}
    Consider $f: \mathbb{R} \to \mathbb{R}$ given by $f(x) = x^2 - bx - c$, where $b,c \in \mathbb{R}$. If $f(x) \geq 0 \; \forall x$, then $b^2 + 4c \leq 0$. Similarly, if $\exists \, x$ s.t. $f(x) \leq 0$, then $b^2 + 4c \geq 0$.
\end{lemma}
\begin{proof}
    \[
    f(x) = x^2 - bx - c = \left( x - \frac{b}{2} \right)^2 - \frac{b^2 + 4c}{4}.
    \] 
    Since $\left( x - \frac{b}{2} \right)^2 \geq 0$ for all $x \in \mathbb{R}$, the global minimum of the function occurs at $x = \frac{b}{2}$, where $f\left(\frac{b}{2}\right) = -\frac{b^2 + 4c}{4}$. If $f(x) \geq 0$ for all $x$, then specifically $f\left(\frac{b}{2}\right) \geq 0$, which immediately requires $b^2 + 4c \leq 0$. Conversely, if $b^2 + 4c \geq 0$, then $f\left(\frac{b}{2}\right) \leq 0$, demonstrating there exists an $x$ such that $f(x) \leq 0$.
    
\end{proof}
Since $f(y) \leq 0$, we have to have $b^2 + 4c \geq 0$. After simplifying, this gives
\begin{align}\label{app:b2+4cgeq0}
\frac{2}{n} (\alpha - \beta_*)^\top X^\top X \widehat{\Delta} + \frac{2 \delta}{n}\norm[1]{X(\alpha - \beta_*)} \left( \norm[1]{\beta_*} - \snorm{\widehat{\beta}}_1 \right) + \delta^2 \norm[1]{\beta_*}^2 \geq 0.
\end{align}
Since~\eqref{app:b2+4cgeq0} holds by Lemma \ref{app:lemma:f}, the conditions are met for a second one:
\begin{lemma}
    Consider $f: \mathbb{R} \to \mathbb{R}$ given by $f(x) = x^2 - bx - c$ and $g: \mathbb{R} \to \mathbb{R}$ given by $g(x) = x^2 - b^2 - 2c$ where $b,c \in \mathbb{R}$. If $b^2 + 4c \geq 0$ and $f(x) \leq 0$, then $g(x) \leq f(x)$.
\end{lemma}
\begin{proof}
    We aim to show that $f(x) - g(x) \geq 0$. First, we expand the difference between the two functions:
    \[
        f(x) - g(x) = (x^2 - bx - c) - (x^2 - b^2 - 2c) = b^2 - bx + c.
    \]
    Observe that we can construct this exact expression by expanding the perfect square $(x-b)^2$ and subtracting our function $f(x)$:
    \[
        (x-b)^2 - f(x) = (x^2 - 2bx + b^2) - (x^2 - bx - c) = b^2 - bx + c.
    \]
    Therefore, we establish the direct algebraic identity:
    \[
        f(x) - g(x) = (x-b)^2 - f(x).
    \]
    Since the square of any real number is non-negative, we know $(x-b)^2 \geq 0$. We are explicitly given the condition that $f(x) \leq 0$, which implies $-f(x) \geq 0$. 
    Because $f(x) - g(x)$ is expressed as the sum of two non-negative terms, it must also be non-negative:
    \[
        f(x) - g(x) \geq 0 \implies g(x) \leq f(x).
    \]
    \textit{Note:} The condition $b^2 + 4c \geq 0$ guarantees that the premise $f(x) \leq 0$ is attainable for some $x \in \mathbb{R}$.

\end{proof}
Thus, we have $y^2-b^2 - 2c \leq y^2 - by - c$. This gives
\begin{align*}
     - \frac{1}{n}\snorm{X\widehat{\Delta}}_2^2 
 + \frac{4}{n} (\alpha - \beta_*)^\top X^\top X \widehat{\Delta} 
 &+ \frac{4 \delta}{n} \norm[1]{X(\alpha - \beta_*)} \left( \norm[1]{\beta_*} - \snorm{\widehat{\beta}}_1 \right) \\
&+2 \delta^2\left( \norm[1]{\beta_*}^2 + \snorm{\widehat{\beta}}_1^2 \right)  \geq 0.
\end{align*}
To summarize, we have shown
\begin{align*}
    0 \leq \widehat{\mathcal{R}}_\delta(\beta_*)  - \widehat{\mathcal{R}}_\delta(\widehat{\beta}) \leq - \frac{1}{n}\snorm{X\widehat{\Delta}}_2^2 
 + \frac{4}{n} (\alpha - \beta_*)^\top X^\top X \widehat{\Delta} 
 &+ \scaleeq{\frac{4 \delta}{n} \norm[1]{X(\alpha - \beta_*)} \left( \norm[1]{\beta_*} - \snorm{\widehat{\beta}}_1 \right)}{1} \\
&+2 \delta^2\left( \norm[1]{\beta_*}^2 + \snorm{\widehat{\beta}}_1^2 \right),
\end{align*}
and hence we have that
\begin{empheq}[box=\fbox]{align*}
0 \leq \snorm{\widehat{\Delta}}_{\Sigma_*}^2 
    &\leq \snorm{\widehat{\Delta}}_{\Sigma_* - \widehat{\Sigma}}^2 + \frac{4}{n} (\alpha - \beta_*)^\top X^\top X \widehat{\Delta} \\
    &\quad + \frac{4 \delta}{n} \norm[1]{X(\alpha - \beta_*)} \left( \norm[1]{\beta_*} - \snorm{\widehat{\beta}}_1 \right) \\
    &\quad + 2 \delta^2\left( \norm[1]{\beta_*}^2 + \snorm{\widehat{\beta}}_1^2 \right)
\end{empheq}
where we have used that $\widehat{\Sigma} = \frac{1}{n} X^\top X$ denotes the empirical covariance.

\end{proof}

\textbf{Rates mentioned in the remark:} 

\begin{proof}
We start by looking at the term
\[
\frac{1}{n} \norm[1]{X(\alpha-\beta_*)} = \frac{1}{n} \sum_{i=1}^n |\underbrace{(\alpha - \beta_*)^\top x_i}_{y_i \in \mathbb{R}}|.
\]
$\mathbb{E}[y] = 0$ and $\mathbb{V}[y] = (\alpha-\beta_*)^\top \Sigma_* (\alpha-\beta_*)$; letting $(\alpha-\beta_*)^\top \Sigma_* (\alpha-\beta_*) = \sigma_y^2$ therefore means $y \sim \mathcal{N}(0, \sigma_y^2)$. $|y|$ follows a folded normal distribution \cite{tsagris_folded_2014}, and hence has mean 
\[
\mathbb{E}[|y|] = \sigma_y\sqrt{\frac{2}{\pi}}
\]
and variance 
\[
\mathbb{V}[|y|] = \sigma_y^2 - \frac{2}{\pi}\sigma_y^2 = \frac{\pi - 2}{\pi}\sigma_y^2.
\]
Furthermore, $\abs{y}$ is sub-Gaussian with parameter $\sigma_y$. To see this, we can write $y = \sigma_y z$, where $z \sim \mathcal{N}(0,1)$. We then define the function $f(z) = \sigma_y \abs{z} = \abs{y}$, and note that
\[
\abs{f(z_1) - f(z_2)} = \sigma_y \abs{\abs{z_1} - \abs{z_2}} \leq \sigma_y \abs{z_1 - z_2} = \sigma_y \norm{z_1 - z_2}_2,
\]
meaning $f$ is a $\sigma_y$-Lipschitz function applied to a standard Gaussian variable. By Gaussian concentration for Lipschitz functions (see e.g. \cite[Thm. 2.26]{wainwright_high-dimensional_2019}), $f(z)=\abs{y}$ is sub-Gaussian with parameter $\sigma_y$.

We now use the Hoeffding bound for a sum of independent sub-Gaussian variables, see e.g. \cite[prop. 2.5]{wainwright_high-dimensional_2019}:
\[
\mathbb{P}\left( \sum_{i=1}^n \left( |y_i| - \sigma_y \sqrt{\frac{2}{\pi}} \right) \geq t \right) \leq \exp \left\{ -\frac{t^2}{2\sum_{i=1}^n \sigma_y^2 }\right\},
\]
so
\[
\mathbb{P}\left( \frac{1}{n} \sum_{i=1}^n |y_i| \geq \frac{t + n \sigma_y \sqrt{2/\pi}}{n} \right) \leq \exp \left\{ -\frac{t^2}{2n\sigma_y^2} \right\},
\]
and 
\[
\mathbb{P}\left( \frac{1}{n} \sum_{i=1}^n |y_i| \leq \frac{t+n\sigma_y \sqrt{2/\pi}}{n} \right) \geq 1 - \exp \left\{ -\frac{t^2}{2n\sigma_y^2} \right\}.
\]

Letting 
\[
\exp \left\{ -\frac{t^2}{2n\sigma_y^2} \right\} = 2\gamma \quad \text{for } \gamma \in (0,1/2)
\]
finally implies that 
\begin{equation*}
\boxed{
\frac{1}{n} \norm[1]{X(\alpha-\beta_*)} = \frac{1}{n} \sum_{i=1}^n |y_i| \leq \frac{\sqrt{2} \sigma_y}{\sqrt{\pi}} + \frac{\sqrt{2} \sigma_y}{\sqrt{n}} \sqrt{\ln{(\gamma^{-1}/2})}
}
\end{equation*}
with probability at least $1 - 2\gamma$. We move on and consider $\snorm{\widehat{\Delta}}_{\Sigma_* - \widehat{\Sigma}}^2$. We have
\[
\snorm{\widehat{\Delta}}_{\Sigma_* - \widehat{\Sigma}}^2 = \widehat{\Delta}^\top (\Sigma_* - \widehat{\Sigma}) \widehat{\Delta} \leq \snorm{\Sigma_* - \widehat{\Sigma}}_\infty \snorm{\widehat{\Delta}}_1^2.
\]
We now want to find a bound for 
\[
\snorm{\Sigma_*-\widehat{\Sigma}}_\infty. 
\]
Using the spiked model, described in detail in Appendix \ref{app:spiked_cov}, we get
\begin{align*}
    \Sigma_*-\widehat{\Sigma}
    = \tau \left( 1-\frac{1}{n} \sum_{i=1}^n z_i^2 \right) \beta_* {\beta_*}^\top 
    -\sqrt{\tau} \beta_* \left( \frac{1}{n} \sum_{i=1}^n z_i w_i^\top \right) 
    &-\left[ \sqrt{\tau} \beta_* \left( \frac{1}{n} \sum_{i=1}^n z_i w_i^\top \right) \right]^\top \\
    &+ I_d - \frac{1}{n} \sum_{i=1}^n w_i w_i^\top, 
\end{align*}
and then applying the triangle inequality gives
\begin{align*}
\scaleeq{\norm{\Sigma_* - \widehat{\Sigma}}_\infty
    \leq 
    \underbrace{\norm{\tau \left(\frac{1}{n} \sum_{i=1}^n z_i^2 -1\right) \beta_* {\beta_*}^\top}_\infty}_{\text{\textcircled{$\ast$}}}
    + \underbrace{2\norm{\sqrt{\tau} \left( \frac{1}{n} \sum_{i=1}^n z_i w_i \right) {\beta_*}^\top }_\infty}_{\text{\textcircled{$\ast$}}\text{\textcircled{$\ast$}}}
    + \underbrace{\norm{I_d - \frac{1}{n} \sum_{i=1}^n w_i w_i^\top}_\infty}_{\text{\textcircled{$\ast$}}\text{\textcircled{$\ast$}}\text{\textcircled{$\ast$}}}}{0.93}.
\end{align*}
\textbf{\underline{\textcircled{$\ast$}:}}

We note that $\frac{1}{n}\sum_{i=1}^n z_i^2 - 1$ is a scalar and that $\sum_{i=1}^n z_i^2 \sim \chi^2(n)$. Following \cite[ex. 2.11]{wainwright_high-dimensional_2019}, we have that this particular chi-squared is sub-exponential with parameters $(\nu, \alpha) = (2\sqrt{n},4)$, which gives the two-sided bound
\[
\prob{\abs{\frac{1}{n}\sum_{i=1}^n z_i^2 - 1} \geq t} \leq 2e^{-nt^2/8} \quad \forall t \in (0,1).
\]
Thus,
\[
\prob{\abs{\frac{1}{n}\sum_{i=1}^n z_i^2 - 1} \leq t} \geq 1 - 2e^{-nt^2/8}, 
\]
and letting $e^{-nt^2/8} = \gamma$ for $\gamma \in (0,1/2)$ therefore gives
\[
\abs{\frac{1}{n}\sum_{i=1}^n z_i^2 - 1} \leq \frac{2\sqrt{2}}{\sqrt{n}} \sqrt{\ln{(\gamma^{-1})}}
\]
with probability at least $1-2\gamma$. Using this and the fact that $\snorm{\beta_* {\beta_*}^\top}_\infty = \norm{\beta_*}_\infty^2$, we have
\begin{equation*}
\boxed{
\norm{\tau \left(\frac{1}{n} \sum_{i=1}^n z_i^2 -1\right) \beta_* {\beta_*}^\top}_\infty = \tau \abs{\frac{1}{n} \sum_{i=1}^n z_i^2 -1} \norm{\beta_* {\beta_*}^\top}_\infty \leq \tau \frac{2\sqrt{2}}{\sqrt{n}} \sqrt{\ln{(\gamma^{-1})}} \norm{\beta_*}_\infty^2
}
\end{equation*}
with probability at least $1-2\gamma$.

Note here that this bound is only valid for $t \in (0,1)$, however the criterion is easily met for relatively low $n$; e.g., for $\gamma = 0.025$, $n$ can be as low as $30$.

\textbf{\underline{\textcircled{$\ast$}\textcircled{$\ast$}:}}

We denote the $j$th entry of $w_i$ as $w_{ij}$, meaning $w_{ij} \sim \normal{0,1} \; \forall (i,j) \in [n] \times [d]$. This means
\[
2\norm{\sqrt{\tau} \left( \frac{1}{n} \sum_{i=1}^n z_i w_i \right) {\beta_*}^\top }_\infty 
= \frac{2\sqrt{\tau}}{\sqrt{n}} \norm{\beta_*}_{\infty} \max_{j \in [d]} \abs{\frac{1}{\sqrt{n}}\sum_{i=1}^n z_i w_{ij}}.
\]
We first aim to bound the absolute sum $\frac{1}{\sqrt{n}}\abs{\sum_{i=1}^n z_i w_{ij}}$ for a fixed $j\in [d]$. From \cite[rem. 2.9.4]{vershyninHighdimensionalProbabilityIntroduction2026}, for every $t \geq 0$ we have that
\begin{equation}\label{eq:2.9.4}
\prob{\abs{\frac{1}{\sqrt{n}} \sum_{i=1}^n z_i w_{ij}}  \geq t} \leq 
\begin{dcases}
    2 e^{-c_1t^2} & t \leq \sqrt{n}, \\
    2 e^{-c_1t\sqrt{n}} & t \geq \sqrt{n}, \\
\end{dcases}
\end{equation}
since $\{z_i w_{ij}\}_{i=1}^n$ is a set of independent, zero mean, subexponential random variables for a constant $c_1>0$ depending only on the subexponential norm of $z_i w_{ij}$. For small deviations from the mean $(t \leq \sqrt{n})$, we get a Gaussian tail bound, and for large deviation $(t \geq \sqrt{n})$, we get heavier subexponential tails. 

Now, by using the properties of the union, we have that 
\begin{align*}
\prob{\max_{j \in [d]} \abs{\frac{1}{\sqrt{n}}\sum_{i=1}^n z_i w_{ij}} \geq t} 
&= \prob{ \bigcup_{j=1}^d \left\{ \abs{\frac{1}{\sqrt{n}}\sum_{i=1}^n z_i w_{ij}}  \geq t \right\}} \\
&\leq \sum_{j=1}^d \prob{ \abs{\frac{1}{\sqrt{n}}\sum_{i=1}^n z_i w_{ij}}  \geq t },
\end{align*}
and we can then apply~\eqref{eq:2.9.4} to each individual term. First, assuming $t \leq \sqrt{n}$, we get
\[
\prob{\max_{j \in [d]} \abs{\frac{1}{\sqrt{n}}\sum_{i=1}^n z_i w_{ij}} \geq t} 
\leq \sum_{j=1}^d 2e^{-c_1t^2}
= 2de^{-c_1t^2},
\]
which means 
\[
\prob{\max_{j \in [d]} \abs{\frac{1}{\sqrt{n}}\sum_{i=1}^n z_i w_{ij}} \leq t} 
\geq 1 - 2de^{-c_1t^2}. 
\]
By letting $de^{-c_1t^2} = \gamma$ for $\gamma \in (0,1/2)$, we get
\[
\max_{j \in [d]} \abs{\frac{1}{\sqrt{n}}\sum_{i=1}^n z_i w_{ij}} \leq \sqrt{\frac{\ln{(d/\gamma)}}{c_1}}
\]
with probability at least $1-2\gamma$. This holds as long as $n \geq \frac{\ln{(d/\gamma)}}{c_1}$.

If we use the same reasoning but instead assume the second case, i.e., $t \geq \sqrt{n}$, we get
\[
\max_{j \in [d]} \abs{\frac{1}{\sqrt{n}}\sum_{i=1}^n z_i w_{ij}} \leq \frac{\ln{(d/\gamma)}}{c_1\sqrt{n}}
\]
with probability at least $1-2\gamma$, where we have let $de^{-c_1t\sqrt{n}} = \gamma$. This holds as long as $n \leq \frac{\ln{(d/\gamma)}}{c_1}$.

Thus, we can conclude that
\begin{align*}
\boxed{
2\norm{\sqrt{\tau} \left( \frac{1}{n} \sum_{i=1}^n z_i w_i \right) {\beta_*}^\top }_\infty 
\leq 2\sqrt{\tau} \norm{\beta_*}_{\infty}
\begin{rcases}
\begin{dcases}
    \frac{\ln{(d/\gamma)}}{c_1 n} & n \leq \frac{\ln{(d/\gamma)}}{c_1} \\
    \sqrt{\frac{\ln{(d/\gamma)}}{nc_1}} & n \geq \frac{\ln{(d/\gamma)}}{c_1}
\end{dcases}
\end{rcases}
}
\end{align*}
with probability at least $1-2\gamma$ and for a constant $c_1>0$.

\textbf{\underline{\textcircled{$\ast$}\textcircled{$\ast$}\textcircled{$\ast$}:}}

For the third and final term, we note that the matrix $I_d - \frac{1}{n} \sum_{i=1}^n w_i w_i^\top$ now provides a mixture between \textcircled{$\ast$} and \textcircled{$\ast$}\textcircled{$\ast$}, with the important property that all entries of the matrix are sub-exponential with zero mean. By denoting
\[
W = \frac{1}{n} \sum_{i=1}^n w_i w_i^\top - I_d
\]
the diagonal elements are given by
\[
W_{jj} = \frac{1}{n} \sum_{i=1}^n w_{ij}^2 - 1,
\]
i.e., what we had in \textcircled{$\ast$}, and the off-diagonal elements are given by 
\[
W_{jk} = \frac{1}{n} \sum_{i=1}^n w_{ij}w_{ik}, \quad j \neq k
\]
which is what we dealt with in \textcircled{$\ast$}\textcircled{$\ast$}. By the same reasoning as before, and using $\norm{W}_\infty = \max_{j,k} \abs{W_{jk}}$, we have that 
\[
\prob{\max_{j,k} \abs{W_{jk}} \geq t} = \prob{\bigcup_{j,k} \left\{ \abs{W_{jk}} \geq t \right\}} \leq \sum_{j,k} \prob{\abs{W_{jk}} \geq t}.
\]
Again using \cite[rem. 2.9.4]{vershyninHighdimensionalProbabilityIntroduction2026} then gives
\begin{align*}
\boxed{
\norm{W}_\infty =  \norm{I_d - \frac{1}{n} \sum_{i=1}^n w_i w_i^\top}_\infty 
\leq 
\begin{dcases}
    \frac{2\ln{(d/\sqrt{\gamma})}}{c_2 n} & n \leq \frac{2\ln{(d/\sqrt{\gamma)}}}{c_2} \\
    \sqrt{\frac{2\ln{(d/\sqrt{\gamma})}}{nc_2}} & n \geq \frac{2\ln{(d/\sqrt{\gamma})}}{c_2}
\end{dcases}
}
\end{align*}
for a constant $c_2 > 0$. 

We conclude that using the spiked covariance model, treating $\tau$, $\gamma$ and $\snorm{\beta_*}_\infty$ as constants, with high probability we have 
\begin{equation*}
\boxed{
\norm{\Sigma_*-\widehat{\Sigma}}_\infty \lesssim \sqrt{\frac{\ln{d}}{n}}
}\;.
\end{equation*}

\end{proof}

\subsection{Proof of Algorithm~\ref{alg:eta-trick}}\label{app:proof:eta}
The algorithm we use was proposed by \citet{ribeiro_efficient_2025}. We outline the main steps of the proof and refer the reader to this paper for a more detailed description. 
\begin{proof}
We want to solve
\begin{align}\label{app:eq:eta_start}
    \argmin_{\beta} \; \sum_{i=1}^n \left(|\alpha^\top x_i - \beta^\top x_i| + \delta \norm[1]{\beta} \right)^2
\end{align}
As outlined already, by solving this subproblem,~\eqref{eq:fixA} is given by the sum of the individual solutions and thus the ''fix $A$, solve for $B$'' part is completed. Note here that we drop the $j$-subscript for simplicity. We denote the summand \textcircled{$\ast$}:
\[
\argmin_{\beta} \; \sum_{i=1}^n \underbrace{\left(|\alpha^\top x_i - \beta^\top x_i| + \delta \norm[1]{\beta} \right)^2}_{\text{\textcircled{$\ast$}}}.
\]
We now use something that is sometimes referred to as the $\eta$-trick (see e.g \cite{ItrickReloadedMultiple2019}):
\begin{align*}
    \norm[1]{\mathbf{w}}^2 = \underset{\mathbf{\eta} \in \Delta_d}{\text{min}} \sum_{\ell = 1}^d \frac{w_{\ell}^2}{\eta_{\ell}},
\end{align*}
where the set $\Delta_d$ is a simplex. If we think of $\norm[1]{\mathbf{w}}^2$ as \text{\textcircled{$\ast$}}, then we get
\begin{align*}
\text{\textcircled{$\ast$}} = \underset{\mathbf{\eta}^{(i)} \in \Delta_{d+1}^{(i)}}{\text{min}} \left\{ \frac{(\alpha^\top x_i - \beta^\top x_i)^2}{\eta_0^{(i)}} + \delta^2 \sum_{\ell=1}^d \frac{\beta_{\ell}^2}{\eta_{\ell}^{(i)}} \right\}.
\end{align*}
This is uniquely minimized at
\begin{align*}
    \eta_0^{(i)} = \frac{|\alpha^\top x_i - \beta^\top x_i|}{|\alpha^\top x_i - \beta^\top x_i| + \delta \norm[1]{\beta}}, \quad \text{and} \quad \eta_{\ell}^{(i)} = \frac{\delta |\beta_{\ell}|}{|\alpha^\top x_i - \beta^\top x_i| + \delta \norm[1]{\beta}}.
\end{align*}
If we set 
\begin{align*}
    w_i := \frac{1}{\eta_0^{(i)}}, \quad \text{and} \quad \gamma_{\ell} := \delta^2 \sum_{i=1}^n \frac{1}{\eta_{\ell}^{(i)}},
\end{align*}
we can write~\eqref{app:eq:eta_start} as 
\begin{align}\label{app:eq:eta_final}
\argmin_{\beta} \; \sum_{i=1}^n w_i(\alpha^\top x_i - \beta^\top x_i)^2 + \sum_{\ell = 1}^d \gamma_{\ell} \beta_{\ell}^2.
\end{align}
\eqref{app:eq:eta_final} is a  weighted ridge regression problem that can be solved using standard methods. To summarize, when we include the minimization over $\mathbf{\eta}$, we get a joint problem over $\beta$ and $\mathbf{\eta}$. Since the minimization over $\mathbf{\eta}$ is uniquely given in closed-form, we can iterate between fixing $\beta$ and $\mathbf{\eta}$ and solve for one at a time until convergence.

\end{proof}

\section{Additional Experimental Results}\label{app:sec:exp}

We provide additional experimental results to give better insight into the performance of the algorithm, and demonstrate how different experiment setups affect the results. 

\subsection{Spiked Covariance Data Model}\label{app:sec:spiked_cov}
For the experiments on synthetic data drawn from the spiked covariance model, we refer to Appendix~\ref{app:spiked_cov} for a detailed description of the data generating process.

\subsubsection{Sparse Versus Dense System}\label{app:sec:sparsevsdense}

We complement Figure \ref{fig:sparsevsdense} with additional plots where we vary the eigengap $\tau$. Other than $\tau$, the setup is exactly the same. For example, we use $n=150$ samples drawn from the spiked covariance model with true spike vector $\beta_* = e_1$, and do 20 repeats over different draws of data.
\begin{figure}[H]
    \centering
    \includegraphics[width=1\linewidth]{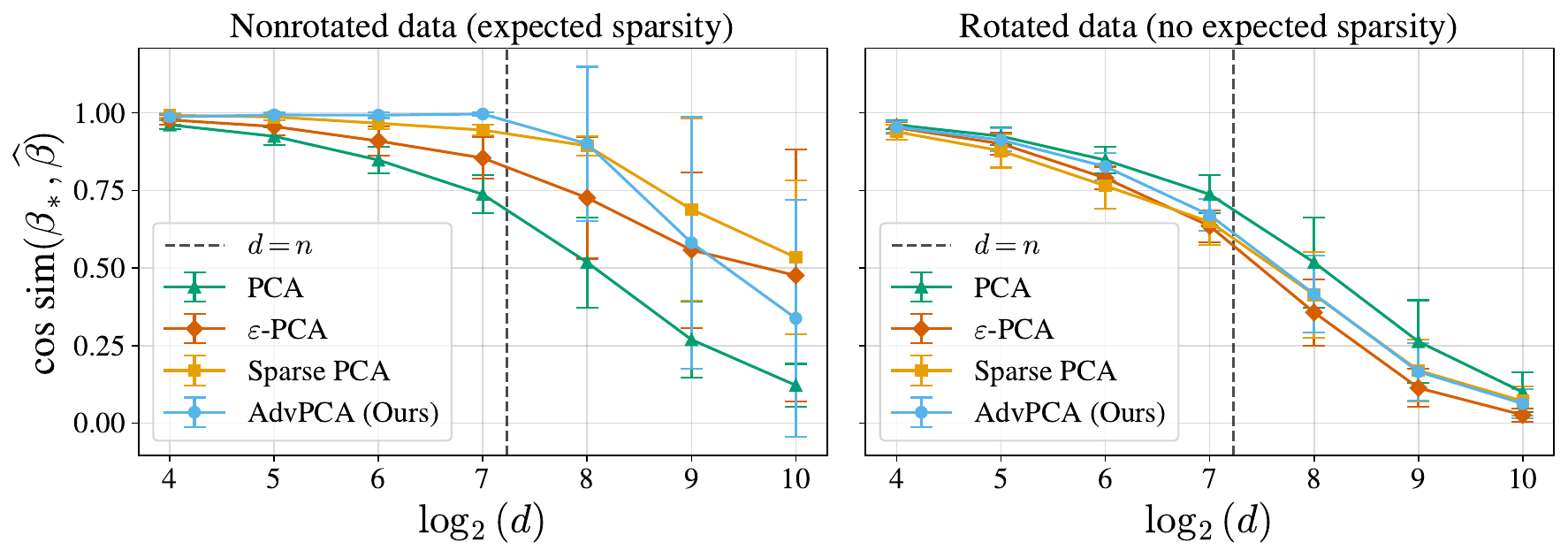}
    \caption{$\tau=2$.}
    \label{fig:placeholder}
\end{figure}
\begin{figure}[H]
    \centering
    \includegraphics[width=1\linewidth]{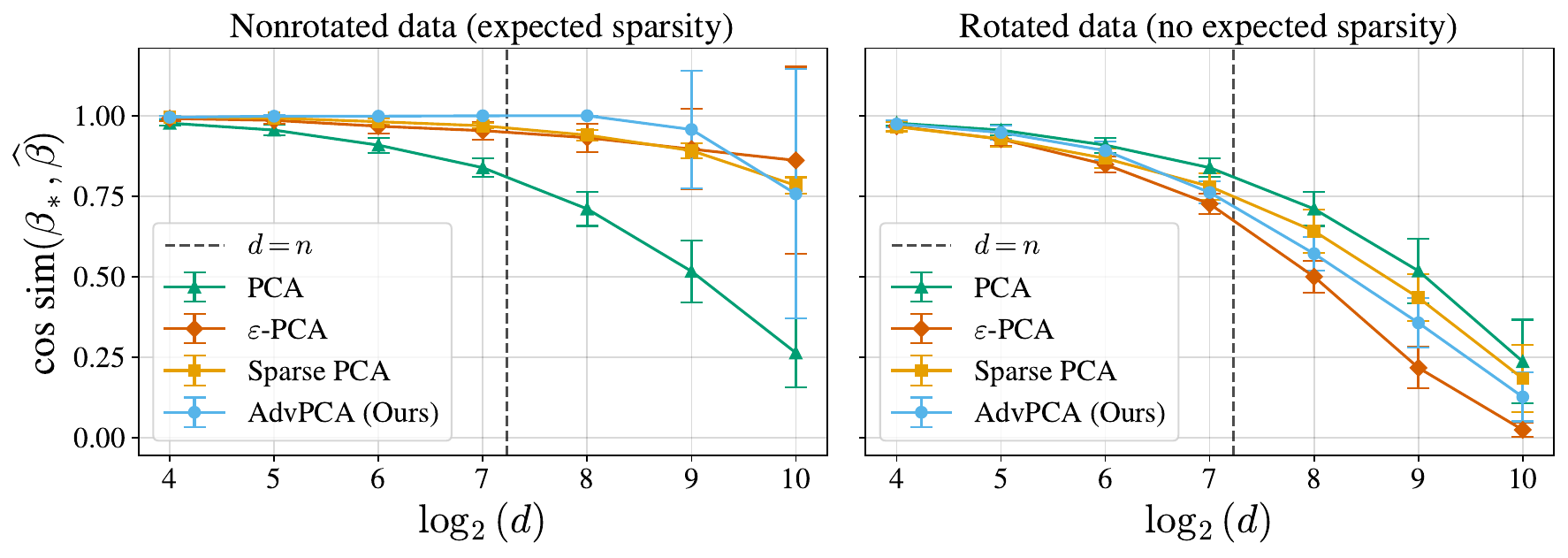}
    \caption{$\tau=3$.}
    \label{fig:placeholder}
\end{figure}
\begin{figure}[H]
    \centering
    \includegraphics[width=1\linewidth]{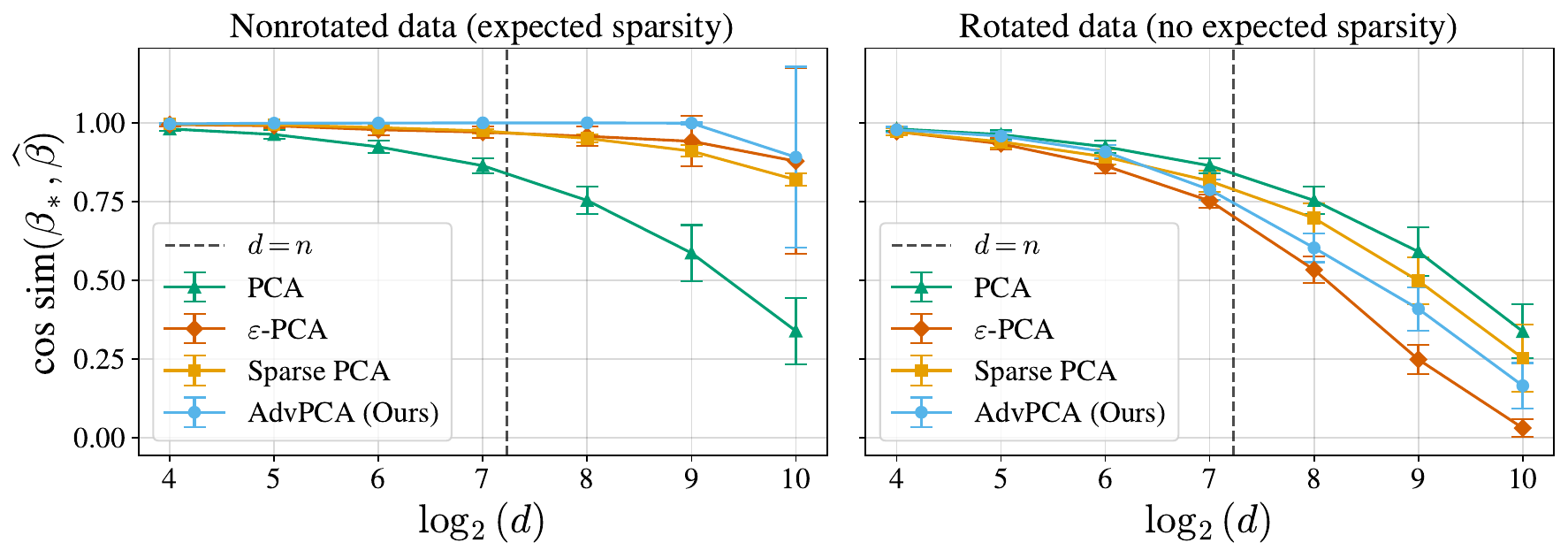}
    \caption{$\tau=3.5$.}
    \label{fig:placeholder}
\end{figure}
\begin{figure}[H]
    \centering
    \includegraphics[width=1\linewidth]{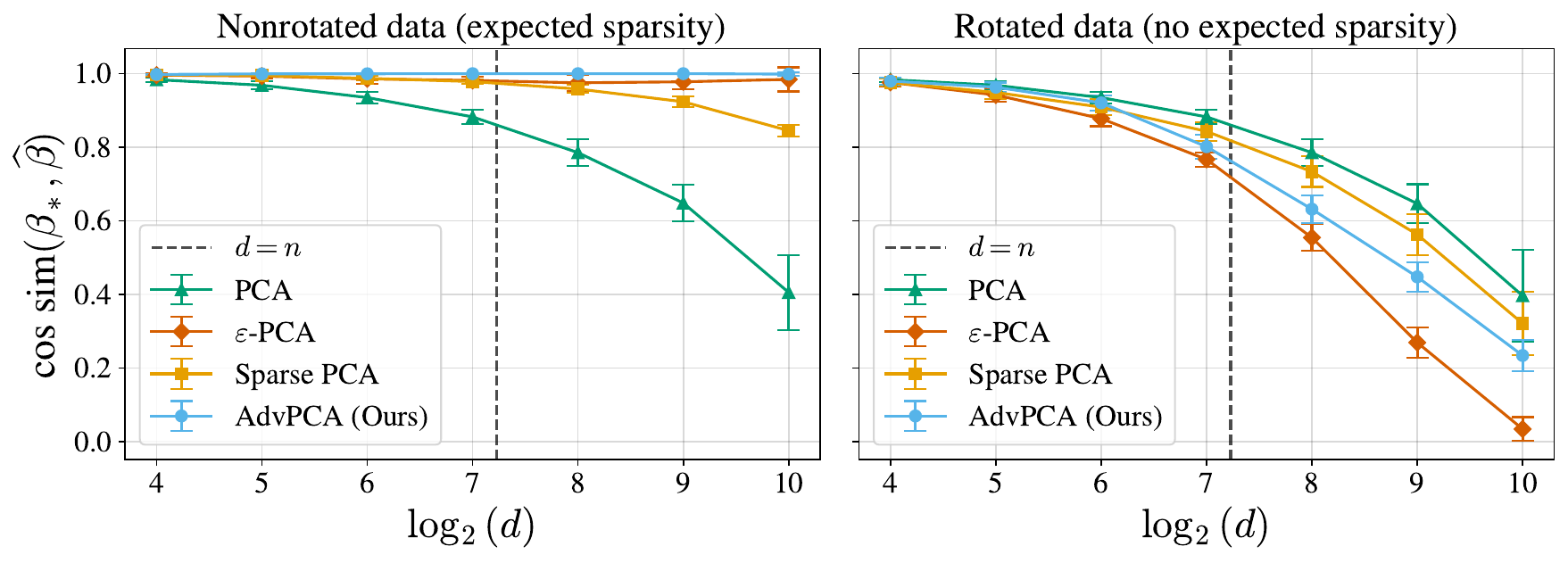}
    \caption{$\tau=4$.}
    \label{fig:placeholder}
\end{figure}

\newpage
We also perform the same experiment as a function of $n$ for a fixed $d=250$.
\begin{figure}[H]
    \centering
    \includegraphics[width=1\linewidth]{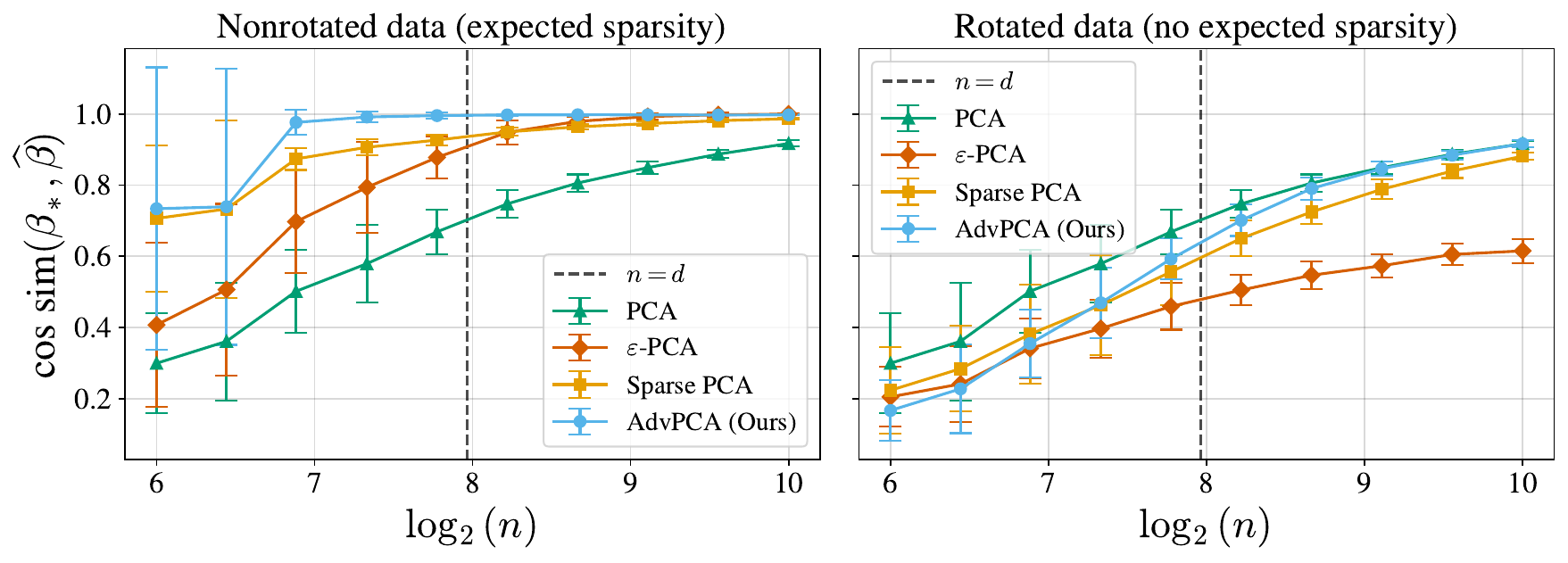}
    \caption{$\tau=2$.}
    \label{fig:placeholder}
\end{figure}
\begin{figure}[H]
    \centering
    \includegraphics[width=1\linewidth]{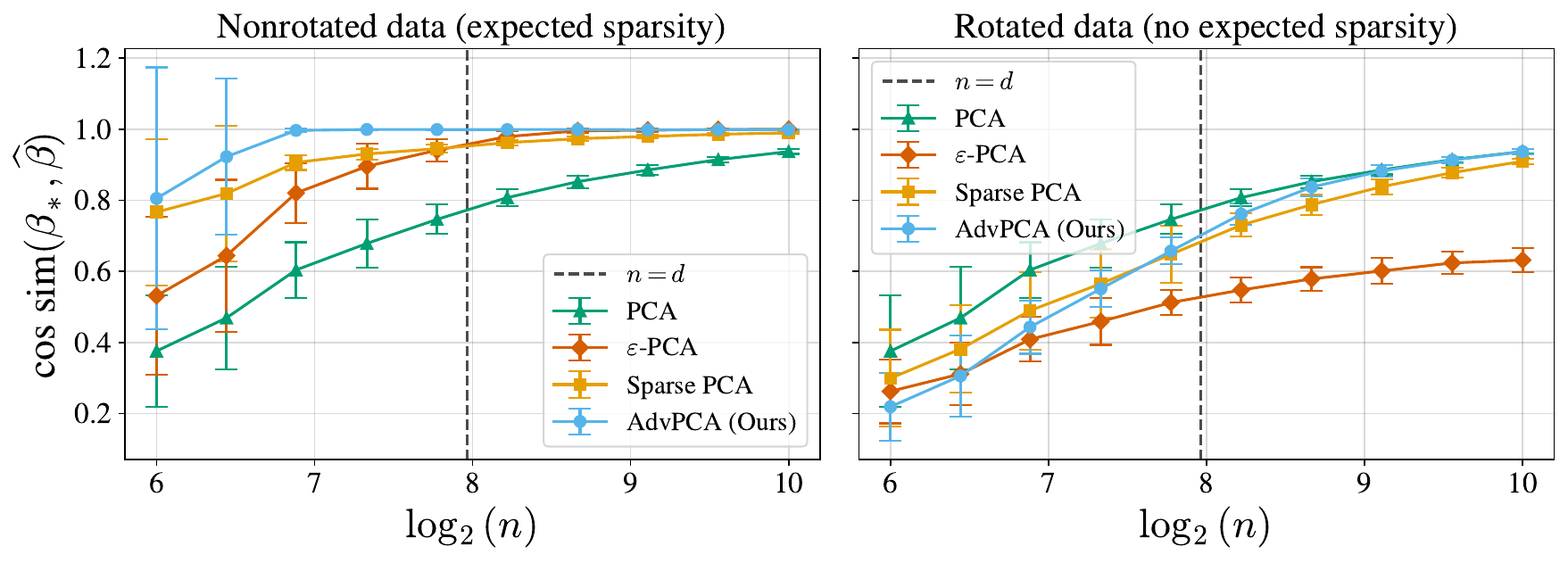}
    \caption{$\tau=2.5$.}
    \label{fig:placeholder}
\end{figure}
\begin{figure}[H]
    \centering
    \includegraphics[width=1\linewidth]{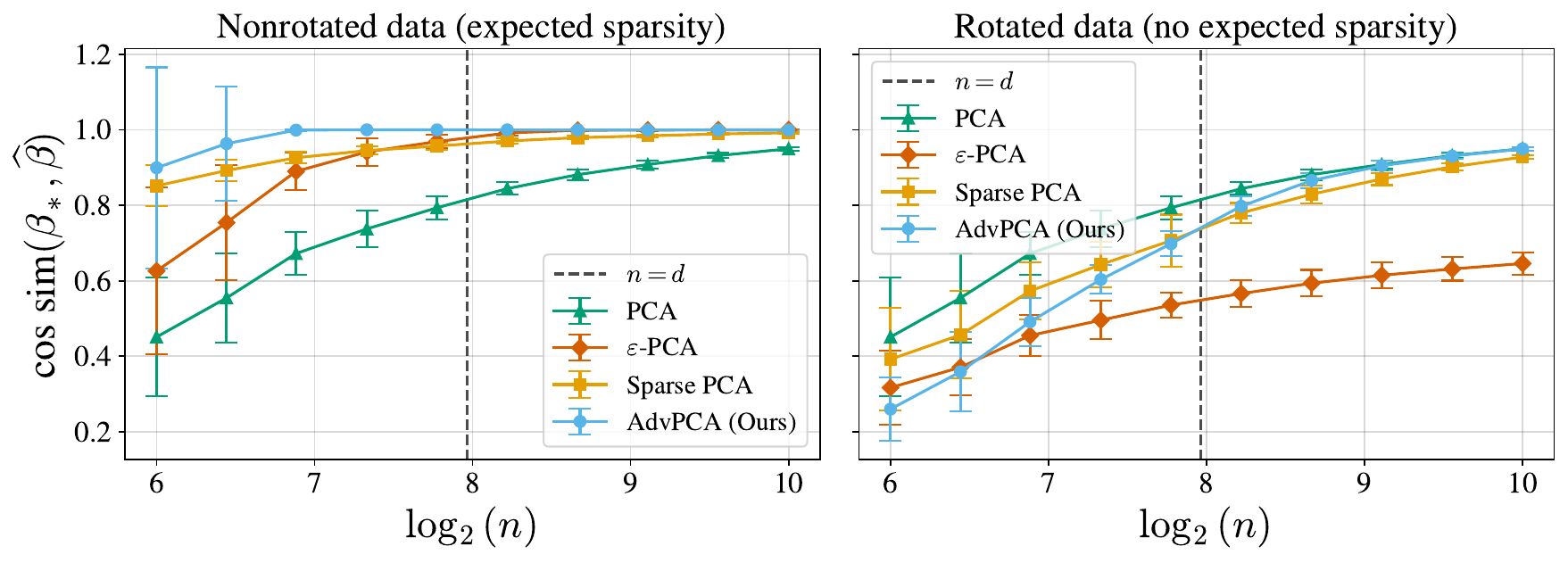}
    \caption{$\tau=3$.}
    \label{fig:placeholder}
\end{figure}
\begin{figure}[H]
    \centering
    \includegraphics[width=1\linewidth]{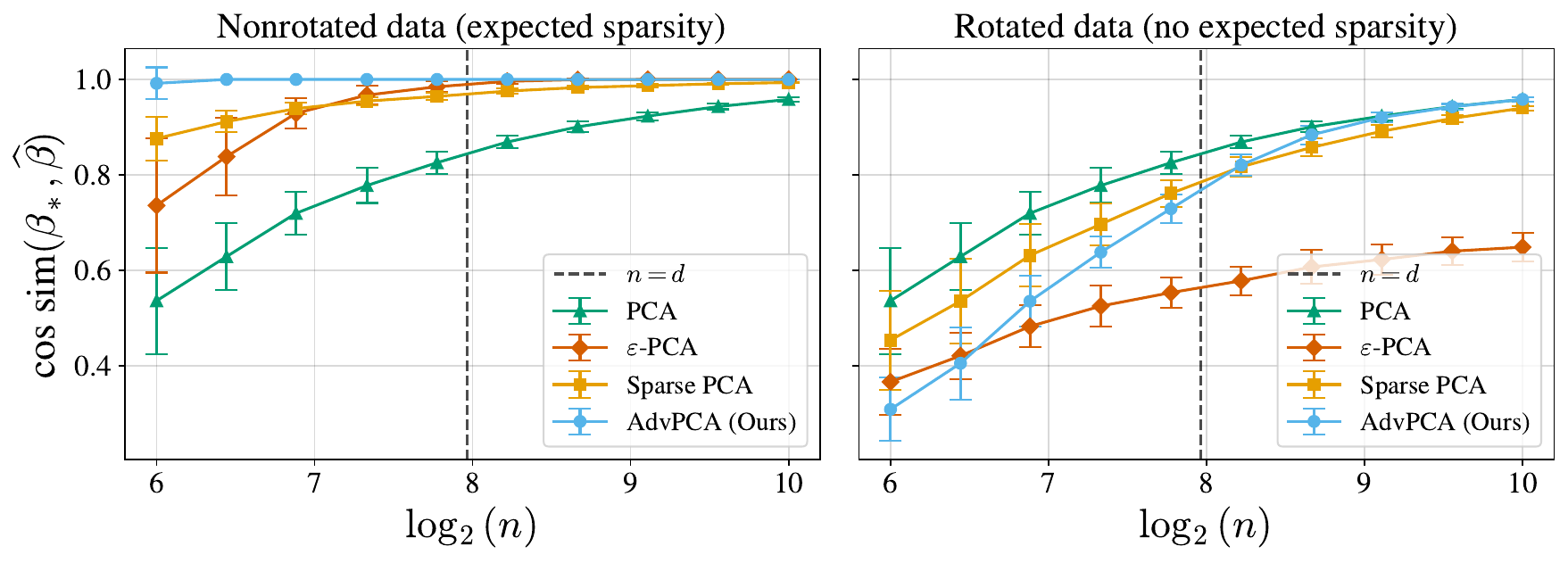}
    \caption{$\tau=3.5$.}
    \label{fig:placeholder}
\end{figure}
\begin{figure}[H]
    \centering
    \includegraphics[width=1\linewidth]{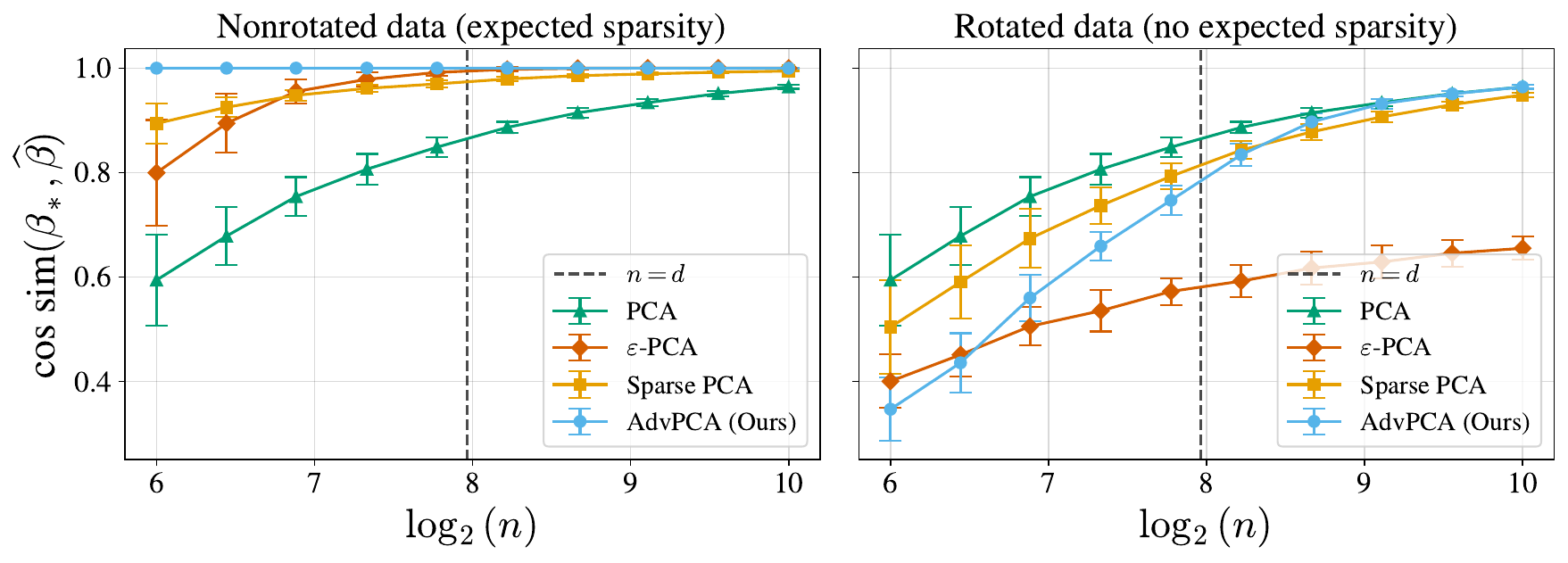}
    \caption{$\tau=4$.}
    \label{fig:placeholder}
\end{figure}

\subsubsection{Varying the Eigengap}

We complement Figure~\ref{fig:sparsevsdense} with an experiment where we plot the cosine similarity as a function of $\tau$. We let $n=500$ and show results for different values of $d$, illustrating both low and high-dimensional scenarios.
\begin{figure}[H]
    \centering
    \includegraphics[width=0.6\linewidth]{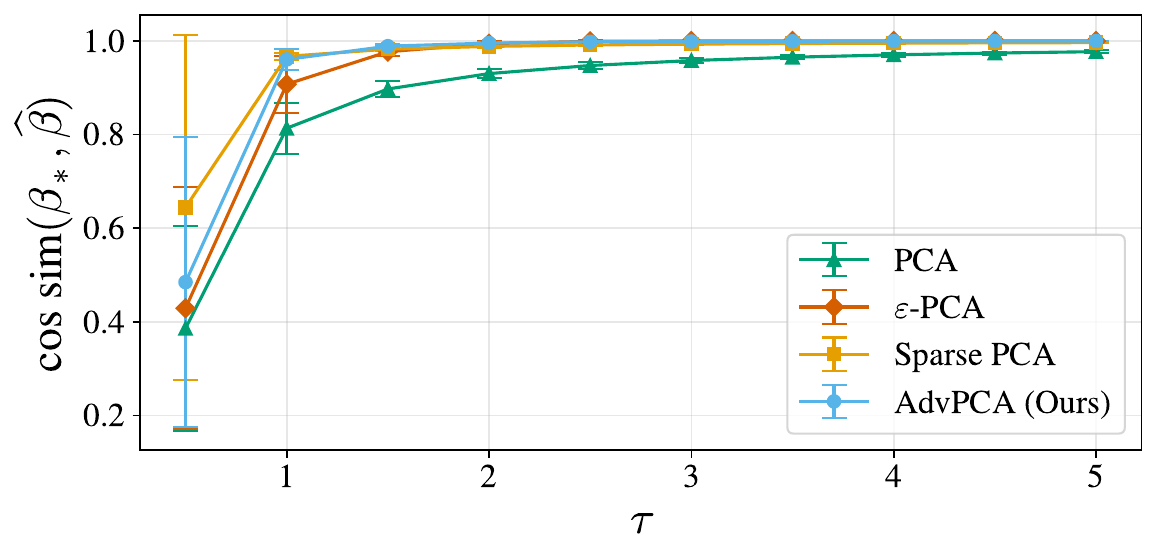}
    \caption{$d=100$.}
    \label{fig:placeholder}
\end{figure}
\begin{figure}[H]
    \centering
    \includegraphics[width=0.6\linewidth]{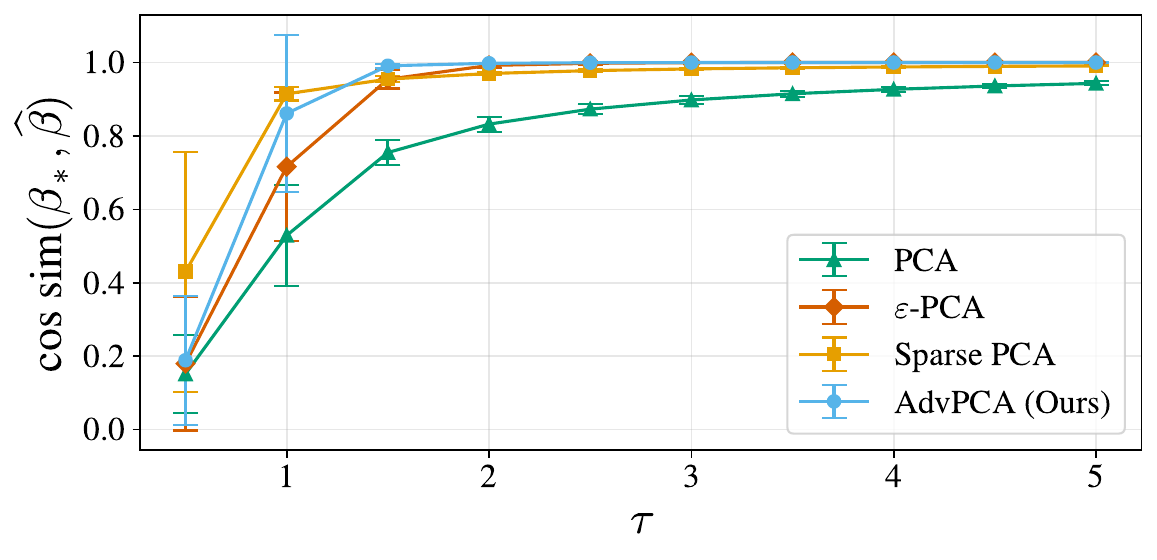}
    \caption{$d=250$.}
    \label{fig:placeholder}
\end{figure}
\begin{figure}[H]
    \centering
    \includegraphics[width=0.6\linewidth]{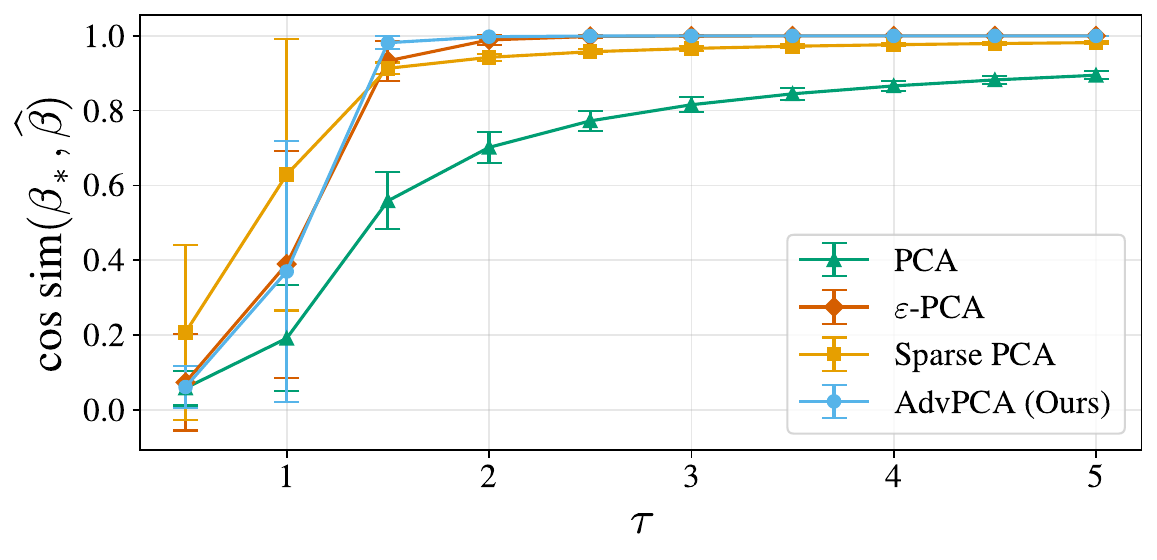}
    \caption{$d=500$.}
    \label{fig:placeholder}
\end{figure}
\begin{figure}[H]
    \centering
    \includegraphics[width=0.6\linewidth]{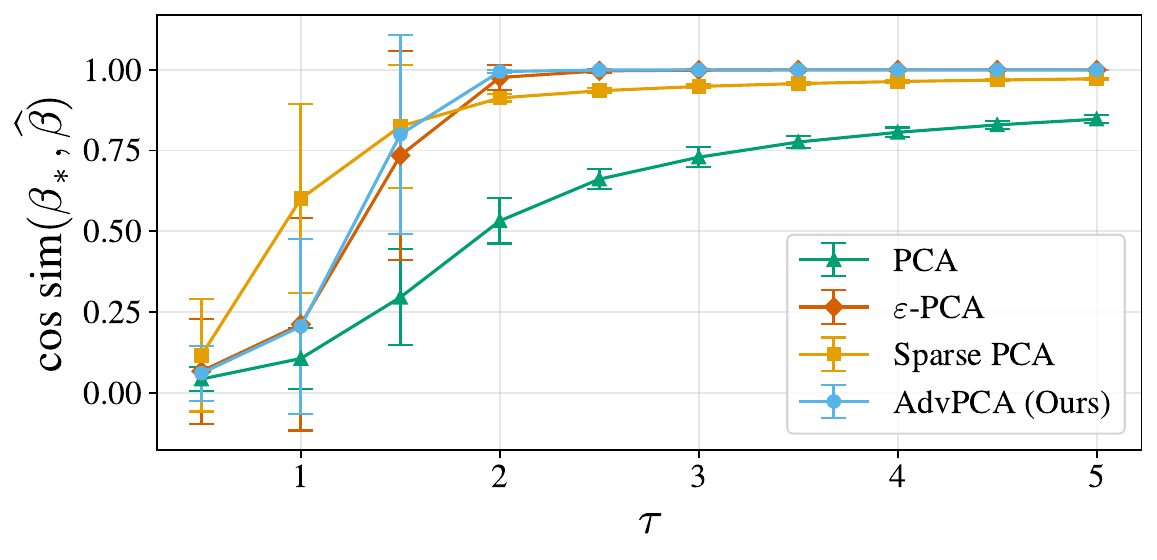}
    \caption{$d=750$.}
    \label{fig:placeholder}
\end{figure}
\begin{figure}[H]
    \centering
    \includegraphics[width=0.6\linewidth]{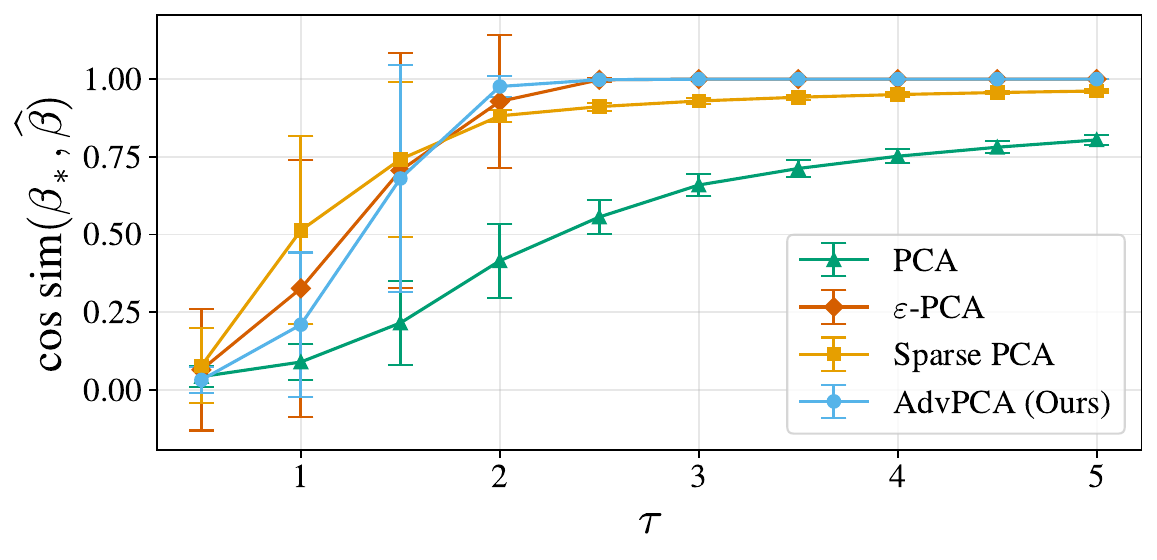}
    \caption{$d=1000$.}
    \label{fig:placeholder}
\end{figure}

\subsubsection{Multi-Component Recovery}\label{app:sec:multicomp}

In Figure~\ref{fig:multi-comp-recovery}, we spike in multiple directions, generalizing the rank-one model used in the top row to several directions. As described in Appendix~\ref{app:spiked_cov}, drawing data from the spiked covariance model requires us to set multiple eigengaps $\tau_1, \dots, \tau_k$, or equivalently, multiple eigenvalues $\lambda_1, \dots, \lambda_k$ where $\lambda_j = 1 + \tau_j \; \forall j \in [k]$. In Figure~\ref{fig:multi-comp-recovery} and all additional plots in this section, we set the eigenvalues $\{ \lambda_j \}_{j=1}^k$ according to a decaying spectrum which we present in Figure~\ref{app:fig:spectrum}.
\begin{figure}[H]
    \centering
    \includegraphics[width=0.4\linewidth]{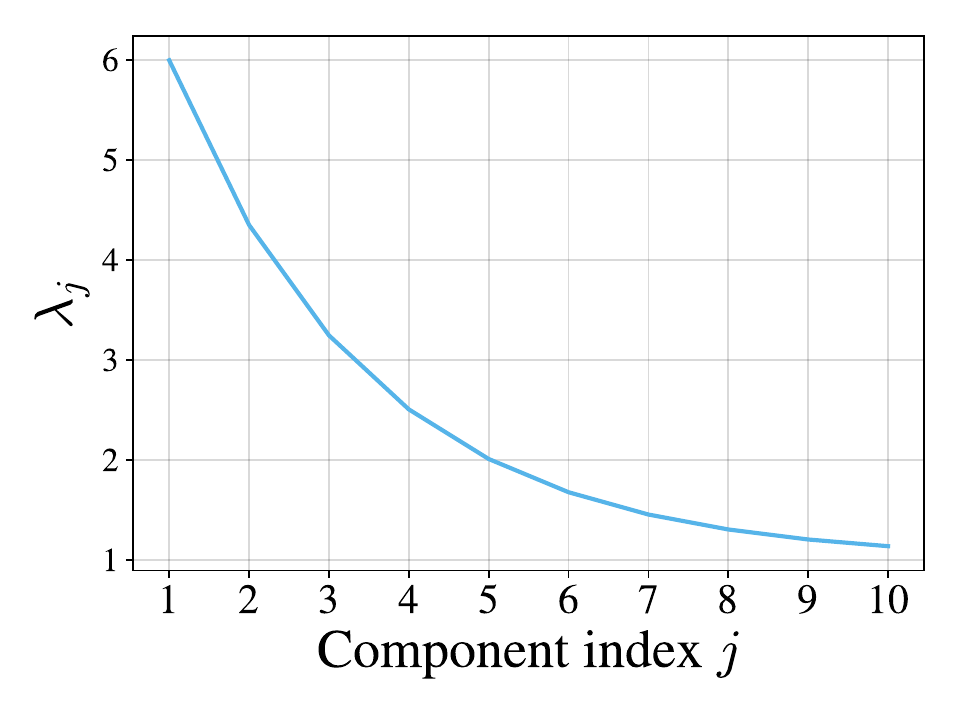}
    \caption{Decaying eigenvalue spectrum for the multi-component recovery experiment.}
    \label{app:fig:spectrum}
\end{figure}

In the subsequent plots, we show the performance for different values of $\delta$ when we let it be a fixed fraction of $\delta_{\max}$ across all components. That is, we let $\delta_j = C \delta_{\mathrm{max},j}$ for different $C\leq 1$. We highlight that for $C=0 \Rightarrow \delta_j = 0 \; \forall j \in [k]$, we recover PCA as described in Section~\ref{sec:closed-formula}; this is shown in Figure~\ref{app:fig:recoverpca} specifically.
\begin{figure}[H]
     \centering
     \begin{subfigure}[b]{0.49\textwidth}
         \centering
         \includegraphics[width=1\textwidth]{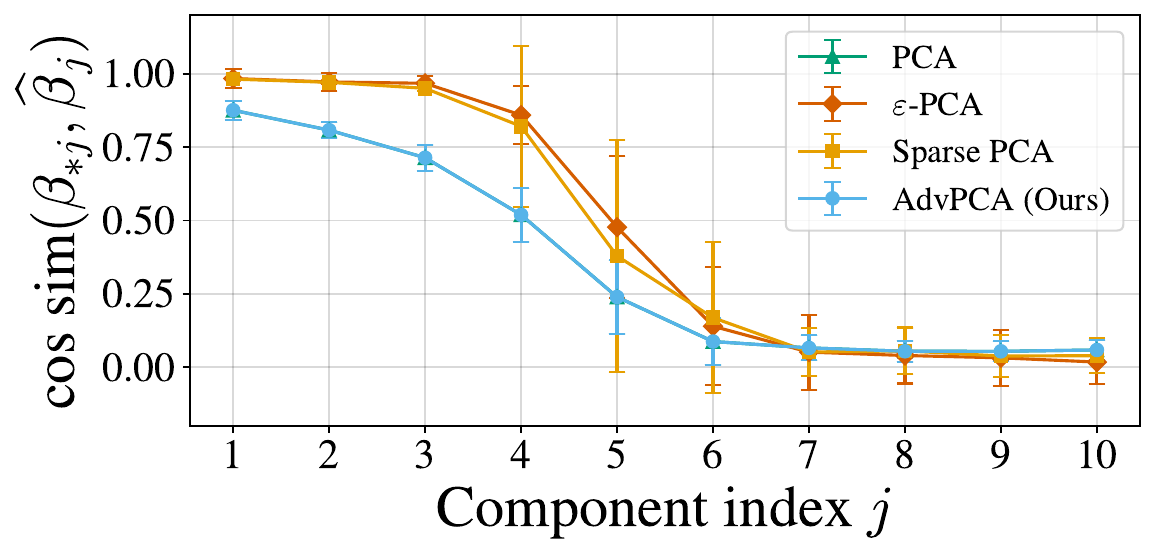}
     \end{subfigure}
     \hfill 
     \begin{subfigure}[b]{0.49\textwidth}
         \centering
         \includegraphics[width=1\textwidth]{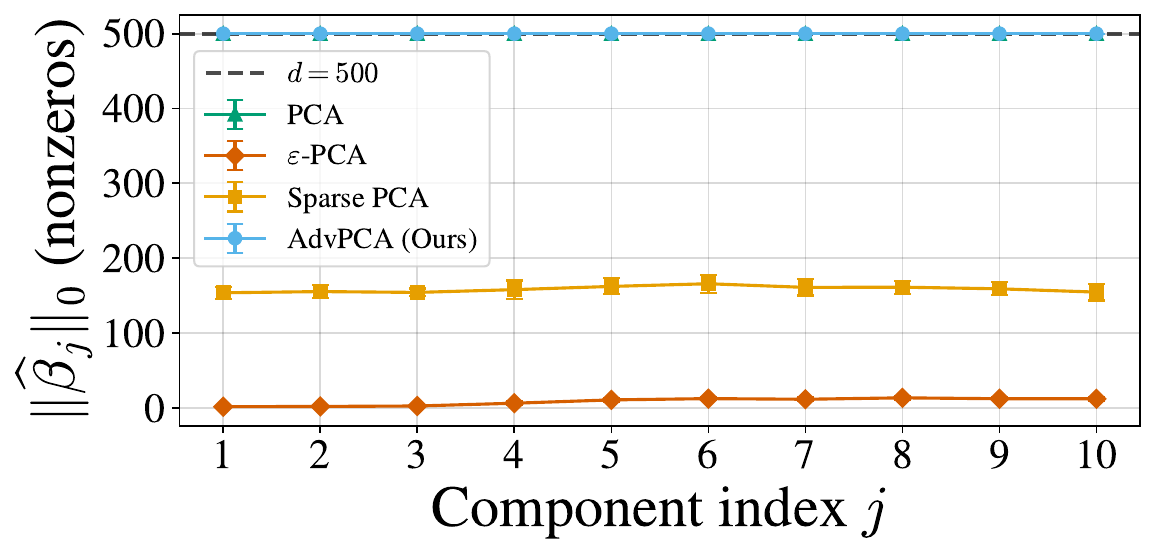}
     \end{subfigure}
     \caption{$C = 0$.}
     \label{app:fig:recoverpca}
\end{figure}
\begin{figure}[H]
     \centering
     \begin{subfigure}[b]{0.49\textwidth}
         \centering
         \includegraphics[width=1\textwidth]{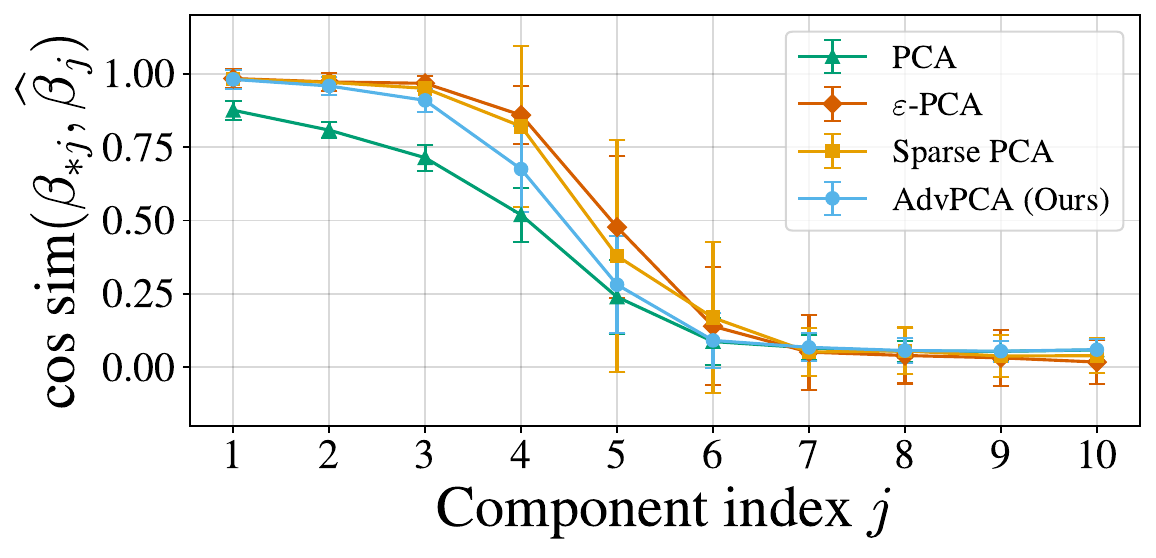}
     \end{subfigure}
     \hfill 
     \begin{subfigure}[b]{0.49\textwidth}
         \centering
         \includegraphics[width=1\textwidth]{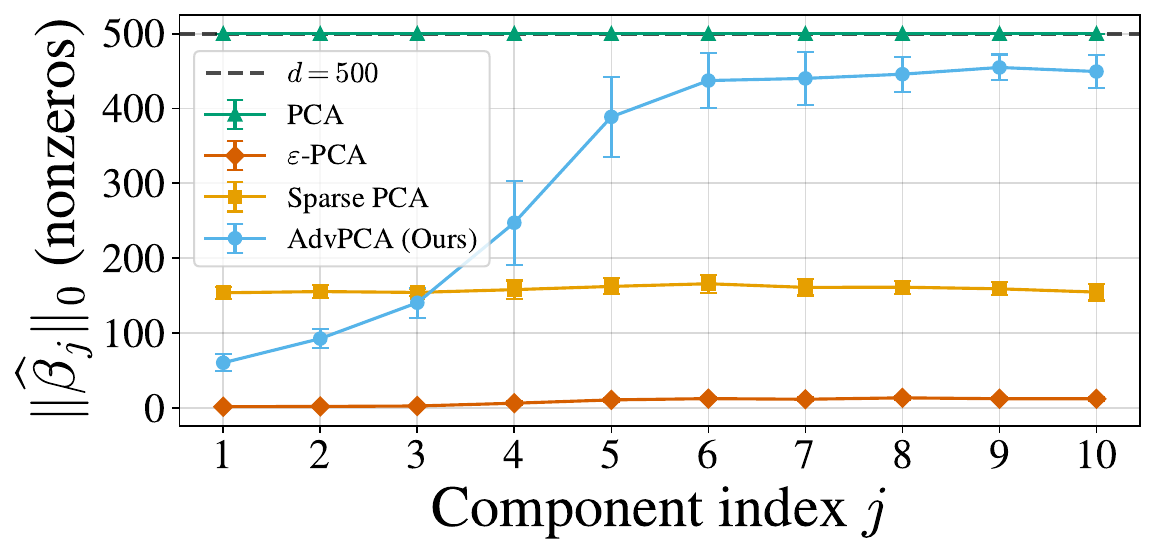}
     \end{subfigure}
     
     \caption{$C = 0.05$.}
     \label{fig:overall_comparison}
\end{figure}
\begin{figure}[H]
     \centering
     \begin{subfigure}[b]{0.49\textwidth}
         \centering
         \includegraphics[width=1\textwidth]{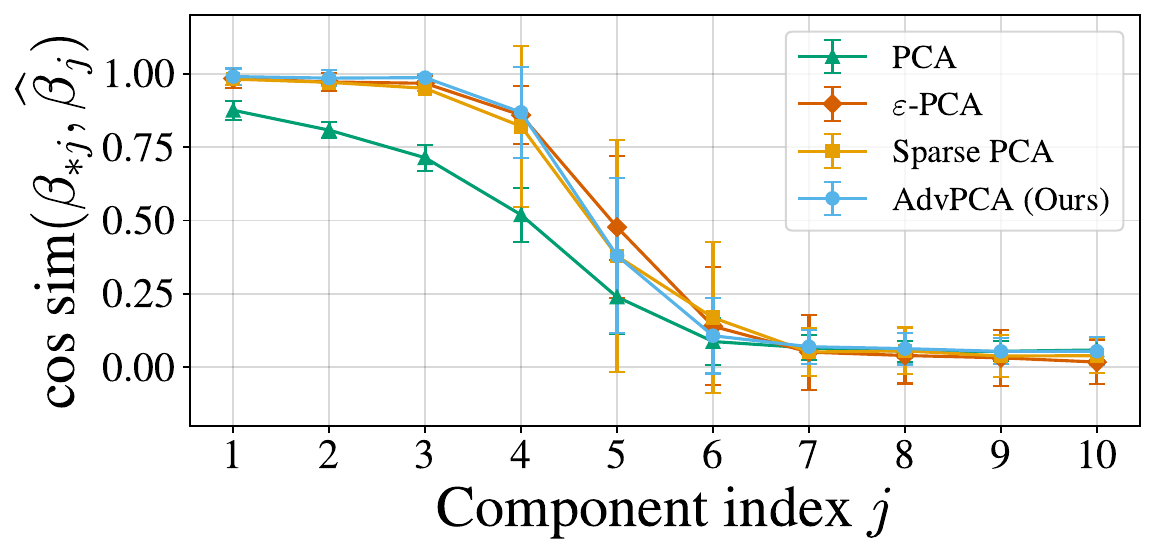}
     \end{subfigure}
     \hfill 
     \begin{subfigure}[b]{0.49\textwidth}
         \centering
         \includegraphics[width=1\textwidth]{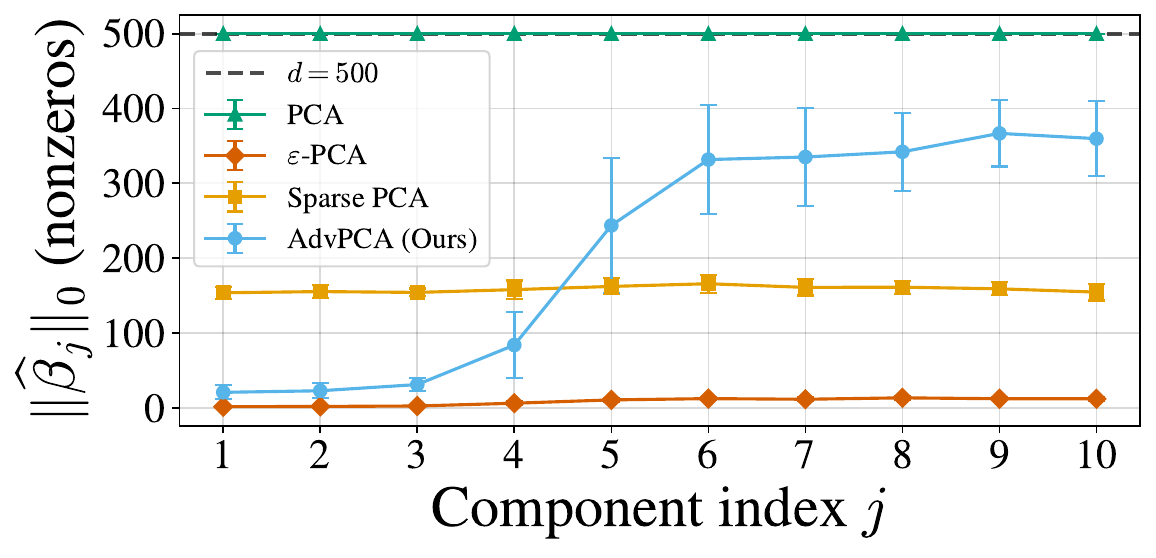}
     \end{subfigure}
     
     \caption{$C = 0.1$.}
     \label{fig:overall_comparison}
\end{figure}
\begin{figure}[H]
     \centering
     \begin{subfigure}[b]{0.49\textwidth}
         \centering
         \includegraphics[width=1\textwidth]{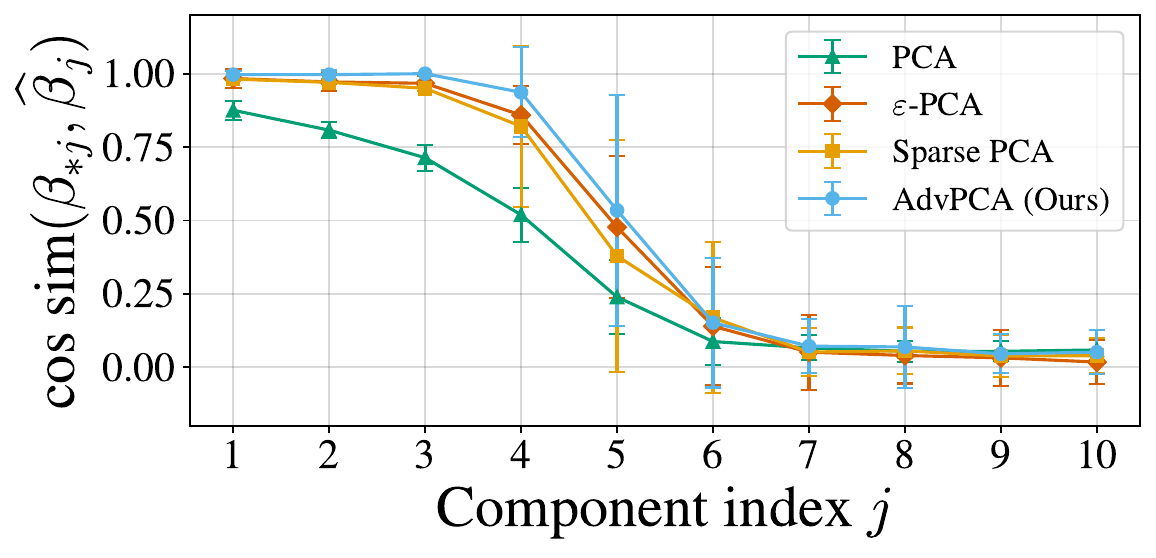}
     \end{subfigure}
     \hfill 
     \begin{subfigure}[b]{0.49\textwidth}
         \centering
         \includegraphics[width=1\textwidth]{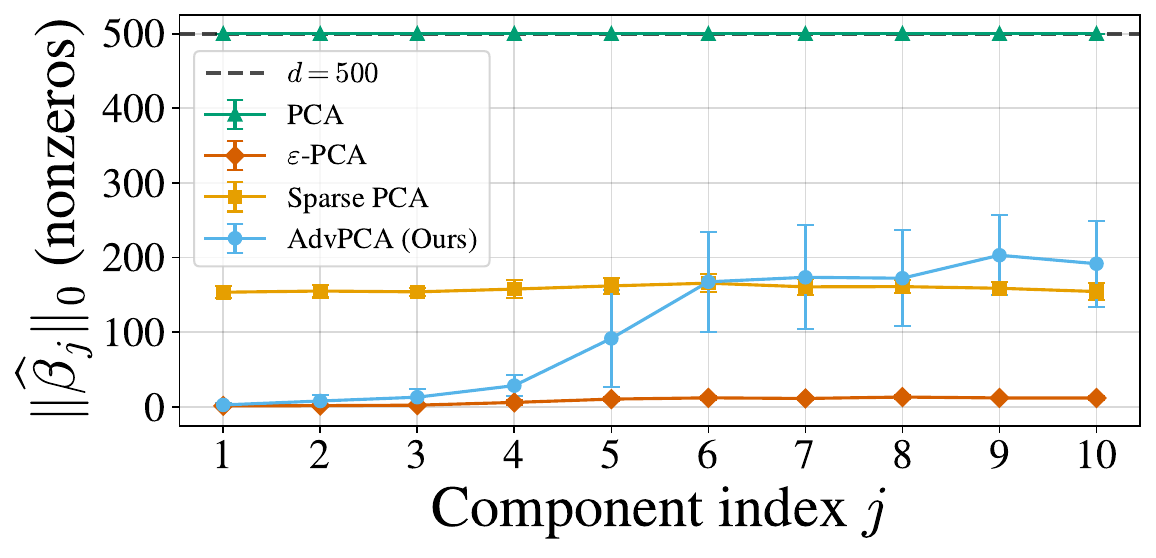}
     \end{subfigure}
     
     \caption{$C = 0.2$.}
     \label{fig:overall_comparison}
\end{figure}
\begin{figure}[H]
     \centering
     \begin{subfigure}[b]{0.49\textwidth}
         \centering
         \includegraphics[width=1\textwidth]{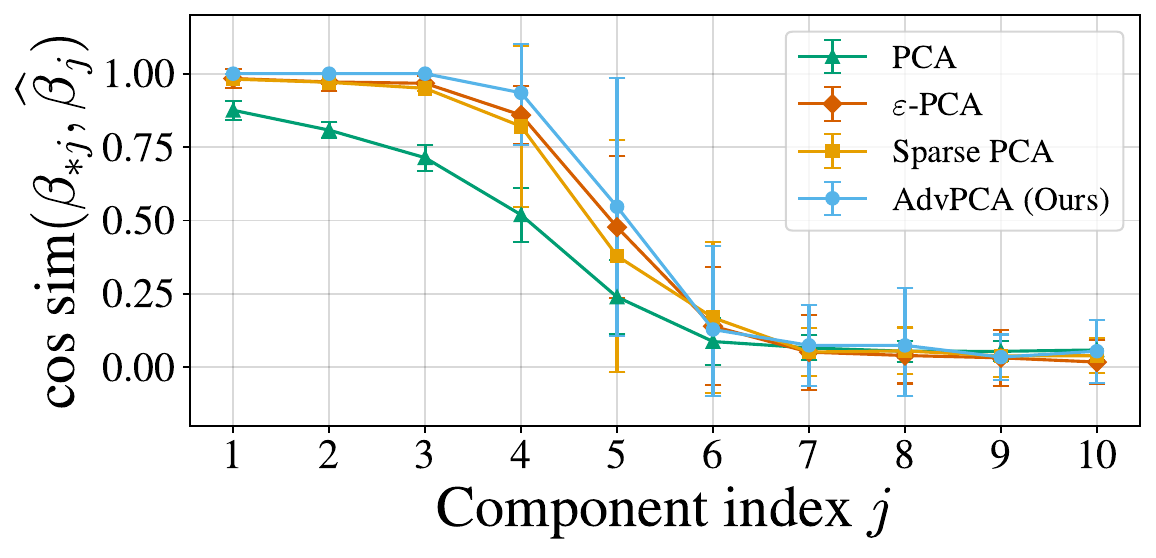}
     \end{subfigure}
     \hfill 
     \begin{subfigure}[b]{0.49\textwidth}
         \centering
         \includegraphics[width=1\textwidth]{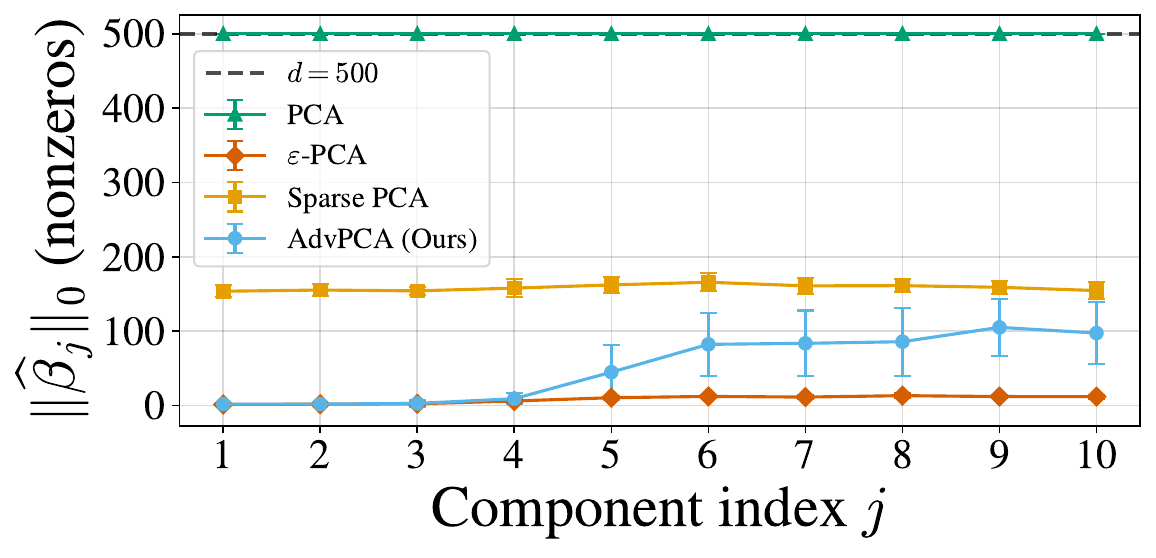}
     \end{subfigure}
     
     \caption{$C = 0.3$.}
     \label{fig:overall_comparison}
\end{figure}

\subsection{The MAGIC Diverse Wheat Dataset}\label{app:sec:magic}

In Figure~\ref{fig:magic}, we reconstruct the MAGIC diverse wheat dataset \cite{scott_limited_2021} and measure the reconstruction error $\snorm{X - \widetilde{X}}_F^2$ for the matrix $X \in \R^{n \times d}$ of wheat genome, along with the $\ell_0$-norm of the solution. The dataset contains $n=504$ samples of dimension $d=55\,067$. We do this for $k=25$ components motivated by the eigenvalue spectrum from PCA, which we present in Figure~\ref{app:fig:full_spectrum}.
\begin{figure}[H]
    \centering
    \includegraphics[width=0.9\linewidth]{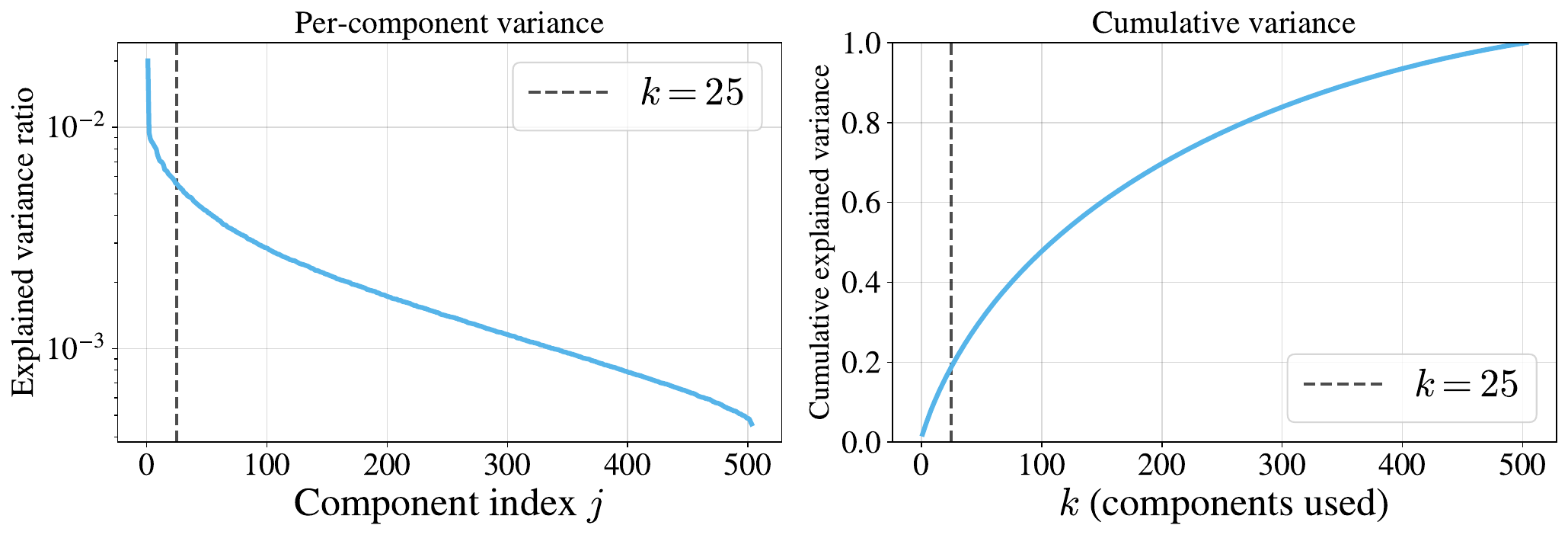}
    \caption{Eigenvalue spectrum for the MAGIC diverse wheat dataset for the full $k=n=504$ dimensions.}
    \label{app:fig:full_spectrum}
\end{figure}
We see a sharp drop in per-component variance in the first \raisebox{0.2ex}{\scriptsize $\sim$}10 components, and after that the spectrum flattens out indicating that even though the data is very high-dimensional, it is not low-rank. Intuitively, this could make sense seeing as we are working with genomics data which should be somewhat high-rank if evolution has been successful, but this is a very hand-wavy argument which we have no citation to back. Either way, based on the spectrum in Figure~\ref{app:fig:full_spectrum}, we let $k=25$ for the reconstruction experiment as this captures the most variation in the per-component variance. After $k=25$, the decay is rather linear. In Figure~\ref{app:fig:25_spectrum}, we show the exact spectrum for the $k=25$ components used in Figure~\ref{fig:magic}.
\begin{figure}[H]
    \centering
    \includegraphics[width=0.9\linewidth]{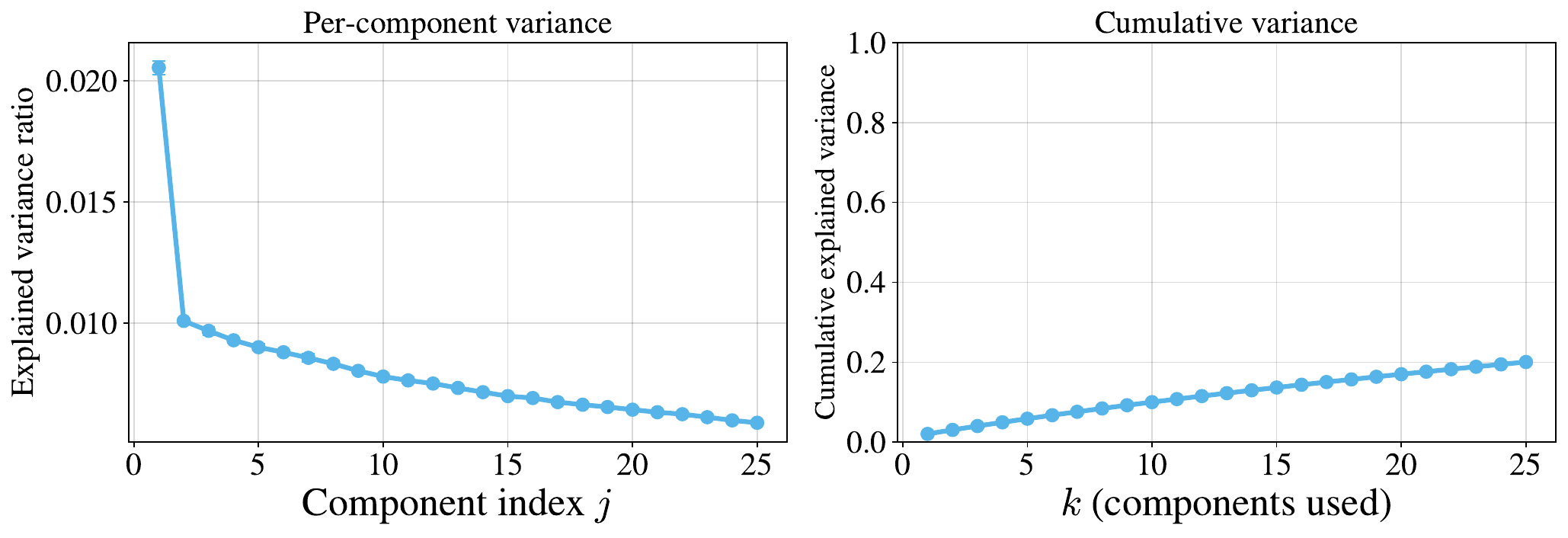}
    \caption{Eigenvalue spectrum for the $k=25$ components used in the MAGIC experiment presented in Figure~\ref{fig:magic}. This is essentially a zoomed-in version of Figure~\ref{app:fig:full_spectrum}.}
    \label{app:fig:25_spectrum}
\end{figure}
In addition to the $\ell_0$-norm, we also present the $\ell_1$-norm of the solution.
\begin{figure}[H]
    \centering
    \includegraphics[width=0.6\linewidth]{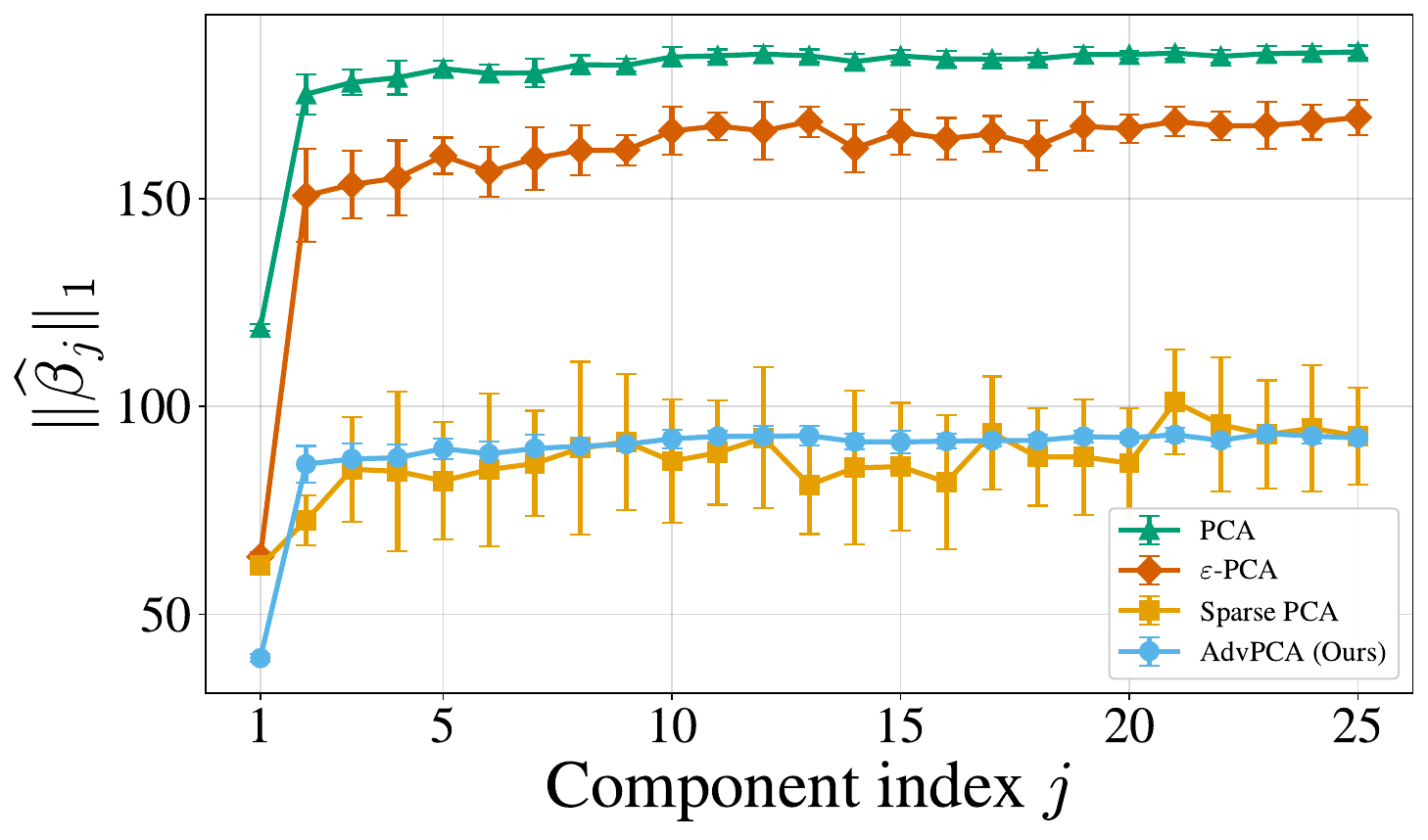}
    \caption{Corresponding $\ell_1$-norm of the solution in the MAGIC experiment presented in Figure~\ref{fig:magic}.}
    \label{fig:placeholder}
\end{figure}


\end{document}